\documentclass[final,onefignum,onetabnum]{siuro250211}

\def\pageoption{2}
\ifnum\pageoption=2
\usepackage{multibib}
\newcites{supp}{References}
\fi

\usepackage{lipsum}
\usepackage{amsfonts}
\usepackage{graphicx}
\usepackage{epstopdf}
\usepackage{algorithm}
\usepackage{algpseudocode}
\usepackage{standalone}
\usepackage{tikz}
\usepackage{amssymb,mathrsfs}
\usepackage[mathlines]{lineno}
\usepackage{soul}
\ifpdf
  \DeclareGraphicsExtensions{.eps,.pdf,.png,.jpg}
\else
  \DeclareGraphicsExtensions{.eps}
\fi

\usepackage{xcolor}
\newif\iftrackchanges
\trackchangestrue   

\usepackage{enumitem}
\setlist[enumerate]{leftmargin=.5in}
\setlist[itemize]{leftmargin=.5in}

\newsiamremark{remark}{Remark}
\newsiamremark{hypothesis}{Hypothesis}
\crefname{hypothesis}{Hypothesis}{Hypotheses}
\newsiamthm{claim}{Claim}
\newsiamremark{fact}{Fact}
\crefname{fact}{Fact}{Facts}

\headers{Batched Single-Index Global Multi-Armed Bandits with Covariates}{Sakshi Arya and Hyebin Song}

\title{Batched Single-Index Global Multi-Armed Bandits with Covariates}

\author{
  Sakshi Arya\thanks{Equal contribution. Department of Mathematics, Applied Mathematics and Statistics, Case Western Reserve University
  (\email{sxa1351@case.edu}).}
  \and 
  Hyebin Song\thanks{Equal contribution. Department of Statistics, Pennsylvania State University
  (\email{hps5320@psu.edu}).}
}

\usepackage{amsopn}

\ifpdf
\hypersetup{
  pdftitle={Batched Single-Index  Global Multi-Armed Bandits with Covariates},
  pdfauthor={S. Arya, and H. Song}
}
\fi

\newcommand{\E}{\mathbb{E}}
\renewcommand{\P}{\mathbb{P}}
\newcommand{\R}{\mathbb{R}}
\newcommand{\bB}{\mathbb{B}}
\renewcommand{\L}{\mathcal{L}}
\newcommand{\A}{\mathcal{A}}
\newcommand{\B}{\mathcal{B}}
\newcommand{\T}{\mathcal{T}}
\newcommand{\G}{\mathcal{G}}
\newcommand{\K}{\mathcal{K}}
\renewcommand{\S}{\mathcal{S}}
\newcommand{\I}{\mathcal{I}}
\newcommand{\TA}{\mathcal{T}_\mathcal{A}}
\newcommand{\SC}{\mathcal{S}_{C}}
\newcommand{\GC}{\mathcal{G}_{C}}
\newcommand{\Supp}{\operatorname{Supp}}

\newtheorem{assumption}{Assumption}

\usepackage{multirow}
\begin{document}
\maketitle

\begin{abstract}
The multi-armed bandits (MAB) framework is a widely used approach for sequential decision-making, where a decision-maker selects an arm in each round with the goal of maximizing long-term rewards. In many practical applications, such as personalized medicine and recommendation systems, contextual information is available at the time of decision-making, rewards from different arms are related rather than independent, and feedback is provided in batches.
We propose a novel semi-parametric framework for batched bandits with covariates that incorporates a shared parameter across arms. We leverage the single-index regression (SIR) model to capture relationships between arm rewards while balancing interpretability and flexibility. 
Our algorithm, Batched single-Index Dynamic binning and Successive arm elimination (BIDS), employs a batched successive arm elimination strategy with a dynamic binning mechanism guided by the single-index direction. We consider two settings: one where a pilot direction is available and another where the direction is estimated from data, deriving theoretical regret bounds for both cases. When a pilot direction is available with sufficient accuracy {and the number of arms $K$ is fixed}, our approach achieves minimax-optimal rates (with $d = 1$) for nonparametric batched bandits, 
circumventing the curse of dimensionality. Extensive experiments on simulated and real-world datasets demonstrate the effectiveness of our algorithm compared to the nonparametric batched bandit method introduced by \cite{jiang2025batched}.
\end{abstract}


\section{Introduction}\label{sec: 1_intro}
Sequential decision-making under uncertainty is fundamental in data-driven domains such as healthcare, agriculture, and online services. A foundational framework for this is the multi-armed bandit problem ~\cite{lai1985asymptotically, lai1987adaptive}, which aims to optimize the selection of actions (or arms) to maximize cumulative rewards over time. In this framework, a learner sequentially selects actions and observes their corresponding rewards. In many applications, contextual information (covariates), can significantly enhance decision-making. Incorporating these covariates extends the framework to contextual bandits or multi-armed bandits with covariates (MABC) \cite{perchet2013multi, yang2002randomized}.

Standard MABC approaches often assume independent arms, limiting their applicability in scenarios where playing one arm reveals insights about others, particularly for similar covariates. This shared informativeness is crucial in applications like clinical trials and personalized recommendations. For example, in clinical trials, treatments with similar chemical compositions are likely to exhibit analogous effects on patients with similar profiles (e.g., similar age group or disease severity).  To address this, the Global Multi-Armed Bandit (GMAB) framework was introduced, in which arms share a global parameter and are thus globally informative \cite{atan2015global, globalbanditsTekin, shen2018generalizedGlobalBandit}. However, standard GMAB model assumes known reward functions and cannot accommodate covariate effects, limiting its real-world applicability.

In this work, we address these limitations by introducing the \emph{Global Multi-Armed Bandit with Covariates} (GMABC) framework, which generalizes GMAB by (i) allowing reward functions to be unknown, and (ii) incorporating covariate information. In GMABC, arms are interconnected through a shared global parameter and the functions linking the global parameter to the rewards are unknown and can depend on the covariates.

In the MABC framework, the relationship between rewards and covariates is typically modeled using regression methods, which can be broadly classified as parametric 
\cite{goldenshluger2013linear, Filippi2010_GLM_Bandits, chu2011contextual, yadkori2011_linear, agrawal_TS_Linear_13} or non-parametric \cite{rigollet2010nonparametric, wanigasekara2019nonparametric,arya2023kernel}. 
Parametric methods assume a predefined relationship (such as linear or generalized linear models), offering interpretability and efficiency when correctly specified, but they can perform poorly under model misspecification. There are works that study parametric bandits under misspecification \cite{ghosh2017misspecified} but usually suffer an additional non-vanishing additive factor on the regret upper bound that depends on the degree of misspecification. 

Nonparametric bandits offer greater flexibility than parametric approaches and can model complex covariate-reward relationships. A large body of work has investigated nonparametric bandit models under the assumption that reward functions belong to certain infinite-dimensional function classes, such as the Lipschitz or H\"{o}lder classes \cite{yang2002randomized, perchet2013multi, rigollet2010nonparametric, gur2022smoothness, hu2020smooth}.
Another related research direction explores kernel and neural bandits \cite{valko2013finite, chowdhury2017kernelized, zhu2021pure, zhou2020neural}, where the reward functions are modeled in rich function spaces like reproducing kernel Hilbert spaces (RKHS) or neural networks, with assumptions on the \emph{effective dimensionality} of the covariates. These models allow more complex context-arm interactions, offering greater flexibility at the cost of added complexity. 

While these nonparametric approaches provide modeling flexibility, they come at the cost of computational complexity and reduced interpretability. Moreover, these methods treat arms independently, failing to exploit the shared relationship between covariates and rewards across arms that often exists in real-world applications. To address these limitations, we adopt a semi-parametric approach using the single-index model (SIM)~\cite{li1989regression, ichimura1993semiparametric, hardle1993optimal, kuchibhotla2020efficient, dai2022convergence}, where the expected reward for each arm depends on a one-dimensional projection of the covariates. This single-index model generalizes classical generalized linear models (GLMs) by treating the link function as unknown, offering greater flexibility while preserving interpretability. In contrast to unsupervised techniques such as Principal Component Analysis (PCA), which seek directions that maximize covariate variance irrespective of the outcome, the SIM framework aligns the projection direction with the conditional distribution of the reward. This supervised nature of the index vector estimation is critical in bandit problems, where exploration must be guided by reward-relevant structure rather than input variability alone, and also provides a well-suited framework to leverage the shared covariate-reward relationship across arms.

In many practical scenarios, such as clinical trials, data are collected in batches rather than in a fully sequential manner. For example, clinical trials often proceed in phases, where treatments are allocated for an entire batch and outcomes are analyzed collectively before updating the decision policy. Batched bandits with both fixed and adaptive batch sizes have been studied extensively in the literature ~\cite{perchet_batched2016, Esfandiari_Karbasi_Mehrabian_Mirrokni_2021, batched_thompson_sampling,jin2021almost}. Theoretical work on batched bandits has provided regret guarantees for both parametric \cite{han2020sequential, Ren_etal_glm_batchedbandit2022}   and nonparametric frameworks \cite{Gu_batchedNeuralBandits2024, jiang2025batched, feng2022lipschitz}, highlighting the relevance and challenges in scenarios with a small number of batches $(M \approx 2,3,4,5)$, as often seen in clinical trials.

\paragraph{Our Contributions} 
In this work, we study multi-armed bandits with covariates and shared information across arms in a batched setting. We propose a semi-parametric approach using the single-index model, offering flexibility, interpretability, and a natural framework for parameter sharing. 
To the best of our knowledge, this is the first systematic study of contextual bandits under a sufficient-dimension reduction paradigm using a single-index model structure. Our main contributions are as follows:
\begin{itemize}
    \item \textbf{GMABC Framework:} We introduce the Global Multi-Armed Bandit with Covariates (GMABC) model that leverages shared parameters across arms through a semi-parametric single index framework, allowing model flexibility while mitigating the curse of dimensionality and maintaining model interpretability.
    \item \textbf{BIDS Algorithm:} We propose a Batched single-Index Dynamic Binning and Successive arm elimination (BIDS) algorithm tailored to the batched GMABC setting.
    \item \textbf{Regret Guarantees:}  We derive a minimax lower bound for the batched semi-parametric GMABC problem under the single-index model, quantifying the fundamental difficulty of learning in this setting. We provide  regret guarantees for BIDS in two regimes: (i) when a reliable pilot estimate of the index is available and show that our upper bound is tight up to logarithmic factors {when $K$ (number of arms) is treated as fixed}, and (ii) when the index must be learned from data, characterizing trade-offs between estimation and learning.
    \item \textbf{Practical Implications:} Our analysis yields practical insights into the role of covariates and batch constraints in efficient decision-making under the GMABC model.
\end{itemize}

\paragraph{Related literature} 
Beyond the Global MAB framework, other bandit formulations have been considered for structured learning across arms.  Federated multi-armed bandits \cite{shi2021federated, xia2020multi} treat heterogeneous local models at distributed clients as random realizations of a shared global model, while structured or correlated bandits \cite{van2024optimal, gupta2021multi} assume  rewards lie within a known compact convex set or are linked through a latent random source. While federated bandits are designed for decentralized learning across multiple clients, each with its own local data, GMABC operates in a centralized setting with a single learner leveraging shared structure across arms and covariates. Structured and correlated bandits operate in static, non-contextual environments, whereas GMABC handles contextual, covariate-dependent rewards via a shared single-index projection, rendering those methods unsuitable for this contextual, semi-parametric setting.

A related line of work is the semi-parametric bandits framework \cite{Greenewald2017_actioncentered_semipara, krishnamurthy2018semiparametric, kim_semiparametric_19d}, which differs from our approach in its underlying model structure and the motivation for introducing nonparametric components. These works represent the mean reward function as the sum of a linear function of the arm with a shared parameter and a non-linear perturbation that is independent of the action/arm, treated as a confounder. Unlike the semi-parametric bandits literature, our model allows for non-linear treatment effects through unknown link functions specific to each arm and estimates the shared global parameter using single-index regression.

Another relevant theme is dimension reduction in the MABC framework under other structural assumptions such as sparsity or additivity. For instance, \cite{bastani2020online} introduces a LASSO bandit for high-dimensional covariates. Then, \cite{cai2022stochastic, pmlr-v37-kandasamy15} study additive models, where the regression function is assumed to be a sum of univariate functions of the $d$ individual covariates. Other works on dimension reduction in contextual bandits include \cite{qian2016kernel, li2021regret, li2022simple, li2023dimension,Qian02042024}. 

{We became aware, after initial submission of this paper, of concurrent work \cite{kang2025single} studying contextual bandits using a single-index model. Their setting, however, differs substantially from ours. In their formulation, each (non-batched) round presents $K$ candidate arm feature vectors sampled iid from a common distribution, and rewards are modeled through a common link function $f$ and a shared parameter $\theta^*$ which do not depend on the arm. In contrast, in our setting, a single context vector is observed at each round and the learner chooses among a fixed set of arms. Our framework considers arm-specific link functions under a common index parameter.
}



\section{Problem Setup}\label{sec: 2_problem_setup}
We begin by presenting the problem setup for the \emph{batched global multi-armed bandit with covariates (GMABC)} problem that we will be working with hereafter. 
We assume that we have $d$-dimensional covariates $X_1,X_2,\dots$ such that $X_t \sim \P_X$ i.i.d. for $t=1,\dots,T$. 
{We consider a $K$-armed setting in which the decision maker selects an arm $k \in [K] := \{1,\dots,K\}$ at each time point.}

The model for rewards for each arm $k \in {[K]}$ is given by:
\begin{align}\label{eq: model}
	Y_{t}^{(k)} = g^{(k)}(X_{t}) + \epsilon_t
\end{align} 
for $t=1,\dots,T$,  where $g^{(k)}:\R^d\to \R$ are the mean reward functions, and $\{\epsilon_t\}_{t\ge 0}$ is a sequence of independent mean zero random variables.  Furthermore, we assume the following single index model structure for $g^{(k)}$:
\begin{align}\label{def:sim}
	g^{(k)}(x) = f^{(k)}(x^\top \beta_0) 
\end{align} 
for $k=1,2,\dots,K$, where $f^{(k)}:\R\to\R$ are $1$-dimensional \emph{link functions} and $\beta_0 \in \mathbb{R}^d$ is the unknown \emph{index parameter or direction} shared by both arms. Throughout the paper, we assume $\|\beta_0\|_2=1$ for the identifiability of the parameter. Model \eqref{eq: model} together with \eqref{def:sim} defines the GMABC regression framework for the sequential decision-making problem.

A \emph{policy} $\pi_t:\mathcal{X} \to {[K]}$ for $t=1,\dots,T$ determines an action $A_t \in {[K]}$ at $t$. Based on the chosen action $A_t$, a reward $Y_t^{(A_t)}$ is obtained. In the sequential setting without batch constraints, the policy $\pi_t$ can depend on all the observations $(X_s, Y_s^{(A_s)})$ for $s<t$. In contrast, in a batched setting with $M$ batches, where $0=t_0 < t_1 <\dots<t_{M-1}<t_M = T$, for $t \in [t_i, t_{i+1})$, the policy $\pi_t$ can depend on observations from the previous batches, but not on any observations within the same batch. In other words, policy updates can occur only at the predetermined batch boundaries $t_1,\dots,t_M$. 

Let $\mathcal{G} = \{t_0, t_1,\dots,t_M\}$ represent a partition of time $\{0,1,\dots,T\}$ into $M$ intervals, and $\pi = (\pi_t)_{t=1}^T$ be the sequence of policies applied at each time step.  The overarching objective of the decision-maker is to devise an $M$-batch policy $(\mathcal{G}, \pi)$ that minimizes the expected \textit{cumulative regret}, defined as $\mathcal{R}_T(\pi) = \E[R_T(\pi)]$, where
\begin{equation}
	R_T(\pi) = \sum_{t=1}^T g^{(*)}(X_t) - g^{(\pi_t(X_t))}(X_t) = \sum_{t=1}^T f^{(*)}(X_t^\top\beta_0) - f^{(\pi_t(X_t))}(X_t^\top\beta_0), \label{eq: regret_def}
\end{equation}
and $g^{(*)}(x) = \max_{k\in{[K]}} g^{(k)}(x)$ is the expected reward from the optimal choice of arms given a context $x$. 
The cumulative regret quantifies the gap between the cumulative reward attained by $\pi$ and that achieved by an optimal policy, assuming perfect foreknowledge of the optimal action at each time step.
We make the following  assumptions on the reward functions.
\begin{assumption}[Lipschitz Smoothness]\label{assum: Smoothness}
We assume that the link function $f^{(k)}:\R \to \R$ for each arm is $L$-Lipschitz, i.e., there exists $L > 0$ such that for each $k \in {[K]}$,
\begin{align*}
    |f^{(k)}(u) - f^{(k)}(u^\prime)| \leq L |u - u^\prime|,
\end{align*}
holds for $u,u'\in \R$.
\end{assumption}
{
To state the margin condition, we first introduce additional notation. Denote the second pointwise maximum of the functions $g^{(k)}$, $ k = 1, \dots, K $, by $ g^{(\#)} $; formally, for every $ x \in \mathcal{X} $ where not all arms are simultaneously optimal, i.e., $ \min_k g^{(k)}(x) \ne  g^{(*)}(x) $, we define
\begin{align*}
g^{(\#)}(x) := \max_{k \in [K]} \big\{ g^{(k)}(x) \,:\, g^{(k)}(x) < g^{(*)}(x)\big\},
\end{align*}
which is the second best reward at $x$. When all arms tie for the best, we set $g^{(\#)}(x) := g^{(*)}(x)= g^{(1)}(x)$.
The gap function $ \Delta(x) := g^{(*)}(x)- g^{(\#)}(x) \ge 0$ critically controls the complexity of the problem, as it measures how clearly the optimal arm dominates the next best arm at context $x$.
The next assumption quantifies how often this gap is small.
}

\begin{assumption}[Margin]\label{assum: Margin}
Reward functions satisfy the margin condition with parameter $\alpha > 0$, that is, there exists $\delta_0 \in (0,1)$ and $D_0 > 0$ such that
    \begin{align*}
        \P_X(0 < {\Delta(X)} \leq \delta) \leq D_0 \delta^\alpha,
    \end{align*}
holds for all $\delta \in [0,\delta_0]$.
\end{assumption}

\begin{remark}
The margin parameter measures the complexity of the problem. A small $\alpha$ means that the two functions are quite close to each other in many regions. 
Throughout this paper, we assume that $\alpha \leq 1$, because in the $\alpha > 1$ regime, the context information becomes irrelevant as one arm dominates the other (e.g., see \cite{perchet2013multi}). 
\end{remark}

Let $\bB_2(r;c) = \{v \in \R^d; \|v-c\|_2 \le r\}$ denote the $\ell_2$ ball of radius $r$ centered at $c$.
The next assumption, Assumption \ref{assum: cond_X}, specifies conditions on the distribution of the reward $Y^{(k)}$ and covariate $X$.
\begin{assumption}\label{assum: cond_X}
The reward $Y_t^{(k)}$ satisfies $|Y_t^{(k)}| \leq  1$ for all $t=1,\dots,T,\, k\in {[K]}$. The probability measure $\P_X$ is absolutely continuous with respect to the Lebesgue measure, and its support set $\Supp(\P_X)$ is bounded, i.e., there exists $R_X< \infty$ such that $\Supp(\P_X) \subseteq \bB_2(R_X;0)$. Moreover, there exists $R_0>0$ such that for any $v \in \bB_2(R_0;\beta_0)$ and $\|v\|_2=1$, $\P_{X^\top v}$ is supported on an interval $\I_v \subseteq \R$, and the density function $f_{X^\top v}$ on $\I_v$ is bounded above and below by some constants $\overline{c}_X >0 $ and $\underline{c}_X >0$ independent of $v$. 
\end{assumption}
The boundedness assumption for rewards is made for technical reasons to apply concentration bounds. The constant $1$ is chosen for simplicity of exposition, but can easily be replaced with other (large) constants. For the distribution $\P_X$ of $X$, we assume that $\P_X$ has a density, its support is bounded in $\R^d$, and the density of the projection of $X$ onto a direction near $\beta_0$ is non-vanishing and supported on an interval in $\R$. Essentially, the last condition allows us to obtain information on $f^{(k)}$ from all regions given a sufficiently accurate working direction. Similar assumptions have been made in other non-parametric bandit settings for $\P_X$~\cite{perchet2013multi,jiang2019non}, where $\P_X$ is supported on a hypercube and its density does not vanish within that hypercube.

To provide a concrete example of $\P_X$ satisfying Assumption \ref{assum: cond_X}, consider $X$ following a truncated multivariate normal distribution $N(\mathbf{0}, \Sigma)$ constrained within a unit hypercube $\mathcal{H} = \prod_{j=1}^d 1\{|x_j|\le 0.5\}$, i.e., whose density is proportional to $\exp(-\frac{1}{2}x^\top\Sigma^{-1}x)1\{x\in \mathcal{H}\}$. We can find $R_0, \overline{c}_X$, and $\underline{c}_X$ that satisfy Assumption \ref{assum: cond_X}. See Lemma \ref{lem: tmtvn_condX} for details. The proof for the Lemma is provided in Section\ifnum\pageoption=2~\ref{appendix: proofsofSection2} \else~SM2 \fi in Supplementary Material.

\begin{lemma}\label{lem: tmtvn_condX}
Suppose $X \sim N_{T}(0,\Sigma;\mathcal{H})$ whose density is given by 
\begin{align*}
    f_{X}(x) = \begin{cases}
        \frac{1}{Z(\Sigma)} \exp\{-\frac{1}{2}x^\top\Sigma^{-1}x\}   & x\in \mathcal{H}\\
        0 & \text{otherwise}
    \end{cases}
\end{align*}
with $Z(\Sigma) =\int_{x\in \R^d} e^{-\frac{1}{2}x^\top\Sigma^{-1}x} 1\{x\in \mathcal{H}\} dx $ where $\mathcal{H} = \prod_{j=1}^d 1\{|x_j|\le 0.5\}$. Then we can find $R_0>0$ such that for any $v \in \bB_2(R_0;\beta_0)$ and $\|v\|_2=1$, the density of $\P_{X^\top v}$ is bounded above and below by some constants $\overline{c}_X >0 $ and $\underline{c}_X >0$. independent of $v$, on its support $\mathcal{I}_{v}$, which is an interval in $\R$.
\end{lemma}

\section{BIDS Algorithm for Batched GMABC} \label{sec: 3_BIDS_alg}

In this section, we propose an algorithm, which we call Batched single Index Dynamic Binning and Successive arm elimination (BIDS), for the batched GMABC problem.
%
Our algorithmic approach adapts the Adaptive Binning and Successive Elimination (ABSE) algorithm, first proposed in \cite{perchet2013multi} for contextual bandit problems with fully nonparametric reward functions. ABSE was shown to achieve the minimax rate under suitable smoothness and margin conditions. This strategy was adapted for batched settings in \cite{jiang2025batched}, which was also shown to achieve the minimax rate under batched constraints. 

We first provide a brief introduction on the ABSE strategy in subsection \ref{sec: 3_1_background}, then present the BIDS algorithm in subsection \ref{sec: dynamic_binning}, whose main idea is to execute the ABSE strategy in the \textit{projected} space based on the single index direction.

\subsection{Background on Adaptive Binning and Successive elimination Strategy}\label{sec: 3_1_background}
Perchet and Rigollet~\cite{perchet2013multi} propose two nonparametric contextual bandit algorithms, namely, \emph{Binned Successive Elimination (BSE)} and \emph{Adaptively Binned Successive Elimination (ABSE)} that leverage partitioning of the covariate space to manage exploration. In BSE, the context space \( [0,1]^d \) is uniformly divided into a fixed grid of bins. Within each bin, a separate instance of the classical Successive Elimination (SE) algorithm is run: for each arm, the empirical mean reward is updated based only on observations falling into that bin, and arms are successively eliminated when the difference in their estimated mean rewards from the current best arm exceeds a data-dependent confidence threshold. ABSE improves on this by dynamically refining the partition. It starts with large bins and adaptively splits them into smaller sub-bins when sufficient data has not been accumulated and the identity of the best arm is not yet clear. This localized refinement focuses exploration on regions where the optimal arm is hard to distinguish, allowing ABSE to match minimax-optimal regret rates (up to logarithmic factors) under H\"{o}lder smoothness assumptions on the reward function. Figure~\ref{fig:abse_illustration} provides a visual illustration of the ABSE algorithm in a two-dimensional covariate space, showing successive refinements at Level $1,2,$ and $3$.

Jiang and Ma~\cite{jiang2025batched} extend the ABSE approach to the batched bandit setting via the \emph{Batched Successive Elimination with Dynamic Binning (BaSEDB)} algorithm. They emphasize the importance of dynamic binning, where the covariate space is progressively refined with bin widths tailored to the batch size, in achieving minimax-optimal regret.

In this work, we address the batched GMABC problem and propose the \emph{Batched single-Index Dynamic binning and Successive arm elimination (BIDS)} algorithm. While BIDS builds on the adaptive refinement ideas of ABSE, it departs in two key ways: (i) it performs binning not in the full covariate space but along a one-dimensional projection defined by the estimated single-index direction, which in turn induces a partition in the covariate space; (ii) it explicitly models shared structure across arms through a global parameter. This allows BIDS to combine adaptive partitioning with sufficient dimension reduction, enabling more statistically and computationally efficient learning in high-dimensional contextual settings. Notably, both ABSE and BaSEDB treat arms independently and rely on uniform grid-based binning in the full covariate space, making them less suitable for settings with complex covariates or shared patterns across arms.

\begin{figure}[htbp]
\centering
\begin{minipage}{0.47\textwidth}
\centering
\resizebox{\textwidth}{!}{ 
\begin{tikzpicture}[scale=8]

\fill[blue!10] (0.5,0) rectangle (1,0.5);
\fill[blue!10] (0,0.5) rectangle (0.5,1);
\fill[blue!30] (0.5,0.5) rectangle (1,1);
\fill[blue!30] (0,0) rectangle (0.25,0.5);
\fill[blue!30] (0.25,0.25) rectangle (0.5,0.5);
\fill[red!20] (0.25,0) rectangle (0.5,0.25);

\draw[thick] (0,0) rectangle (1,1);

\draw[thick] (0.5,0) -- (0.5,1);  
\draw[thick] (0,0.5) -- (1,0.5);  

\draw[thick] (0,0.25) -- (0.5,0.25);  
\draw[thick] (0.25,0) -- (0.25,0.5);  

\draw[thick] (0.5,0.75) -- (1,0.75);  
\draw[thick] (0.75,0.5) -- (0.75,1);  

\draw[thick] (0.25,0.125) -- (0.5,0.125);
\draw[thick] (0.375,0.0) -- (0.375,0.25);



\node at (0.25,0.75) {\small $B_1$};
\node at (0.75,0.25) {\small $B_4$};

\node at (0.125,0.375) {\small $B_{2,1}$};
\node at (0.375,0.375) {\small $B_{2,3}$};
\node at (0.125,0.125) {\small $B_{2,2}$};

\node at (0.625,0.875) {\small $B_{3,1}$};
\node at (0.875,0.875) {\small $B_{3,3}$};
\node at (0.625,0.625) {\small $B_{3,2}$};
\node at (0.875,0.625) {\small $B_{3,4}$};

\node at (0.3125, 0.0625) {\tiny $B_{2,4,2}$};  
\node at (0.4375, 0.0625) {\tiny $B_{2,4,4}$};  
\node at (0.3125, 0.1875) {\tiny $B_{2,4,1}$};  
\node at (0.4375, 0.1875) {\tiny $B_{2,4,3}$}; 

\node[below] at (0,0) {\small $(0,0)$};
\node[below] at (1,0) {\small $(1,0)$};
\node[above] at (0,1) {\small $(0,1)$};
\node[above] at (1,1) {\small $(1,1)$};
\node[below] at (0.5,0) {\small $(0.5,0)$};
\node[left] at (0,0.5) {\small $(0,0.5)$};

\end{tikzpicture}
}
\end{minipage}
\hfill
\begin{minipage}{0.35\textwidth}
\centering
\scriptsize
\begin{tabular}{ll}
\colorbox{blue!10}{\phantom{XX}} & One active arm after Level 1\\
\colorbox{blue!30}{\phantom{XX}} & One active arm after Level 2\\
\colorbox{red!20}{\phantom{XX}} & Multiple arms \\
& \\
\textbf{Split factor:} & \\
Level 1: & 2 (split into $2^2$ bins) \\
Level 2: & 2 (split into $2^2$ bins) \\
Level 3: & 2 (split into $2^2$ bins) 
\end{tabular}
\end{minipage}
\caption{Illustration of ABSE in 2-dimensional setting. The algorithm partitions the context space ( $[0,1]^2$ ) at Levels 1, 2, and 3, running local arm elimination in each bin. Bins with confidently identified optimal arms (light-blue colored bins for Level 1 and blue-colored bins for Level 2) are not refined further, while bins without optimal arms are split into $2^2 = 4$ equal-sized sub-bins. 
}
\label{fig:abse_illustration}
\end{figure}

\subsection{Index based dynamic binning and arm elimination} \label{sec: dynamic_binning}
The main idea of our approach is to partition the covariate space $\mathcal{X}$ based on its \textit{one-dimensional projection} along the specified index estimate, using any off-the-shelf single-index estimator \cite{babichev2018slic, cai2020online}. This projection yields meaningful partitions, as the index is learned via supervised modeling of the reward-context relationship.
Once the partition is formed, decisions within each bin of the covariate space can be made by treating the problem as a standard stochastic bandit problem without covariates, with the average regret within each bin estimated as a constant.

To form a partition, an index vector $\beta$ is required to determine the direction along which $x\in\R^d$ is projected.
We consider two settings: one where a pilot estimate $\beta \in \R^d$ is provided with reasonable accuracy, and another where no pilot estimate is available.  
When a pilot estimate $\beta$ is available, for instance from previous studies or other preliminary analyses, we propose the BIDS algorithm based on partitioning of the covariate space guided by the direction of $\beta$ (Algorithm \ref{algorithm: SIRBatchedBinning}). 
In the absence of a pilot estimate, we begin with an initial phase where we first collect i.i.d. observations from each arm in a cyclic manner. These observations are then used to estimate the index vector. Once the direction is estimated, the BIDS algorithm applied in the first setting can be utilized. 
First, we discuss the BIDS algorithm with a given direction $\beta$. In the next subsection (Section \ref{sec: initial_phase}), we present an algorithm to estimate the index vector during the initial phase when $\beta$ is not available.  

To enhance readability, we summarize key notations in Table\ifnum\pageoption=2~\ref{tab:notations_long} \else~SM1 \fi in Supplementary Material.
Given a pilot direction $\beta \in \R^d$ such that $\|\beta\|_2 = 1$,
the dynamic binning strategy employed in our algorithm can be explained through a tree-based interpretation as follows. 

\paragraph{Hierarchical partitioning and tree structure}
We build a tree $\mathcal{T}$ of depth $M$ (recall, $M$ is the number of batches) to adaptively partition the covariate space based on the projected direction $\beta$. Each layer consists of a progressively finer partition of the covariate space $\mathcal{X}\subseteq \R^d$, where the partitions are defined by the direction $\beta$ and the number of splits at each layer $\{b_l\}_{l=0}^{M-1}$.

Let $\I_\beta =\{x^\top\beta; x\in \mathcal{X}\}$, which is an interval by Assumption \ref{assum: cond_X}, i.e., let $\I_\beta = [L_\beta,U_\beta]\subseteq\R$. For layer $i=1,\dots,M$, we create a partition $\mathcal{A}_i$ of $[L_\beta,U_\beta]$ by splitting it into $n_i=\prod_{l=0}^{i-1} b_l$ equal-width intervals. Each interval $A_i \in \mathcal{A}_i$ has width
\begin{align}
    w_i = \frac{U_\beta - L_\beta}{n_i} = (U_\beta - L_\beta)(\prod_{\ell = 0}^{i-1} b_\ell)^{-1},  \label{width: def}
\end{align}
and takes the form:
\begin{align*}
    A_i :=  \begin{cases}
        [L_\beta+(v-1)w_i, L_\beta+vw_i) & v = 1,2,\dots, n_i-1\\
        [L_\beta+(n_i-1)w_i, U_\beta] & v = n_i\\
    \end{cases}
\end{align*}
where 
for each layer $i=1,2,\dots,M.$
We then define a partition $\mathcal{B}_i$ of $\mathcal{X}$ for layer $i=1,2,\dots,M$, which consists of bins $C_{A_i}(\beta)$ defined as:
\begin{align*}
    C_{A_i}(\beta) = \{x \in \mathcal{X}: x^\top \beta \in A_i\}.
\end{align*}
It is easy to check that each $\mathcal{B}_i$ is a partition of $\mathcal{X}$.

The tree $\mathcal{T}$ is defined as the collection of $\mathcal{B}_i$'s, i.e., $\mathcal{T} = \cup_{i=1}^M \mathcal{B}_i$, and for reference, we define $\mathcal{T}_\mathcal{A} = \cup_{i=1}^M \mathcal{A}_i$.
Note that by the setup, for each bin $C \in \mathcal{T}$,  we have $C = C_{A}(\beta)$ for some set $A \in \mathcal{T}_\mathcal{A}$.  We will sometimes need to refer to the width of $A$ that defines $C$. For $C \in \T$, define $|C|_\T$ as $|C|_\T = |A|$ where $C = C_A(\beta)$.

\paragraph{Parent and children bins}
The nested structure of partitions naturally creates parent-child relationships between bins.
For $A \in \TA$, we define its child and parent sets as follows. Since $A \in \TA$, we have $A \in \mathcal{A}_i$ for some $i \in \{1,\hdots,M-1\}$. We define its \textit{child} set as $\text{child}(A):=\{A'\in \mathcal{A}_{i+1}; A' \subseteq A\}$, consisting of all intervals in the next layer contained in $A$. The \textit{parent} of $A$ is defined as $p(A) = \{A'\in \mathcal{A}_{i-1}; A \in \text{child}(A')\}$, which is the interval in the previous layer that contains $A$. 
These relationships extend to bins in the covariate space $\mathcal{X}$: for a bin $C_{A}(\beta) \in \mathcal{B}_i$, we define its child and parent as
$\text{child}(C_{A}(\beta))=\{C_{A'}(\beta); A' \in \text{child}(A)\}$ and $p(C_A(\beta)) = \{C_{A'}(\beta); A \in \text{child}(A')\}$. 
For $C\in \mathcal{T}$ (or $\TA$), we define $p^k(C) = p(p^{k-1}(C))$ to be the $k$th ancestor of $C$ for $k \geq 2$. Then we let $\mathcal{P}(C) = \{C^\prime \in \mathcal{T} \mbox{ (or $\TA$)}: C^\prime = p^{k}(C)$ for some $k \geq 1\}$ be the set of all ancestors of $C$. By construction, the parent-child relationships are consistent between the projected intervals and bins in covariate space: if $A' = p(A)$ then $C_{A'}(\beta) = p(C_A(\beta))$.

\paragraph{BIDS algorithm}
Our proposed algorithm, Algorithm \ref{algorithm: SIRBatchedBinning} (BIDS), proceeds in batches and each batch has two key terms, a list of \emph{active bins} $\mathcal{L}_t$ at time $t$ and the corresponding \emph{active arms} $\mathcal{I}_C$ for each $C \in \mathcal{L}_t$. 
Before the first batch, $\mathcal{L}_1 = \mathcal{B}_1$, i.e., the list of active bins $\mathcal{L}_1$ contains all bins in layer 1, and $\mathcal{I}_C = {\{1,2,\dots,K\}}$ for all $C \in \mathcal{L}_1$, i.e., each bin contains all active arms. 
In each batch, observations are drawn cyclically from each of the active arms. At the end of the batch, all the rewards in the batch are revealed.  Using this information, we perform an arm elimination procedure to update the active arms set $\mathcal{I}_C$. Specifically, for each active arm set  with multiple active arms, we eliminate arms that are ``statistically worse than the best arm''. Then, if any active bin still has multiple active arms, this suggests the bin is not fine enough for the decision-maker to  tell the difference between the remaining arms. As a result, we split any active bin that still has more than one active arm into its children sets $\text{child}(C)$ in $\mathcal{T}$. Finally, we update the set of active bins and repeat this process at the end of each batch.

Since the set of active bins is only updated at the end of each batch, $\mathcal{L}_t$ only changes in the beginning of a new batch. That is, $\mathcal{L}_t$ is different from $\mathcal{L}_{t-1}$ only when $t = t_0+1,\dots,t_{M-1}+1$. We let $\mathcal{L}^{(i)} = \mathcal{L}_{t_{i-1}+1}$ to denote the list of active sets during the $i$th batch for $i=1,\dots,M$, and $\L^{(0)}=\emptyset$. 
We will say that a set $C\in \T$ is \emph{born} at batch $i$ if $C \notin \L^{(i-1)}$ and $C \in \L^{(i)}$. This happens if $p(C)$ was split at the end of batch $i-1$. We note that by the set-up of algorithm, the sets that are born at the beginning of batch $i$ always belong to $\mathcal{B}_i$. This is because when $i=1$, $\L^{(1)}= \mathcal{B}_1$ by the set-up of the algorithm, so all sets born at batch $1$ belong to $\mathcal{B}_1$. Then the sets that are born at batch $i$ are always children of the sets that were born at $i-1$.

\begin{remark}[Unique batch elimination event for each set]\label{rmk: batch_elimination}
For a set $C$ which was born at batch $i$, by the construction of the algorithm, the batch elimination procedure will be performed for $C$ at the end of batch $i$. Also note that, $C \in \L^{(j)}$ for all $j>i$ if and only if $C$ has exactly one active arm after the batch elimination procedure at the end of batch $i$. In particular, at the end of batch $i$, the batch elimination procedure is performed only for those bins that are born at the beginning of batch $i$. As a consequence, each bin undergoes at most one batch elimination event.
\end{remark}
\paragraph{ Batch elimination procedure }
For each ``newly'' born $C \in \B_i$, for $i=1,\dots,M$, we obtain reward information from each active arm during batch $i$ and perform a batch elimination event at the end of batch $i$. Specifically, during batch $i$, we obtain average rewards on $C$ from active arms by pulling each arm in a fixed, cyclic order whenever $X_t \in C$.  At the end of batch $i$, we perform a batch elimination procedure.

More precisely, let $\tau_{C,i}(s) = \inf\{n\ge \tau_{C,i}(s-1)+1; X_n \in C\}$ be the $s$th time that covariate $X_t$ is in $C$ during the batch $i$, where $\tau_{C,i}(0) = t_{i-1}$, for $s=1,2,\dots$. Let $m_{C,i} = \sum_{t=t_{i-1}+1}^{t_i} 1\{X_t \in  C\} $ be the total number of visits of $X_t$ to $C$ during batch $i$. For the $s$th visit to $C$, we pull the arm {$k=  ((s-1) \bmod |\I_C|) +1$, so that the active arms in $\I_C$ are cycled through in a fixed order. For example, if $\I_C=\{1,2,3\}$, then arms $1,2,3,1,2,3,\dots$ are pulled on successive visits.} 

Let $\tau_{C,i}^{(k)} = \{\tau_{C,i}(s);\, 1\le s \le m_{C,i},\, {k=  ((s-1) \bmod |\I_C|) +1} \} $ be the set of time points $t$ during batch $i$ when $X_t$ visits $C$ and the arm $k \in {[K]}$ is pulled, and let $m_{C,i}^{(k)} = |\tau_{C,i}^{(k)}|$ denote the number of such visits. 
The average rewards for $C$ from arm $k\in{[K]}$ during batch $i\in \{1,\dots,M\}$ is then:
\begin{align}\label{eq: reward_def}
    \bar{Y}_{C,i}^{(k)} &= \frac{1}{m_{C,i}^{(k)}}\sum_{t \in \tau_{C,i}^{(k)}} Y^{(k)}_t.
\end{align}
Once $\bar{Y}_{C,i}^{(k)}$ for $k\in{[K]}$ are obtained, we check whether,
\begin{align}
    \max_{l\in{[K]}}\bar{Y}_{C,i}^{(l)} - \bar{Y}_{C,i}^{(k)} >U(m_{C,i}, T, C),
    \label{eq: batch_elimination}
\end{align}
where we define,
\begin{equation}\label{def: U}
{U(m,T,C) := 4\sqrt{\frac{2\log(2KT|C|_\T)}{\lfloor m/K\rfloor \vee 1}}}
\end{equation}
where we recall $|C|_\T= |A|$ for a set $A$ such that $C = C_A(\beta)$. In particular, for $C \in \B_i$, $|C|_\T = |A_i|$ for $A_i \in \mathcal{A}_i$.
We eliminate $k$ from the set of active arms for $C$ if $k$ satisfies \eqref{eq: batch_elimination}. {This rule ensures that arms are only eliminated when there is statistically significant evidence that another arm dominates it in expectation over the bin $C$.}

\begin{figure}[htbp]
    \centering
    \resizebox{.7\textwidth}{!}{
        \begin{tabular}{c c}
         (a) & (b) \\\includegraphics[width=0.45\linewidth]{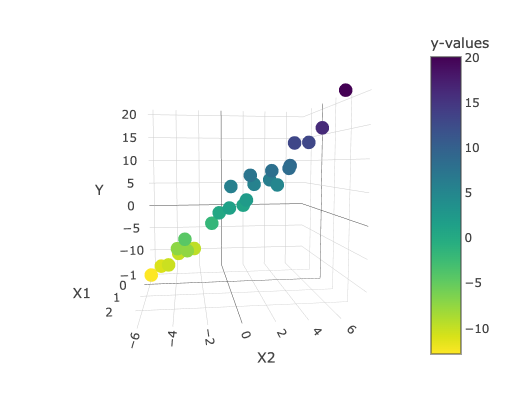}   &   \includegraphics[width=0.4\linewidth]{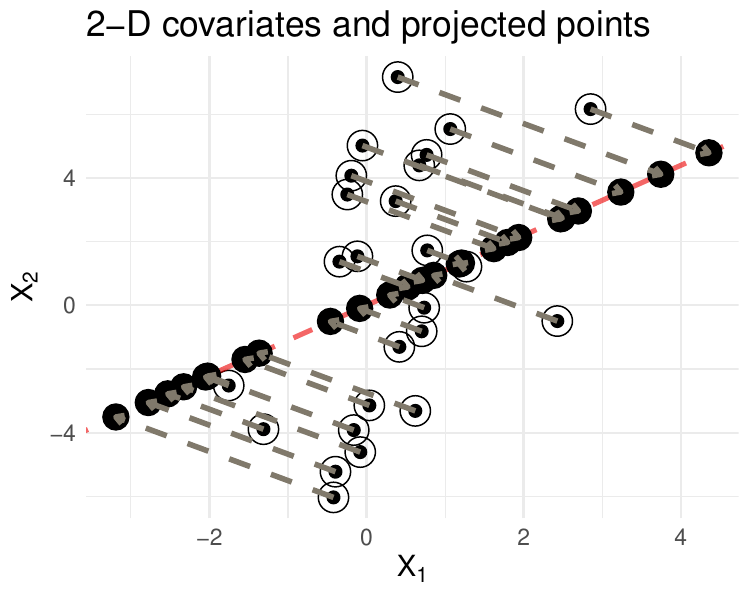}  \\
           (c) & (d) \\ \includegraphics[width=0.4\linewidth]{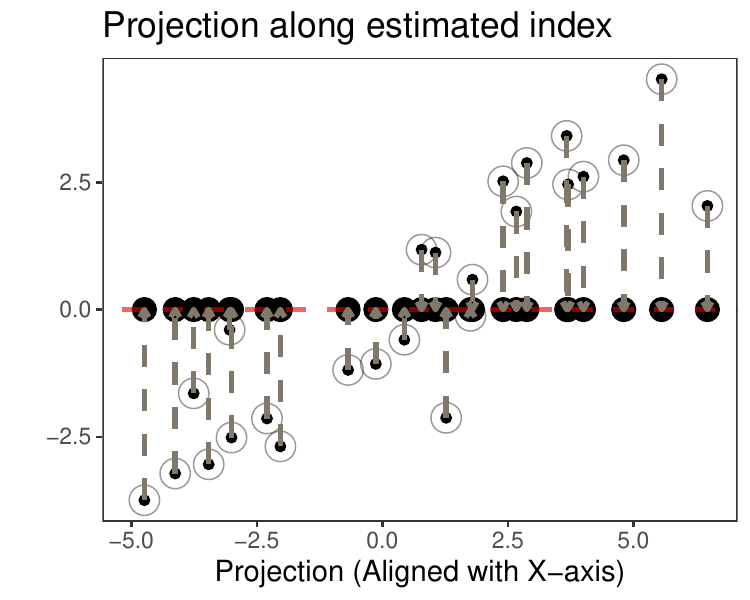}   &  \includegraphics[width=0.4\linewidth]{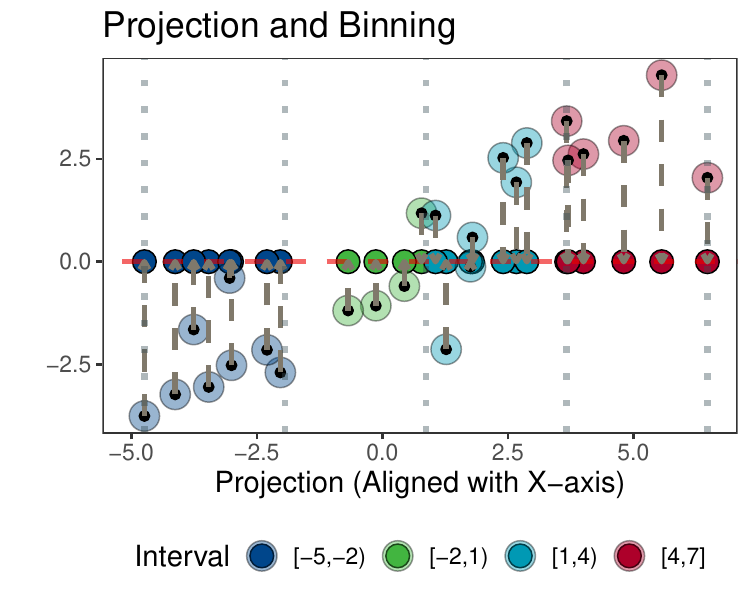}
        \end{tabular}}
    \caption{A linear model example with $Y_t= \beta_1 X_{t,1} + \beta_2 X_{t,2} + \epsilon_t$, where $X_t = (X_{t,1}, X_{t,2})  \in \mathbb{R}^2$ and $\epsilon_t \overset{\textrm{i.i.d} }{\sim}N(0, \sigma^2 = 1)$ for $t = 1, \dots, 25$. 
    (a) 3-D representation of the simulated data. (b)  Projection of covariates $X \in \mathbb{R}^2$ (circles with holes) onto the single-index direction (red dotted line),  with projected points shown as black circles connected by gray lines. (c) Rotated view of (b) to align the SIR direction with the x-axis. (d) Binning of the projected interval into four sub-intervals, with colors representing bin membership. The same process holds for all layers, $i = 1,\hdots, M$.
    }
    \label{fig: ProjectandBin}
\end{figure}

Algorithm \ref{algorithm: SIRBatchedBinning} summarizes the BIDS algorithm, which performs hierarchical partitioning based on projection along a given index vector and dynamic binning through successive arm elimination and active set updates. Figure \ref{fig: ProjectandBin} visualizes this partitioning in the projected space.

\begin{algorithm}[ht]
\begin{algorithmic}[1]
\State \textbf{Input: } No. of batches $M$, grid $\{t_i\}_{i=0}^M$, split factors $\{b_i\}_{i=0}^{M-1}$, working direction: $\beta$
\State Initialize active bins: $\mathcal{L}^{(1)} \leftarrow \mathcal{B}_1$. 
\State Initialize active arms: $\mathcal{I}_C \leftarrow {[K]}$ for all $C \in \mathcal{L}^{(1)}$
\For{$i = 1, \dots, M$}
\For{$t = t_{i-1}+1, \dots, t_i$} \Comment{draw observations (during batch $i$)} 
\State Find $C \in \mathcal{L}^{(i)}$ such that $X_t \in C$.
\State Pull an arm from $\mathcal{I}_{C}$ in a cyclic manner (let $s$ be the number of visits to $C$ up to the current time. set $Y_t = Y_t^{(k)}$, for ${k = (s-1) \bmod |\I_C| +1}$.)
\EndFor
\If{$t = t_i$ and $i < M$} \Comment{Batch elimination (at the end of batch $i$)}
\State Rewards during batch $i$, $Y_{t_{i-1} + 1}, \dots, Y_{t_i}$, are revealed.
\State Initialize $\mathcal{L}^{(i+1)} = \{\}$.
\For{$C \in \mathcal{L}^{(i)}$}\Comment{Iterate over active bins}
\If{$|\mathcal{I}_{C}| = 1$} \Comment{if only one active arm remains in $C$}
\State $\mathcal{L}^{(i+1)} = \mathcal{L}^{(i+1)} \cup \{C\}$ 
\State Break (Proceed to the next bin $C$)
\Else{ $|\mathcal{I}_{C}| > 1$} \Comment{if more than one active arm remains}
\State $\bar{Y}_{C, i}^{\text{max}} = \max_{k \in \mathcal{I}_{C}} \bar{Y}_{C, i}^{(k)}$ 
\For{$k$ in $\mathcal{I}_{C}$}\Comment{successive arm elimination}
\If{$\bar{Y}_{C,i}^{\text{max}} - \bar{Y}_{C,i}^{(k)} > U(m_{C,i}, T, C)$} 
\State $\mathcal{I}_{C} = \mathcal{I}_{C}\setminus \{k\}$
\EndIf 
\EndFor 
\If{$|\mathcal{I}_{C}| > 1$}\Comment{if arm elimination did not occur}
\State $\mathcal{I}_{C'} = \mathcal{I}_{C}$, for $C^\prime \in \text{child}(C)$\Comment{split the bin into children bins}
\State $\mathcal{L}^{(i+1)} = \mathcal{L}^{(i+1)} \cup \{C'; C' \in \text{child}(C)\}$ \Comment{update the active bins}
\EndIf 
\EndIf
\EndFor
\EndIf
\EndFor
\end{algorithmic}
\caption{BIDS algorithm}
\label{algorithm: SIRBatchedBinning}
\end{algorithm}

\subsection{Estimation of single-index vector without a pilot estimate}\label{sec: initial_phase}
In this subsection, we discuss the process of estimating the single-index vector using a separate initial phase when no pilot estimate is available. {We divide the time horizon $1,\dots,T$ into two phases: an initialization phase ($t = 1,\ldots, t_{\text{init}})$, and a second phase ($t = t_{\text{init}} + 1, \ldots, T$), in which we run the BIDS algorithm over the remaining $M-1$ batches.
In the initialization phase (first batch), we draw i.i.d. samples from each arm $k \in [K]$ to estimate the index direction, and in the subsequent $M-1$ batches, we run the BIDS algorithm (Algorithm \ref{algorithm: SIRBatchedBinning}) using the estimated direction.}

More specifically, in the initial phase, we draw i.i.d. samples cyclically from all arms, assigning arm ${k = (t-1) \bmod |\I_C| +1}$ at time $t$.
We construct i.i.d. datasets $\mathcal{D}^{(k)}_{\rm init} = (X_t, Y_t^{(k)})_{t \in \tau^{(k)}}$ for each arm $k \in {[K]}$, where
{$\tau^{(k)} := \{1 \leq t \leq t_{\rm init} ;\, k = (t-1)\bmod K +1\}$.} Once these datasets are available, any single-index regression (SIR) algorithm can be employed to estimate the direction $\beta_0$. For example, in Section\ifnum\pageoption=2~\ref{sec: SADEappendix} \else~SM4 \fi in Supplementary Material, we demonstrate this process using the Sliced Average Derivative Estimation (SADE) method from \cite{babichev2018slic}.

Let $\hat{\beta}^{(k)}$ denote the estimate of $\beta_0$ obtained using $\mathcal{D}^{(k)}_{\rm init}$ for $k = {1,2,\dots,K}$. Since single-index models estimate the direction up to a rotation, we cannot simply combine these vectors by taking their (weighted) average. We propose to first estimate the projection matrix $\mathcal{P}_0 = \beta_0\beta_0^\top $ of $\beta_0$ by computing a (weighted) average of the projection matrices from each arm with weights $\omega_k$, i.e., $\hat{\mathcal{P}} = \sum_{k=1}^{{K}} \omega_k \,\hat{\beta}^{(k)}(\hat{\beta}^{(k)})^\top$, then we obtain the final vector $\hat{\beta}$ by computing the first eigenvector of the estimated matrix $\hat{\mathcal{P}}$.
In our simulations and real-data illustrations in Sections \ref{sec: 5_simulation} and \ref{sec: 6_realdata}, we use the average with equal weights $\omega_k = {1/K}$ for datasets corresponding to each of the $K$ arms.  We summarize the procedure for estimating the single index vector during the initial phase in Algorithm \ref{algorithm: Initial_Dir_Estimation}.
\begin{algorithm}[ht]
\begin{algorithmic}[1]
\State \textbf{Input: } Number of samples in the initial phase $t_{\rm init}$, weights for each arm $(\omega_k)_{k=1}^K$, an SIR algorithm \textsc{SIR}$(\cdot)$
\For{$t = 1, \dots,t_{\rm init}$}
\State Pull arm ${k=(t-1) \bmod K+1}$.
\EndFor
\For{{$k=1,2,\dots,K$}}
\State Define the indices assigned to arm $k$: ${\tau^{(k)} = \{ 1 \leq t \leq t_{\rm init}; \, k=(t-1) \bmod K+1 \}}$
\State Compute $\hat{\beta}^{(k)} \leftarrow \textsc{SIR} ((X_t, Y_t^{(k)})_{t \in \tau^{(k)}})$
\EndFor
\State Compute the estimated projection matrix $\hat{\mathcal{P}} = \sum_{k=1}^{{K}} \omega_k \,\hat{\beta}^{(k)}(\hat{\beta}^{(k)})^\top$ of $\mathcal{P}_0$.
\State Return $\hat{\beta}$, the eigenvector corresponding to the largest eigenvalue of $\hat{\mathcal{P}}$.
\end{algorithmic}
\caption{Initial Direction Estimation}
\label{algorithm: Initial_Dir_Estimation}
\end{algorithm}

\section{Regret bounds}\label{sec: 4_regret_bounds}
In this section, we establish fundamental limits and achievable performance guarantees for the batched contextual bandit problem under a single-index model structure. We first derive a minimax lower bound that characterizes the optimal regret rates as a function of the number of batches $M$, {number of arms $K$,} and margin parameter $\alpha$. This lower bound reveals an inherent difficulty of the problem. We then analyze our proposed BIDS algorithm under the two scenarios, i.e., with and without a pilot estimate. When the pilot direction estimate is available with sufficient accuracy, our upper bound matches the lower bound up to log factors {for fixed $K$}, establishing minimax optimality {in this regime. When $K$ grows with $T$, however, the two bounds differ in their dependence on $K$, where the upper bound has polynomial dependence on $K$ while the lower bound has only logarithmic dependence.} When the pilot direction is unknown and needs to be estimated, the upper bound matches the lower bound under certain margin conditions, though a gap remains between upper and lower bounds in some ranges of the margin condition.
\subsection{Fundamental limits}
\label{sec: fundamental_limits}

Let $\mathcal{P}_X$ denote the collection of probability distributions $\P_X$ which satisfy Assumption \ref{assum: cond_X}. Let $${\mathcal{F}(\alpha; \,(\beta_0,\P_X)) := \{(f^{(1)},\dots,f^{(K)}); \, f^{(k)}\mbox{ satisfies Assumptions \ref{assum: Smoothness} and \ref{assum: Margin}}, \,\forall k\in [K]\}},$$  denote the class of reward function pairs satisfying Assumptions \ref{assum: Smoothness} and \ref{assum: Margin} for a given direction $\beta_0 \in \mathbb{S}^{d-1}$ and covariate distribution $\P_X$.

For ${k\in [K]}$, define $\P_{f(X)}^{(k)}(\cdot) = \P(Y^{(k)} \in \cdot \, | X)$, as the conditional distribution of $Y^{(k)}$ given $X$ with the conditional mean $\E[Y^{(k)}|X]=f(X)$.
We make the following assumption on the conditional distribution of $Y^{(k)}$ given $X$ which bounds the KL divergence between the two conditional distributions by the squared distance between their mean parameters. This KL divergence bound assumption is similar to Assumption (B) in Section 2.5 of \cite{tsybakov2008introduction}, and was originally proposed and used in \cite{rigollet2010nonparametric}. For example, this assumption is satisfied when $Y^{(k)}$ follows a Bernoulli distribution (see Lemma 4.1 in \cite{rigollet2010nonparametric}).

\begin{assumption}\label{assum: KL_bound}
There exists $\tau \in (0,1/2)$ such that for each $k\in {[K]}$, the family $\{\P_\theta^{(k)}, \,\, \theta \in [1/2-\tau, 1/2+\tau]\}$ satisfies the KL-divergence bound
\begin{align}
    \text{KL}(\P_\theta^{(k)} , \P_{\theta'}^{(k)}) \le \frac{1}{\kappa^2}(\theta-\theta')^2 \label{eq: KL_bound}
\end{align}
for some $\kappa>0$ and all $\theta,\theta' \in [1/2-\tau, 1/2+\tau]$.      
\end{assumption}

\begin{theorem}[Regret Lower Bound for Batched Global Multi-Armed Bandits with Covariates]\label{thm: lower_limit}
Suppose $0< \alpha \le 1$. Assume the conditional distributions of $Y^{(k)}$ given $X$, for ${k=1,2,\dots,K}$, satisfy Assumption \ref{assum: KL_bound}.
Then for any $M$-batch policy $\pi$ with prespecified batch endpoints $\mathcal{G} = \{t_0, t_1, \ldots, t_M\}$, where $0 = t_0 < t_1 < \cdots < t_M = T$, there exist reward functions {$(f^{(1)},\dots, f^{(K)})\in \mathcal{F}(\alpha;\,(\beta_0,\P_X))$}, a direction $\beta_0 \in \mathbb{S}^{d-1}$, and a covariate distribution $\P_X \in \mathcal{P}_X$ such that, the expected cumulative regret of $\pi$ satisfies
\begin{align*}
{
   \mathcal{R}_T(\pi) \ge C_1 (\log K)^{1-\beta_M}T^{\beta_M}, \qquad \text{where} \quad \gamma = \frac{\alpha + 1}{3}\mbox{ and }\beta_M = \frac{1-\gamma}{1-\gamma^M},}
\end{align*}
{where $C_1$ is a constant which only depends on model parameters $(\alpha, D_0, L, \overline{c}_X, \underline{c}_X, R_X)$ but not on $T$ or $K$.}
\end{theorem}

{When $K$ is treated as a fixed constant, the $\log K$ factor is absorbed into the implicit constant and the bound reduces to $\mathcal{R}_T(\pi)\gtrsim T^{\beta_M}$, which }coincides with the lower bound result derived for the fully nonparametric batched bandits setting in \cite{jiang2025batched} (Theorem 1), but when the dimension is $d= 1$. 
{We note that \cite{jiang2025batched} establish their lower bound only for the two arm case $K=2$. In contrast, our analysis covers general $K$, treating $K$ as a variable that can potentially grow slowly with $T$, and reveals an additional $(\log K)^{1-\beta_M}$ factor in the lower bound which reflects an increased difficulty in identifying the optimal arm among a set of larger alternatives.}
Our construction of the lower bound generally follows the framework and construction of hard instances from \cite{rigollet2010nonparametric} and \cite{jiang2025batched}, with some non-trivial modifications to adapt to our global batched bandit setting {with $K$ arms}.  We defer the proof to Section\ifnum\pageoption=2~\ref{sec: lower_bound_proof} \else~SM3.1 \fi in Supplementary Material.

\subsection{Regret upper bounds}
\label{sec: regret_upper_bound}
In this section, we discuss the regret bounds in two settings, when a pilot estimate of $\beta_0$ is available and when it is not.
First, recall that our adaptive binning is performed by partitioning the projected space, where the projection is based on the pilot index vector. As a result, the regret depends on how accurate the initial index vector is. To quantify this accuracy, we make the following assumption regarding the $\ell_2$-difference between the initial index $\beta$ and the true index $\beta_0$.

Since we are estimating the direction of $\beta_0$ rather than the vector itself, we quantify the distance in terms of the principal angle between two directions. More specifically, for $u,v \in \R$ such that $\|u\|_2 = \|v\|_2 = 1$, let $\angle u,v = \cos^{-1}(|u^\top v|) \in [0, \pi/2]$ be the principal angle between the directions $u$ and $v.$ 
Note that $\angle u,v = 0$ implies that $|u^\top v| = 1$, i.e., $u$ and $v$ are identical up to sign. At the other extreme, $\angle u,v = \pi/2$ implies that $|u^\top v| = 0$, which means $u$ and $v$ are orthogonal. Equivalently, we can express this in terms of the sine principal angle distance $\sin \angle u,v \in [0,1]$, where $\sin \angle u,v =0$ implies that $u,v$ are identical up to sign and $\sin \angle u,v = 1$ implies $u$ and $v$ are orthogonal.


\begin{assumption}\label{assum: initial_beta_rate}
The initial vector $\beta$ satisfies
\begin{align}\label{eq: initial_beta_sine}
   { \sin \angle \beta,\beta_0 \le \rho_{\rm pilot}}
\end{align}
{for some $\rho_{\rm pilot}\in [0,1]$}.
\end{assumption}
Note that Assumption~\ref{assum: initial_beta_rate} implies
there exists $o \in \{-1,1\}$ such that
    $\|\beta\cdot o - \beta_0\|_2 \le 2^{1/2}{\rho_{\rm pilot}}$ (see, e.g., proof of Lemma \ref{lem: estimated_beta_bound}).
For future reference, we define $\beta_{sgn} = \beta \cdot o$ which is either $\beta_{sgn} = \beta$ or $\beta_{sgn} = -\beta$ such that the above bound holds.
We note that $\beta_{sgn}$ is an oracle quantity since it depends on the unknown sign. It is used only in the proofs and is not required for the actual implementation of the algorithm.

\paragraph{Regret analysis when a pilot index is available} 
When a pilot direction satisfying Assumption \ref{assum: initial_beta_rate} is provided {with sufficiently small $\rho_{\rm pilot}$}, our regret analysis follows a similar approach to the adaptive binning with successive elimination method of \cite{perchet2013multi, jiang2025batched}, but with non-trivial modifications to accommodate the single-index (GMABC) model setting. 

We show that, with an optimal choice of batch size and splitting factor, our regret bound for Algorithm \ref{algorithm: SIRBatchedBinning} matches (up to logarithmic factors) with the lower bound in Theorem \ref{thm: lower_limit}, which is also the minimax rate of non-parametric batched contextual bandits but with $d=1$ (noting that their $\gamma$ depends on the covariate dimension $d$, meaning that their rate for $d>1$ is significantly slower than ours). To achieve this, we carefully select the batch size and splitting factors to ensure that the regret from one batch does not dominate the regrets from other batches. Specifically, we adopt the allocation rule and splitting factor setup proposed by \cite{jiang2025batched},  but with the choice of dimension $d=1$.

Recall that the list of split factors $\{b_i\}_{i=0}^{M-1}$ determines the number of bins $n_i = \prod_{l=0}^{i-1} b_l$ in the partition $\mathcal{A}_i$ of $[L_\beta,U_\beta]$ and the width $w_i=(U_\beta-L_\beta)/n_i$ of each bin in $\mathcal{A}_i$. Let $\gamma = \frac{ (1+ \alpha)}{3}$ and set {$a \asymp \{T/(K\log T)\}^{\frac{1-\gamma}{1-\gamma^M}}$}. The split factors are then chosen as follows:
\begin{align}
    b_0 = \lfloor a^{1/3} \rfloor, \ \text{and}\ \ b_i = \lfloor b_{i-1}^\gamma \rfloor,i = 1,\dots, M-2.\label{eq: bi_formula}
\end{align}
Note that this leads to the following choice of bin widths:
\begin{align}
    w_i \asymp (b_0 b_1 \dots b_{i-1})^{-1} \asymp b_0^{-(1+\gamma+\dots+\gamma^{i-1})} \asymp  b_0^{-\frac{1-\gamma^i}{1-\gamma} } , \ i= 1,\dots, M-1.\label{eq: wi_formula}
\end{align}
The number of samples allocated to batch $i$, i.e., $t_i-t_{i-1}$, is chosen so that it increases with the number of bins in the $i$th layer. Specifically, we let 
\begin{align}
    {t_i - t_{i-1} = \lfloor c_B  \,  K w_i^{-3} \log{(2KTw_i)}\rfloor,\,\, 1\leq i \leq M-1. }\label{eq: batch_size}
\end{align}
where $c_B = 4(4L_0+1)^{-2}(\overline{c}_X)^{-1}$, with ${L_0=L(2^{3/2}R_X+ 1)}$, is a constant independent of $T$.
{
These choices are made essentially to balance regret across batches. First, as we see from the proof of Theorem~\ref{thm: regret_bound}, for batch $i=1,\dots,M-1$, the per-batch regret is bounded by a term proportional to $(t_i - t_{i-1})w_{i-1}^{1+\alpha}$. We choose the batch size $t_i - t_{i-1}$ so that the elimination threshold $U$ (equation~\ref{def: U}) is of order $w_i$--large enough to detect arms whose gap exceeds the current resolution $w_i$, but small enough to avoid falsely eliminating near-optimal arms. This requires the expected number of visits to each bin to scale as $K\log(KT w_i)/w_i^2$, and combined with the fact that the probability of visiting each bin scales with $w_i$, forces $t_i - t_{i-1}\asymp Kw_i^{-3}\log (KTw_i)$ as in \eqref{eq: batch_size}. Substituting, the per-batch regret for batch $i$ scales as $K\log (KTw_i)w_i^{-3}w_{i-1}^{1+\alpha}\lesssim K\log (T)w_i^{-3}w_{i-1}^{1+\alpha}$, which varies across batches. We equalize the regret contribution across batches by setting $w_i^{-3}w_{i-1}^{1+\alpha}$ to be constant, since otherwise the total regret will be dominated by the worst batch. This leads to the choice of $b_i \asymp b_{i-1}^{\gamma}$ with $\gamma=(1+\alpha)/3$, making each per batch regret bounded by $K\log (T)b_0^3$. Finally, the base scale $b_0$ is determined by balancing these $M-1$ identical per-batch contribution bounds against the last batch regret $(T-t_{M-1})w_{M-1}^{1+\alpha}\asymp T b_0^{-3\gamma(\frac{1-\gamma^{M-1}}{1-\gamma})}$, leading to the choice of $b_0 \asymp (T/(K\log T))^{(1-\gamma)/(3(1-\gamma^M))}$.

{Finally, we note that this choice of the splitting schedule depends explicitly on the margin parameter $\alpha$ through $\gamma=(1+\alpha)/3$. Therefore, the current version of BIDS is tuned to a 
known margin parameter $\alpha$, rather than being adaptive to an unknown 
$\alpha$.  Designing an adaptive procedure that does not require prior knowledge of $\alpha$ is an interesting direction for future work.}
}

With these choices, we now present Theorem \ref{thm: regret_bound}, which establishes the regret bound for the proposed BIDS algorithm when the batch size $M$ is at most of order $\log (T) $. The proof is provided in Section\ifnum\pageoption=2~\ref{sec: regret_proof_Thm1} \else~SM3.2 \fi in Supplementary Material.
\begin{theorem}\label{thm: regret_bound}
Suppose Assumptions \ref{assum: Smoothness}--\ref{assum: cond_X} hold, and let a pilot direction $\beta$ with $\|\beta\|_2 = 1$ be given, satisfying Assumption \ref{assum: initial_beta_rate} {with accuracy $\rho_{\rm pilot}$}. 
{Suppose $\alpha \leq 1$, $M =O(\log{T})$, $K=O(\log T)$, and the pilot accuracy satisfies $\rho_{\rm pilot} = O(\{T/(K\log T)\}^{-1/3}).$}
For the BIDS algorithm $\pi$ described in Algorithm \ref{algorithm: SIRBatchedBinning}, with the choices of split factors and batch size satisfying \eqref{eq: bi_formula} and \eqref{eq: batch_size}, the following bound on the expected regret $\mathcal{R}_T(\pi)=\E[R_T(\pi)]$ holds for sufficiently large $T$:
\begin{align*}
{\mathcal{R}_T(\pi)
\;\le\;
C_2\, M (K\log T)^{1-\beta_M}T^{\beta_M},
\qquad
\text{where }
\gamma = \frac{1+\alpha}{3}\mbox{ and }\beta_M = \frac{1-\gamma}{1-\gamma^M}.}
\end{align*}
Here $C_2$ is a constant depending on model parameters such as $\alpha, D_0, L, \overline{c}_X, \underline{c}_X$, and $ R_X$, but not on {$T$ or $K$}.
\end{theorem}
Theorem \ref{thm: regret_bound} shows that the BIDS Algorithm, when provided with a sufficiently accurate pilot estimate, achieves near-optimal regret performance across different batch regimes.
{When $K$ is treated as a fixed constant, the expected regret upper bound in Theorem \ref{thm: regret_bound} matches the lower bound in Theorem \ref{thm: lower_limit} up to logarithmic factors. When $K$ is allowed to grow with $T$, however, the upper bound exhibits polynomial dependence on $K$, whereas the lower bound contains only logarithmic dependence on $K$.
Notably, for a fixed $K$, this rate coincides with the $d=1$ minimax-optimal rate for nonparametric batched bandits established for the two-arm case in Theorem 1 in \cite{jiang2025batched}, thereby avoiding the curse of dimensionality.}
\begin{remark}
The regret bound in Theorem~\ref{thm: regret_bound} exhibits a polynomial dependence on the number of arms $K$. Such dependence is common in contextual bandit models where each arm is associated with its own reward function or parameter, sometimes referred to as \emph{multi-parameter} bandit models \cite{sivakumar2020structured}. In these settings the learner must estimate the reward structure of all $K$ arms simultaneously, which typically leads to polynomial dependence on $K$ in regret bounds. Similar behavior appears in  nonparametric bandits \cite{perchet2013multi,qian2016kernel} as well as linear contextual bandit analyses \cite{bastani2021mostly,dimakopoulou2019balanced}.\\
When the number of batches $M \gtrsim \log\log{T}$, the bound in Theorem~\ref{thm: regret_bound} simplifies to,
\begin{align*}
\mathcal{R}_T(\pi)\lesssim (K\log T)^{(1+\alpha)/3}T^{1-(1+\alpha)/3},
\end{align*}
yielding a sub-linear polynomial dependence of order $K^{(1+\alpha)/3}$. This matches the dependence obtained for adaptive binning and successive elimination methods such as ABSE \cite{perchet2013multi} in the fully online setting with $d=1$, suggesting that the $K$ scaling of BIDS is consistent with existing binning-based contextual bandit algorithms while benefiting from the shared single-index structure that reduces the effective covariate dimension.\\
Finally, the lower bound in Theorem~\ref{thm: lower_limit} contains only a logarithmic dependence on $K$. In the nonparametric contextual bandit literature, the primary focus has typically been on the dependence of regret on the time horizon $T$ and the smoothness of the reward functions, while the number of arms $K$ is often treated as a fixed constant. For instance, the upper bound of ABSE was compared with a lower bound derived for $K=2$~\cite{perchet2013multi}. 
Under this interpretation, the proposed regret bound achieves the optimal rate in $T$.\\
When $K$ is allowed to grow with $T$, however, the optimal dependence on $K$ is, to our knowledge, less well understood. Our analysis allows $K$ to grow slowly with $T$, for example up to $K = O(\log T)$, since the concentration bounds appearing in the regret analysis involve logarithmic factors of the form $\log(KT)$, which remain of order $\log T$ in this regime. The polynomial dependence on $K$ appearing in the upper bound reflects the binning and elimination procedure used by the algorithm, which requires collecting sufficient samples within each bin to reliably compare all candidate arms. Whether this dependence can be improved, or whether stronger lower bounds that capture the dependence on $K$ can be established, remains an interesting open question that we leave for future work.
\end{remark}

\begin{remark}[Effect of pilot accuracy]
{The threshold $\rho_{\mathrm{pilot}} = O(\{T/(K\log T)\}^{-1/3})$ 
in Theorem~\ref{thm: regret_bound} arises because the batch elimination procedure requires 
the within-bin reward variation to be controlled at the scale of the 
bin width $w_i$. Specifically, by Lemma~\ref{lem: barg_min_g}, for any 
$x, y \in C \in \mathcal{B}_i$,
$|g^{(k)}(x) - g^{(k)}(y)| \leq L\{w_i + 2^{3/2}R_X\rho_{\mathrm{pilot}}\}$,
so controlling this variation at scale $w_i$ requires 
$\rho_{\mathrm{pilot}} \lesssim w_i$ for all $i$. Since the finest 
bin width satisfies $w_{M-1} \asymp \{T/(K\log T)\}^{-1/3}$, this 
yields the stated threshold. 
When $\rho_{\mathrm{pilot}}$ exceeds 
this threshold, the within-bin bias may no longer be sufficiently controlled. As a result, the 
good elimination event $\mathcal{S}_C$ in~\eqref{def: SC} may fail more frequently, and therefore the 
regret bound in Theorem~\ref{thm: regret_bound} may no longer hold with high probability. This degradation in performance is also illustrated empirically in Figure~\ref{fig:simulation_knownpilot}, where 
increasing perturbation angles lead to higher regret, although the method appears to be robust to moderate perturbations. A full 
$\rho_{\mathrm{pilot}}$-dependent regret bound is left for future 
work.}
\end{remark}

\paragraph{Regret analysis when no pilot estimate is available}
When no pilot index estimate is available, both the index vector and the link function must be estimated within the batches. We propose using the first batch to estimate $\beta$ (Algorithm \ref{algorithm: Initial_Dir_Estimation}), then performing the BIDS algorithm with the estimated index vector $\beta$ for the remaining batches (Algorithm \ref{algorithm: SIRBatchedBinning}). 

Recall that in the initial phase, for $t\in \{1,\dots,t_{\rm init}\}$, we draw i.i.d. random samples from each arm.  Any suitable single-index model can then be applied in this phase to estimate the index vector.  The index vector can generally be estimated at a parametric rate, e.g., \cite{li1989regression,babichev2018slic, kuchibhotla2020efficient}. Assumption \ref{assum: index_rate} specifies the requirement for the index vector from a Single-Index Regression (SIR) method used in Algorithm \ref{algorithm: Initial_Dir_Estimation}. Specifically, we require that the SIR algorithm used in Algorithm \ref{algorithm: Initial_Dir_Estimation} produces an estimate that satisfies a parametric error bound up to a log term with high probability when applied to an i.i.d dataset of size $n_k$. 
\begin{assumption}\label{assum: index_rate}
Let $k\in {[K]}$ be fixed, and let $\hat{\beta}^{(k)}$ be the estimated vector from an i.i.d sample of size $n_k$, $(x_i,Y_i^{(k)})_{i=1}^{n_k}$ where $Y_i^{(k)}$ follows the single index model \eqref{def:sim}. 
{For a sufficiently large $n_k$, for any $\delta \in (0,1/n_k)$, the following bound holds with probability at least $1-\delta$:
\begin{align*}
  \sin \angle \hat{\beta}^{(k)}, \beta_0 \le   C_{\rm idx}\frac{\textrm{polylog}(n_k/\delta)}{\sqrt{n_k}}.
\end{align*}
for some constant $C_{\rm idx}>0$ which can depend on model parameters but is independent of the sample size $n_k$.
}
\end{assumption}
\begin{remark}
As an example of a single index estimation algorithm that satisfies Assumption \ref{assum: index_rate}, we discuss the Sliced Average Derivative Estimator (SADE) of \cite{babichev2018slic} in Section\ifnum\pageoption=2~\ref{sec: SADEappendix} \else~SM4 \fi in Supplementary Material.  In particular, Theorem\ifnum\pageoption=2~\ref{thm: SADE_estimation} \else~SM4.1 \fi establishes that, under mild conditions,  the estimates $\hat{\beta}^{(k)}$ obtained using the SADE method satisfy Assumption \ref{assum: index_rate}.  Please see  Supplementary Material\ifnum\pageoption=2~\ref{sec: SADEappendix} \else~SM4 \fi for more details.
\end{remark}

The following Lemma \ref{lem: estimated_beta_bound} shows that under Assumption \ref{assum: index_rate}, the estimated direction $\hat{\beta}$ from Algorithm \ref{algorithm: Initial_Dir_Estimation} is (up to sign) within a neighborhood of $\beta_0$ that shrinks at an approximate rate of ${\sqrt{K/t_{\rm init}}}$, with an additional log term.
\begin{lemma}\label{lem: estimated_beta_bound}
Let {$\hat{\beta}^{(1)},\dots,\hat{\beta}^{(K)}$} be the estimated index vectors from each arm, and let $\hat{\beta}$ be the final estimated direction from Algorithm \ref{algorithm: Initial_Dir_Estimation}. Suppose Assumption \ref{assum: index_rate} holds for each {$k=1,2,\dots,K$}. For sufficiently large $T$, {for any $\delta \in (0, K/(2t_{\rm init})]$, the following inequality holds for all $k\in [K]$ with probability at least $1-K\delta$:} 
\begin{align}\label{eq: estimation_rate_for_beta_wi}
  { \sin \angle \hat{\beta},\beta_0  \le \tilde{C}_{\rm idx} \frac{\sqrt{K}\textrm{polylog}(t_{\rm init}/(K\delta))}{\sqrt{t_{\rm init}}},}
\end{align}
for a constant $\tilde{C}_{\rm idx}$.
Moreover, there exists $\hat{o} \in \{-1,1\}$ such that
\begin{align}\label{eq: init_vec_bound}
	{\|\hat{\beta} \cdot \hat{o} - \beta_0 \|_2\le 2^{1/2}\tilde{C}_{\rm idx} \frac{\sqrt{K}\textrm{polylog}(t_{\rm init}/(K\delta))}{\sqrt{t_{\rm init}}}.}
\end{align}
\end{lemma}

The proof for Lemma \ref{lem: estimated_beta_bound} is provided in Section\ifnum\pageoption=2~\ref{sec: proof_lem_estimated_beta_bound} \else~SM3.3 \fi in Supplementary Material.
In terms of regret bound analysis, the primary difference in this setting compared to the previous one is that regret will accrue from the observations drawn during the initial phase.
In particular, the cumulative regret incurred is given by,
\begin{align}
    \mathcal{R}_T(\pi) & = \E[\sum_{t=1}^T g^{(*)}(X_t) - g^{(\pi_t(X_t))}(X_t)] \nonumber\\
    &= \E\left[\sum_{t=1}^{t_{\rm init}} (g^{(*)}(X_t) - g^{(\pi_t(X_t))}(X_t))  + \sum_{t=t_{\rm init} + 1}^T (g^{(*)}(X_t) - g^{(\pi_t(X_t))}(X_t))\right] \nonumber\\
    & \leq 2t_{\rm init} + \E \left[\sum_{t=t_{\rm init} + 1}^T (g^{(*)}(X_t) - g^{(\pi_t(X_t))}(X_t))\right] \label{eq: regret_phase1_bd_phase2}\\
    & =: 2t_{\rm init} + \mathcal{R}_{T-t_{\rm init}}(\pi;\beta). \nonumber
\end{align}
where \eqref{eq: regret_phase1_bd_phase2} follows from the fact that $|Y_t| \le 1$. 

The size of the first batch $t_{\rm init}$ needs to be chosen to balance two competing factors: achieving sufficient accuracy in estimating the single-index parameter while not incurring too much regret.
{Theorem~\ref{thm: regret_bound} requires the working direction $\beta$ to be within a {$\{T/(K\log T)\}^{-1/3}$} neighborhood of $\beta_0$, up to sign. Therefore, to ensure that the estimated direction $\beta$ is sufficiently accurate, we consider the initial phase length as $t_{\rm init} \asymp  K^{1/3}{\rm polylog} (T)T^{2/3}$ where the polylogarithmic factor is taken with sufficiently large exponent so that
\begin{equation}\label{eq: tinit_bound}
    \frac{\sqrt{K} \textrm{polylog}(T)}{\sqrt{t_{\rm init}}} \lesssim \left(\frac{K\log T}{T}\right)^{1/3}.
\end{equation}
}

\begin{theorem} \label{thm: thm2}
Suppose Assumptions \ref{assum: Smoothness}--\ref{assum: cond_X} hold. 
Also, assume that the estimates from Algorithm \ref{algorithm: Initial_Dir_Estimation} satisfy Assumption \ref{assum: index_rate}.
Let $\alpha \leq 1$, {$K = O(\log T)$}, and $M = O(\log{T})$. 
Consider the algorithm $\pi$, which first executes Algorithm \ref{algorithm: Initial_Dir_Estimation} {during the initialization batch} of size $t_{\mathrm{init}} \asymp K^{1/3}\mathrm{polylog}(T)T^{2/3}$, followed by Algorithm \ref{algorithm: SIRBatchedBinning} {over the remaining $M-1$ batches using the binning and batch schedule from \eqref{eq: wi_formula} and \eqref{eq: batch_size}. }
Then, the expected regret for the resulting algorithm $\pi$ is upper bounded by,
{
\begin{align*}
    \mathcal{R}_T(\pi) \leq C_3\, {\rm polylog} (T) \max\{K^{1/3}T^{2/3} , K^{1-\beta_M}T^{\beta_M}\},
\quad
\text{where }
\gamma = \frac{1+\alpha}{3}\mbox{ and }\beta_M = \frac{1-\gamma}{1-\gamma^M}.
\end{align*}
}
Here $C_3>0$ is a constant that depends only on the single index parameter $\beta$ and other constants such as $\alpha,  D_0, L, R_X, \overline{c}_X, \underline{c}_X$.
\end{theorem}
The proof is deferred to Section\ifnum\pageoption=2~\ref{sec: regret_proof_thm2} \else~SM3.4 \fi in Supplementary Material.

\begin{remark}
Compared to the bound in Theorem \ref{thm: regret_bound}, the bound in Theorem \ref{thm: thm2} reflects an additional price for not knowing the pilot index. {The first term $K^{1/3}T^{2/3}$ arises from the initial phase needed to estimate the index vector with sufficient accuracy, while the second term $K^{1-\beta_M}T^{\beta_M}$ matches the known-pilot bound up to a polynomial factor.}

{In terms of dependence on $T$}, however, in certain problem instances, we can still achieve the same rates as those in Theorem \ref{thm: regret_bound}. In Theorem \ref{thm: thm2}, it is easy to note that the second term dominates when $2/3 \le \beta_M=\frac{1-\gamma}{1-\gamma^M}$, 
which simplifies to
\begin{align}
(1+\alpha)^M - (3^M\alpha)/2 \ge 0.\label{eq: condition_on_xi_thm2}
\end{align}
This implies that, for example, when the number of batches after the initial batch is $M=2$,  the rate in Theorem \ref{thm: thm2} matches with that of Theorem \ref{thm: regret_bound} for $0<\alpha \le 0.5$.
The range of $\alpha$ for which the rate without a pilot estimate matches with the rate with a pilot estimate becomes smaller as the number of batches increases. 
For instance, when $M$ is large enough that $\gamma^M \approx 0$, only $\alpha\approx 0$ satisfies \eqref{eq: condition_on_xi_thm2}. That is, the rate without a pilot estimate is optimal only under the near-zero margin condition.
At the other extreme, when $\alpha=1$, the regret grows as $\tilde{O}(T^{2/3})$, whereas when the pilot estimate is known (as in Theorem \ref{thm: regret_bound}), the regret grows as $\tilde{O}(T^{1/3})$. 

{
One possible explanation for this gap lies in the algorithmic structure of BIDS. 
The batched elimination procedure relies on a hierarchical binning scheme constructed along a fixed projection direction, where bins are refined across batches through a nested parent–child partition. When the index direction is unknown, we must first estimate it using an initialization phase before the binning procedure can be applied. This requires an initial sample size large enough to ensure that the projection error is smaller than the spatial resolution of the binning algorithm. As a result, the initialization batch size must scale as 
$t_{\rm init} \asymp K^{1/3}\,\mathrm{polylog}(T)\,T^{2/3}$, which in turn introduces the additional $\widetilde O(K^{1/3}T^{2/3})$ term in the regret bound. An adaptive scheme that re-estimates $\beta_0$ at each batch using all samples accumulated up to that point, matching the estimation error to the current bin resolution rather than the finest, could potentially reduce this initialization cost; however, updating the projection direction across batches breaks the nested partition structure, making such an extension nontrivial. We leave this as an interesting direction for future work.
Nevertheless, it is encouraging to note that even without a pilot estimate, BIDS still achieves a sub-linear regret corresponding to the effective dimension $d=1$}, despite using some initial data to estimate $\beta_0$. 

\end{remark}

\section{Simulation Study} \label{sec: 5_simulation}
In this section, we present numerical experiments to illustrate the performance of the proposed BIDS algorithm (Algorithm \ref{algorithm: SIRBatchedBinning}) in comparison to the nonparametric analogue: Batched Successive Elimination with Dynamic Binning (BaSEDB) algorithm of \cite{jiang2025batched}. We consider both cases discussed in Section \ref{sec: regret_upper_bound}: 1) when the pilot direction is available under varying degrees of accuracy, and 2) when the pilot direction is unknown and estimated using the initial $t_{\rm init}$ amount of data, under varying signal-to-noise level settings.

\paragraph{Simulation settings}
We first consider $K=2$ arm setting, where the mean reward functions $g^{(1)}$ and $g^{(2)}$ follow a single index model structure with the shared parameter $\beta_0\in \R^d$, i.e., 
\begin{align*}
    g^{(k)}(x) = f^{(k)}(x^\top \beta_0),\,\,k=1,2,
\end{align*}
where $f^{(1)},f^{(2)}:[l,u] \to \R$ are link functions for arm $1$ and $2$, with $d=5$ fixed throughout. 

First, the index vector $\beta_0$ is simulated by generating a scaled normal random vector. Specifically, we first draw $u \sim N_d(0,I_d)$ and then let $\beta_{0} = u / \|u\|_2$. 
For the covariate distribution, we let each $X_t \in \R^d$ follow a truncated multivariate normal distribution for $t=1,\dots,T$, i.e., $X_t \sim N_{T}(0,\Sigma_X)$
whose density is given by: 
\begin{align*}
    f_{X}(x) = \begin{cases}
        \frac{1}{Z(\Sigma_X)} \exp\{-\frac{1}{2}x^\top\Sigma_X^{-1}x\}   & x\in \mathcal{H}\\
        0 & \text{otherwise,}
    \end{cases}
\end{align*}
with $\Sigma_X = 5^2I_d$. The normalization constant $Z(\Sigma_X)$ is given by $Z(\Sigma_X) =\int_{x\in \R^d} e^{-\frac{1}{2}x^\top\Sigma_X^{-1}x} 1\{x\in \mathcal{H}\} dx $ with the truncation region $\mathcal{H} = \prod_{j=1}^d 1\{|x_j|\le 3\}$. 
Additionally, we have considered other covariate distributions, including the Normal distribution without truncation and the uniform distribution. The results were qualitatively similar to those presented below for the truncated normal case and are presented in Section\ifnum\pageoption=2~\ref{sec: additional_simulations} \else~SM5.1 \fi in Supplementary Material.

To define $1$-dimensional link functions, first let us define,
\begin{align} \label{eq: fx_simulation}
    f(x) &= a+\frac{2}{B}\sum_{j=1}^{B/2}v_j \,\phi\left(\frac{B}{u-l}(x-q_j)\right),
\end{align}
where $q_j = l + (2j-1) \frac{u-l}{B}$ for $j=1,\dots,B/2$, 
$\phi(x) = (1-|x|) 1\{|x|\le 1\}$, 
$v_j$ for $j=1,\dots,B/2$ are Rademacher random variables ($v_j \in \{-1,1\}$), and $l,u = \mp 3\sqrt{d}$.

We consider two simulation settings for the link functions as illustrated in Figure \ref{fig:simulation_settings}.\\
 \textbf{Setting 1: }
$
    f^{(1)}(x) = f(x) \ \text{with} \ a = 0.5, B = 8, \ \text{and} \ 
    f^{(2)}(x) = \frac{1}{2} + x.
$ \\
\textbf{Setting 2: }
$   f^{(1)}(x) = f(x)\ \text{with} \ a = 0.5, B = 8, \ \text{and} \ 
    f^{(2)}(x) = f(x)\  \text{with} \ a = 0.75, B = 5.$
\begin{figure}[htbp]
    \centering
    \includegraphics[width=0.4\linewidth]{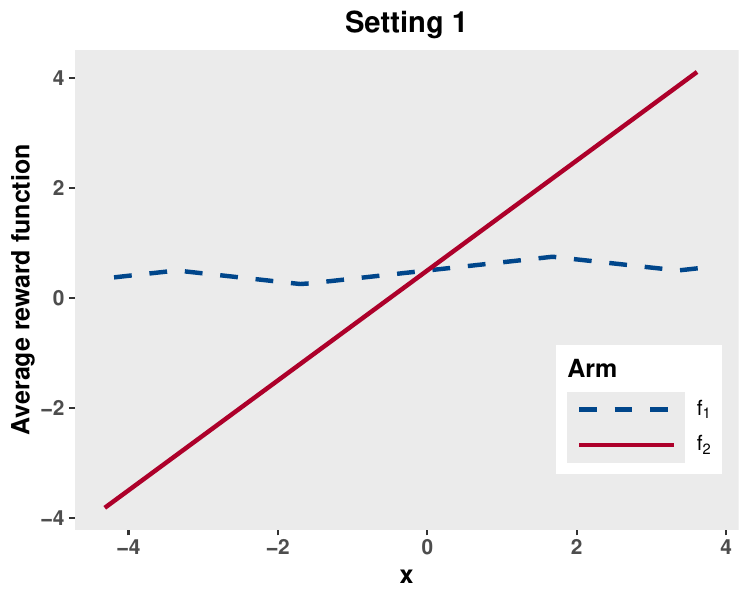}
    \includegraphics[width=0.4\linewidth]{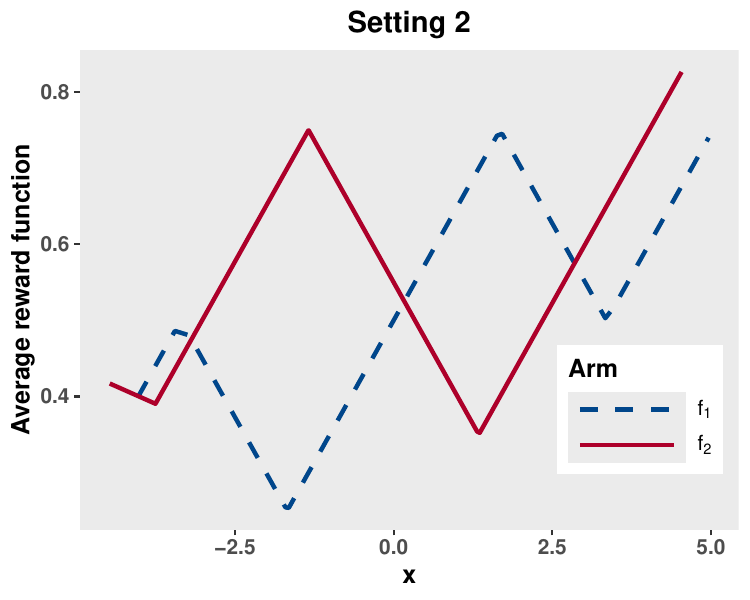}
    \caption{Mean reward functions for the two simulation settings}
    \label{fig:simulation_settings}
\end{figure}

We let $Y_t^{(k)} = f^{(k)}(X_t) + \epsilon_t$, where $\epsilon_t \overset{\text{i.i.d.}}{\sim} N(0, \sigma^2)$ for $t=1,\dots,T$, with $\sigma^2 > 0$, representing the noise variance.
In the first case, where we test the performance of the BIDS algorithm with varying accuracies of pilot directions, we set $\sigma^2 = 0.01^2$. In the second case, where we estimate the initial direction under different noise levels, we set $\sigma \in \{1,\dots,8\}$ for setting 1 and $\sigma \in \{0.02,0.09,0.16,\dots,1\}$ for setting 2, with time horizon $T = 10^6$. 

To further illustrate how the number of arms affects the performance of the BIDS algorithm, we also consider additional experiments where $K \in \{3,5,8\}$ while keeping the remaining simulation setup identical to Setting~1. The corresponding reward functions follow the same single-index construction as described earlier; examples of these multi-arm reward profiles are shown in Figure\ifnum\pageoption=2~\ref{fig:Kgr2_rewardfuncs} \else~SM10 \fi and more details can be found in Section\ifnum\pageoption=2~\ref{sec: simulation_multiarm} \else~SM5.2 \fi in Supplementary Material.

\paragraph{Algorithm set-ups} 
Both BIDS and BaSEDB algorithms require specifying the number of batches $M$ and the grid points $\{t_i\}_{i=0}^M$. {In both cases, we set the total number of batches to ($M = 5$). When the pilot direction is unknown, the first batch serves as an initialization phase, and BIDS is applied to the remaining four batches.} For the BaSEDB algorithm, we follow the specifications described in \cite{jiang2025batched} for choosing grid points. For the BIDS algorithm (Algorithm \ref{algorithm: SIRBatchedBinning}), in the first case with known pilot directions, we make grid point choices according to \eqref{eq: bi_formula} and \eqref{eq: wi_formula}, and in the second case with unknown pilot directions, the initial batch size is set to $T^{2/3}$, and the remaining time points are partitioned according to the same rules. 
In addition, in the latter case, Algorithm \ref{algorithm: Initial_Dir_Estimation} requires specifying an SIR algorithm and arm weights. For the SIR algorithm, we use the SADE estimator from \cite{babichev2018slic}  (Algorithm\ifnum\pageoption=2~\ref{algorithm: SADE} \else~SM4.1 \fi in Supplementary Material) and we used equal arm weights $\omega_k = 1/2, k = 1,2$ for combining directions from each arm. 
Additionally, both algorithms require specifying the endpoints for hierarchical partitioning: $[L_\beta,U_\beta]$ such that $L_\beta \le x^\top \beta \le U_\beta$ for the BIDS algorithm, and $[L,U]$ such that $L \le x_j \le U$ for all $j=1,\dots,d$ for the BaSEDB algorithm. We constructed these intervals based on the observed minimum and maximum values from i.i.d. samples for each arm in the first batch, and expanded them by 20\%.  More specifically, we obtained the minimum $a$ and maximum $b$, where $a = \min_{t\in (t_0,t_1]}x_t^\top \beta$ and $b = \max_{t\in (t_0,t_1]}x_t^\top \beta$ in BIDS algorithm and $a = \min_{t\in (t_0,t_1]} \min_{1\le j\le d}x_{tj}$ and $b = \max_{1\le j\le d}x_{tj}$ in BaSEDB algorithm.  The interval was then set as $[\frac{a+b}{2}-\frac{C(b-a)}{2}, \frac{a+b}{2}+\frac{C(b-a)}{2}]$ with $C=1.2$.

\paragraph{Results}
We run each algorithm 20 times and report the average regret  in Figures \ref{fig:simulation_knownpilot} and \ref{fig:simulation_unknownpilot} for the two settings. Batch endpoints are marked by the vertical solid black (SIR) and dashed blue (nonparametric)  lines in both figures.

\vspace{0.5em}
\noindent \textbf{Case I (given pilot directions with varying accuracies) } 
In this set-up, we compare the performance of BIDS and BaSEDB when a pilot direction is available with varying levels of accuracies. 
Specifically, we set the initial index parameters $\beta$ for the BIDS algorithm so that $\theta = \angle \beta,\beta_0  \in  \{0.01, 0.16, 0.31 \dots, \pi/2\}$. The corresponding $\sin (\theta)$ ranges from $0$ to $1$, where, $\sin(\theta)=0$ implies that $\beta$ is identical to $\beta_0$ up to a sign change, and $\sin(\theta)=1$ implies that the two vectors are orthogonal. 
\begin{figure}[htbp]
    \centering
    \begin{tabular}{cc}
         (a) & (b) \\
        \includegraphics[width = 0.4\linewidth]{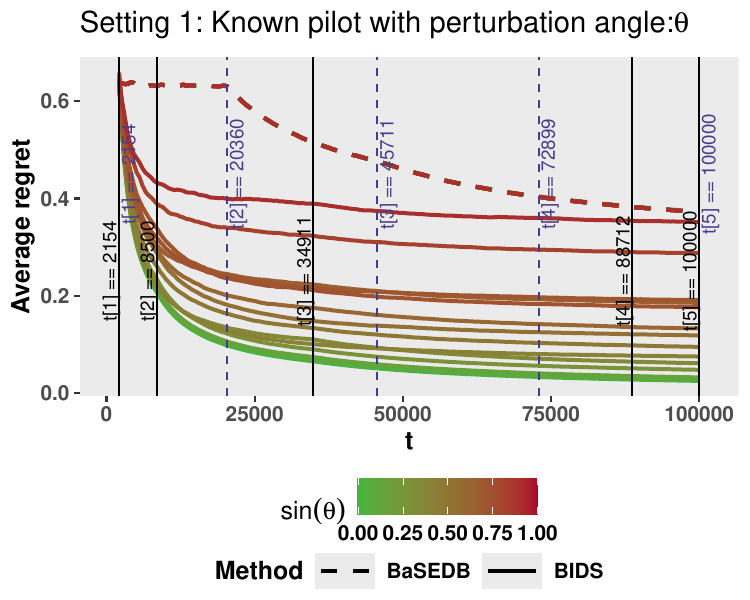} & \includegraphics[width = 0.4\linewidth]{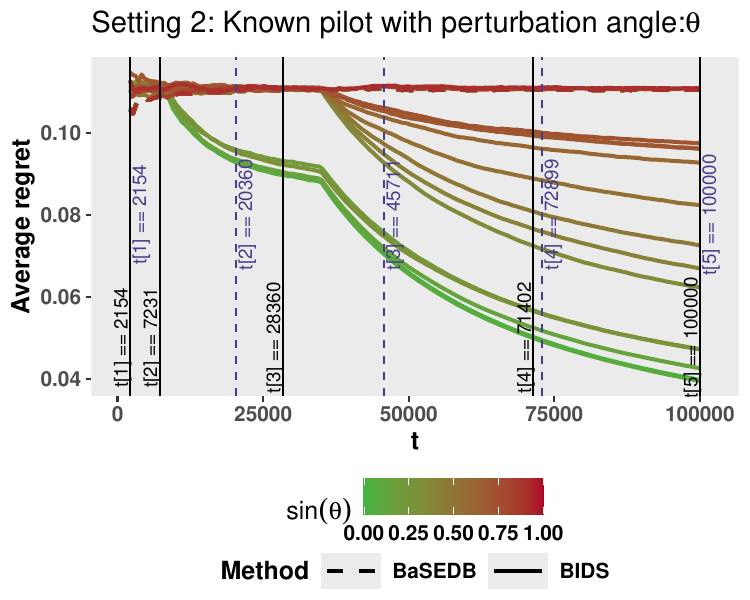}
    \end{tabular}
\caption{Average regret ($(\mathcal{R}_t)_{t=1}^T$) with pilot directions $\beta$ with varying accuracy, measured by $\sin \theta = \sin \angle \beta,\beta_0$ for the two simulation settings. Different colors of the solid lines represent different levels of perturbation, where $\sin\angle \beta, \beta_0 =0$ corresponds to no perturbation, and $\sin\angle \beta, \beta_0 =1$ corresponds to orthogonal vectors. As the degree of perturbation increases, performance deteriorates but still beats the nonparametric analogue.}
    \label{fig:simulation_knownpilot}
\end{figure}

Figure \ref{fig:simulation_knownpilot} presents the average regrets of the BIDS algorithm with pilot directions of varying accuracies, compared to BaSEDB algorithm. 
As the perturbation level increases, the performance of the BIDS algorithm with the perturbed pilot estimate declines. However, it consistently outperforms the nonparametric batched bandit algorithm (BaSEDB), even under high perturbations. Interestingly, in Figure \ref{fig:simulation_knownpilot}(b), we observe that in Setting 2—where the two mean reward functions exhibit greater overlap—the BaSEDB algorithm never eliminates an arm. Consequently, its average regret (dashed red line) does not decay over time.  Moreover, when the perturbation angle exceeds $\pi/3$ in Settings 1  and $\pi/4$ in Setting 2, BIDS performance  deteriorates to the level of its nonparametric counterpart.

\vspace{0.5em}
\noindent \textbf{Case II (no pilot directions) }  
For the case when the pilot estimate is not available, in Figure \ref{fig:simulation_unknownpilot}, we assess the algorithmic performance for varying degrees of model noise $\sigma$.  We also included BIDS (oracle), which uses the true $\beta_0$ as the initial direction. 

\begin{figure}[htbp]
    \centering
    \begin{tabular}{cc}
    (a) & (b)\\
         \includegraphics[width = 0.4\linewidth]{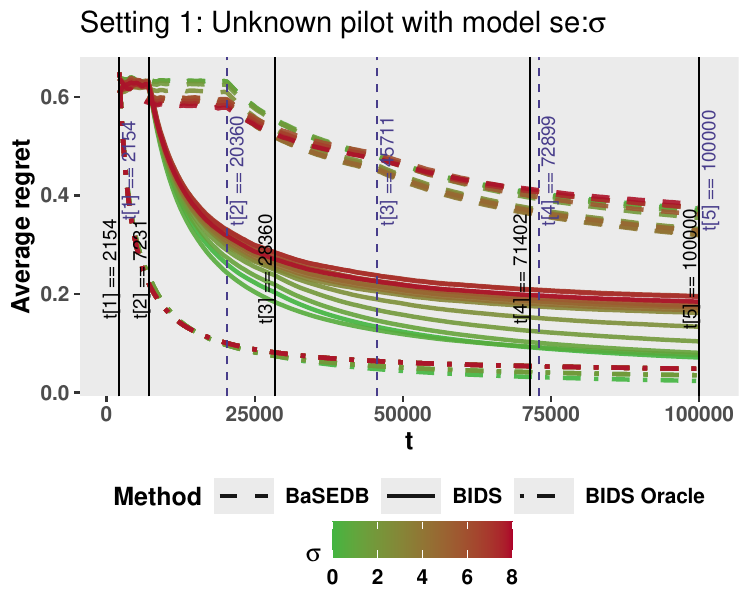}    & \includegraphics[width=0.4\linewidth]{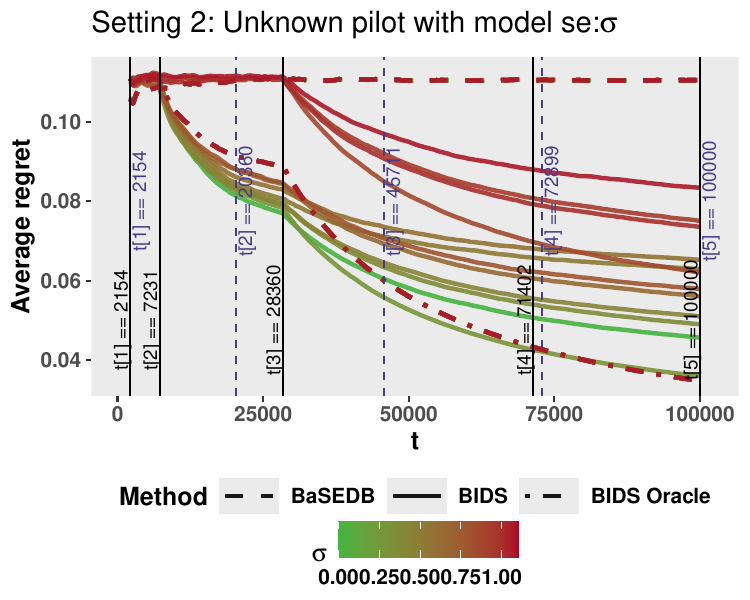} 
    \end{tabular}
\caption{Average regret ($(\mathcal{R}_t)_{t=1}^T$) with varying model noise $\sigma$ for the two simulation settings. As the noise level increases, while the performance of the BIDS algorithm (solid) remains better than the nonparametric analogue (dashed), but deviates further from the BIDS oracle (dashed-dotted). }
    \label{fig:simulation_unknownpilot}
\end{figure}

Note that in setting 1, the two mean reward functions are well-separated, while in setting 2, they have more of an overlap in various regions. Therefore, even with higher model error in setting 1, it is easier to maintain low regret as can be seen in Figure \ref{fig:simulation_unknownpilot}(a).  We consider the standard deviation to be ranging from $\sigma \in \{1,2,\hdots,8\}$ for setting 1 while $\sigma \in \{0.02,0.09,0.16,\hdots,1\}$ for setting 2. 
From Figure \ref{fig:simulation_unknownpilot}, we see that in both settings, the BIDS algorithm appears to perform better than the BaSEDB algorithm for all the noise variance levels. As expected, the performance of the BIDS algorithm (solid) as compared to the oracle BIDS algorithm (dotted-dashed) deteriorates as the noise grows, as the higher noise levels reduce the accuracy of the initial direction vectors.
\begin{figure}[htbp]
\centering
\begin{tabular}{cc}
(a) & (b) \\
\includegraphics[width=0.4\linewidth]{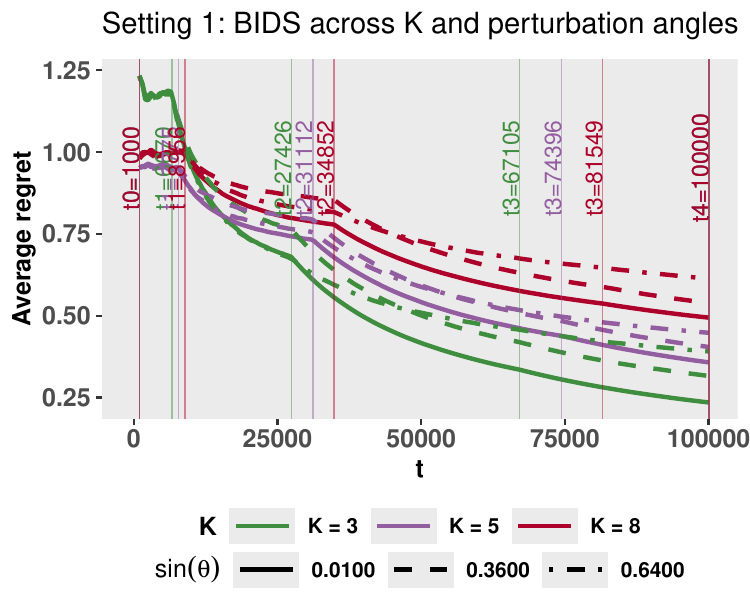} &
\includegraphics[width=0.4\linewidth]{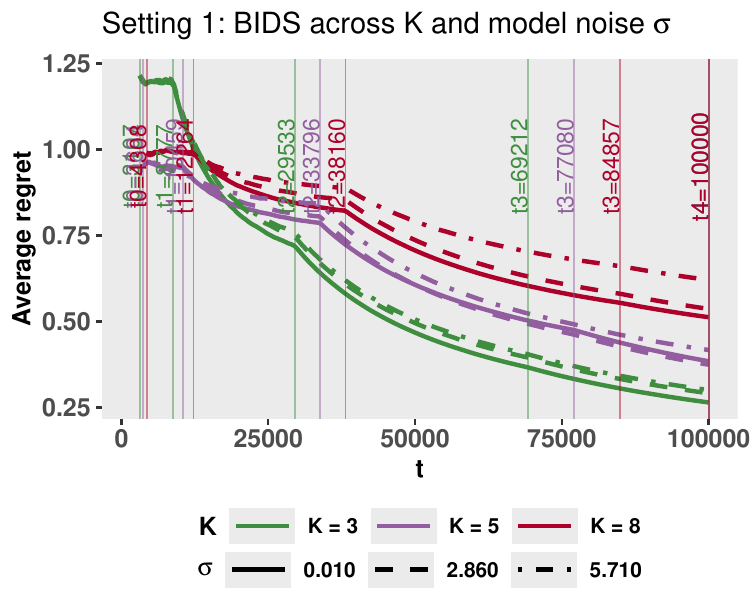}
\end{tabular}
\caption{Average regret of the BIDS algorithm for $K\in\{3,5,8\}$ in Setting 1. 
(a) Known pilot direction with varying perturbation angles $\theta$. 
(b) Unknown pilot direction estimated from data under different model noise levels $\sigma$. 
In both panels, colors represent different numbers of arms $K$, while line types correspond to the perturbation levels $\theta$ in (a) and noise levels $\sigma$ in (b). Vertical lines indicate the batch endpoints of the BIDS algorithm.}
\label{fig:simulation_K_scaling}
\end{figure}

\vspace{0.5em}
\noindent \textbf{Case I and II when the number of arms $K>2$} 
{When the number of arms $K>2$, Figure~\ref{fig:simulation_K_scaling} illustrates the performance of the BIDS algorithm for $K\in\{3,5,8\}$ under Setting 1. The qualitative conclusions were similar under Setting 2 for $K > 2$ arms as shown in the Appendix. In panel (a), we consider the case where a pilot direction is available but perturbed, and vary the angle $\theta=\angle(\beta,\beta_0)$ between the supplied pilot direction $\beta$ and the true index $\beta_0$. {Within the range of $K$ considered, the observed regret increases with $K$, which is consistent with the predicted $K$-dependence of the regret bound. This reflects the increased difficulty of identifying the optimal arm, as well as the larger batch sizes prescribed by the algorithm.} Nevertheless, the regret curves continue to decay over time for all values of {$K \in \{3,5,8\}$}. Smaller perturbation angles lead to faster regret decay, while larger perturbations slow learning but still yield sublinear regret.}
Panel (b) considers the setting where the pilot direction is unknown and $\beta_0$ is estimated using the initialization procedure with batch size $t_{\rm init}\asymp K^{1/3}\mathrm{polylog}(T)T^{2/3}$. We vary the model noise level $\sigma$ while keeping the rest of the simulation setup identical to Setting~1. 
As $\sigma$ increases, the regret curves shift upward due to reduced accuracy in estimating the index direction, but the regret continues to decrease over time across all values of {$K \in \{3,5,8\}$}.

\begin{remark}[Computation considerations]
In terms of computation, the GMABC framework and the BIDS algorithm have a key advantage over the BaSEDB algorithm, as the number of bins that needs to be tracked does not grow with the covariate dimension. In contrast, the number of bins in BaSEDB algorithm grows exponentially with the covariate dimension, making implementation challenging even for moderately large dimensions.

\end{remark}

\section{Application to Real Data} \label{sec: 6_realdata}
We compare the performance of the batched single-index and batched nonparametric BaSEDB algorithm on three publicly available real datasets:
\begin{enumerate}
    \item[a)] Rice classification \cite{rice_(cammeo_and_osmancik)_545}: Classifying rice into two Turkish varieties, namely, Cammeo and Ormancik, using 7 morphological features extracted from 3810 rice grain's images. 
    \item[b)] Occupancy Detection \cite{occupancy_detection__357}: Experimental data used for binary classification (room occupancy) from Temperature, Humidity, Light and $CO_2$.
    \item[c)] EEG Eye State \cite{eeg_eye_state_264}: This dataset records EEG measurements  with binary labels indicating whether the eyes were open. The features consist of 14 EEG channels, labeled AF3, F7, F3, FC5, T7, P, O1, O2, P8, T8, FC6, F4, F8, AF4.
\end{enumerate} All these datasets involve classification tasks using some features. Accordingly, we take the number of decisions $K$ to be the number of classes and consider a binary reward, which is 1 if we select the correct class and 0 otherwise. The dimension of the features for datasets (a)–(c) is 7, 5, and 14, with two arms each, respectively. The number of rows in (a)-(c) are 3809, 8143, and 14980, therefore we choose the number of batches to be 5,6, and 7, respectively.

\paragraph{Setup} We leverage supervised learning classification datasets to simulate contextual bandits learning (e.g., see \cite{bietti2021contextual}). In particular, let $(x_t, c_t) \in \mathbb{R}^d \times \{1,2\}$ row in the dataset where $x_t$ is the context and $c_t$ is the true label for the $t$th instance. We consider this $t$th row as a contextual bandit instance with $x_t$ as given to the bandit algorithm, and we only reveal a binary reward of the chosen action $a_t$ to be 1 if it matches the true label $c_t$ and 0 otherwise. Therefore, for arms $a_t \in \{1,2\}$, we consider the model in \eqref{def:sim} and its non-parametric analogue: $Y_t = g^{(a_t)}(X_t) + \epsilon_t$, where $Y_t \in \{0,1\}$ based on whether the chosen arm is a correct match or not. Note, since we only observe one arm at a given instance $t$, we only observe the reward corresponding to the chosen arm $a_t$ at that particular instance. Apart from comparing the nonparametric batched bandit (BaSEDB) performance with the BIDS algorithm proposed in Algorithm \ref{algorithm: SIRBatchedBinning}, we also consider an oracle BIDS algorithm where we estimate the index parameter $\beta_0$ using the entire dataset, and then use that for sequential decision-making in the BIDS algorithm. We randomly permute the data 60 times and measure the average regret performance of the three algorithms.

\paragraph{Results} We plot the average regret (rolling fraction of incorrect decisions over 60 trials with randomly permuted rows) as a function of the number of instances (rows) seen thus far for the following algorithms:
\begin{enumerate}
    \item Nonparametric batched bandit (BaSEDB algorithm) of \cite{jiang2025batched}.
    \item BIDS algorithm (Algorithm \ref{algorithm: SIRBatchedBinning}) with initial estimator from Algorithm \ref{algorithm: Initial_Dir_Estimation}.
    \item BIDS algorithm (Algorithm \ref{algorithm: SIRBatchedBinning}) with estimated `oracle' index, where the oracle direction is estimated by applying SADE algorithm to the entire dataset.
\end{enumerate}
\begin{figure}[htbp]
    \centering
\includegraphics[width = 0.85\linewidth, height=0.3\linewidth]{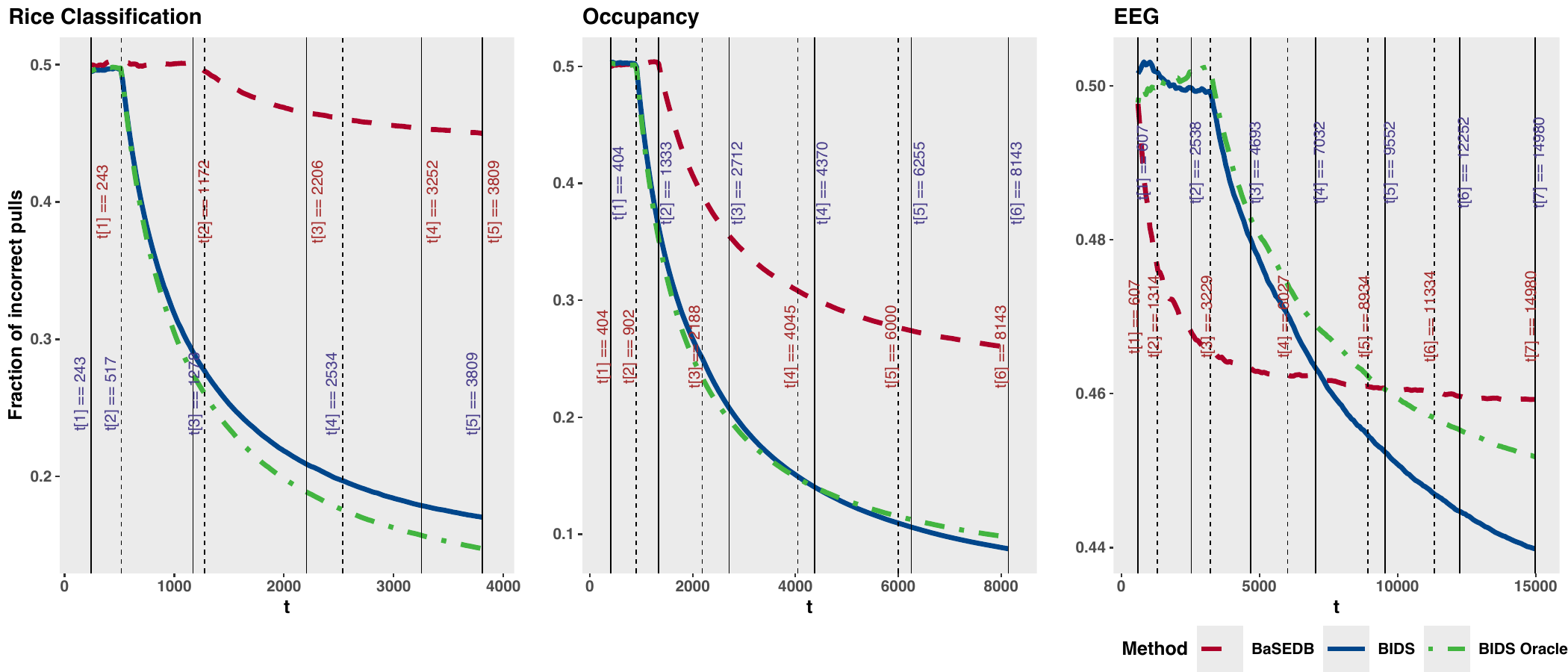}
    \caption{Comparison of expected regret of the proposed semiparametric BIDS algorithm and the nonparametric batched bandit algorithm (BaSEDB)  on a) rice classification, b) occupancy detection, and c) EEG datasets, with $\beta_0$ estimated in the initial phase with $t_{\rm{init}} \approx T^{2/3}$ for their respective data lengths $T$. Vertical solid and dashed lines denote the batch markings for the BIDS and BaSEDB algorithm, respectively. Observe that the BIDS outperforms BaSEDB in all instances, and for the Occupancy and EEG dataset it even performs similar/better to the BIDS oracle algorithm.}
    \label{fig:SIR_NP_ReadData}
\end{figure}
In Figure \ref{fig:SIR_NP_ReadData}, we notice that in all three datasets, the BIDS algorithm that we propose outperforms the nonparametric batched bandit (BaSEDB) algorithm of \cite{jiang2025batched}. We use $t_{\rm{init}} = T^{2/3}$ for each of the datasets. The vertical solid and dashed lines represent the batch end points for the GMABC and the nonparametric setup, respectively. In the Occupancy dataset, BIDS achieves performance comparable to the BIDS oracle algorithm. In the EEG dataset, although BaSEDB initially experiences a steep decline in regret, it eventually plateaus, whereas the regret for BIDS continues to decrease at a faster rate, surpassing BaSEDB after a certain point. To assess the effect of the initial sample size used for estimating the index parameter $\beta_0$, we compare performance across different values of $t_{\rm{init}}$ in Section\ifnum\pageoption=2~\ref{sec: additional_realdata} \else~SM5.3 \fi in Supplementary Material. The observed trends remain consistent: BIDS outperforms the nonparametric batched analogue across all three datasets. However, as the initial sample size increases, the average regret of BIDS approaches that of the oracle BIDS algorithm.

\begin{table}[htbp]
\centering
    \resizebox{.8\textwidth}{!}{
\begin{tabular}{lrrr}
  \hline
 & Rice Classification ($t_{\rm{init}} = 243$) & Occupancy ($t_{\rm{init}} = 404)$  & EEG ($t_{\rm{init}} = 607) $  \\ 
  \hline
$\beta_1$ & Area: 0.0279 (0.0206) & Temp: \textcolor{blue}{0.8326} (0.0817) & AF3:  0.0712 (0.0315) \\ 
 $\beta_2$ & Perimeter: \textcolor{blue}{-0.2979} (0.0247) & Humidity: -0.0036 (0.0046) & F7: {0.2979} (0.0266)  \\ 
$\beta_3$ & MajorAxis: \textcolor{blue}{0.4990} (0.0409) & Light: \textcolor{blue}{-0.0769} (0.0083)  & F3: 0.2088 (0.0387)  \\ 
$\beta_4$ & MinorAxis: \textcolor{blue}{-0.8085} (0.0762) & $CO_2$: \textcolor{blue}{-0.1310} (0.0151) & FC5: \textcolor{blue}{0.3310} (0.0170) \\ 
$\beta_5$ & Eccentricity: 0.0446 (0.0185) & HumidRatio: \textcolor{blue}{0.5327} (0.0782) & T7:  0.1372 (0.0638) \\ 
$\beta_6$ & Convex Area: \textcolor{blue}{0.0748} (0.0215) &    & P7: \textcolor{blue}{0.4034} (0.0512)  \\ 
$\beta_7$ & Extent: {0.0093} (0.0234) &    & O1: {0.2244} (0.0219) \\ 
$\beta_8$ &    &    & O2: {0.1807} (0.0236) \\ 
$\beta_9$&    &    & P8: \textcolor{blue}{0.3290} (0.0288) \\ 
$\beta_{10}$ &    &    & T8: {0.0832} (0.0304) \\ 
$\beta_{11}$ &    &    & FC6: {0.2663} (0.0183)  \\ 
$\beta_{12}$ &    &    & F4: {0.3146} (0.0314) \\ 
$\beta_{13}$ &    &    & F8: \textcolor{blue}{0.3213} (0.0199) \\ 
$\beta_{14}$ &    &    & AF4: {0.3164} (0.0266)  \\ 
   \hline
\end{tabular}}
\caption{Index parameter estimates used in the BIDS algorithm for the three datasets.}
\label{tab: beta_estimates}
\end{table}

\paragraph{Interpretability} In Table \ref{tab: beta_estimates}, we present the index parameter estimates for the three datasets using $t_{\rm{init}} = $  $243$, $404$, and $607$ ($\approx T^{(2/3)}$) observations, respectively. For each dataset with $d = 7,5, 14$, we report the estimated $\beta_i$ along with standard errors (over 60 replications).  Variable relevance is inferred from the magnitude of estimates, with the top four values per dataset highlighted in blue. In the Occupancy dataset, temperature, humidity ratio, light, and $CO_2$ levels are identified as key predictors, aligning with \cite{RoomOccupancy}. Similarly, in the Rice Classification dataset, our results agree with \cite{Cinarer2024RiceCA}, which suggests that `Extent' is not a useful feature in classifying rice into Cammeo and Osmancık rice types. Research on the EEG Eye State dataset has identified key features for distinguishing between eye-open and eye-closed states using EEG signals. These are derived from the 14 electrode channels, and the significant ones in Table \ref{tab: beta_estimates} (e.g., FC5, P7,  P8 and, F8) span all four brain regions as seen from Figure 2 in \cite{rosler2013first}. Right hemisphere channels (e.g., O2, P8, and F8) often show higher values for eye-open states, while left-hemisphere channels (e.g., F7, P7, and  T7) display other distinct patterns, aligning with \cite{rosler2013first,asquith2019classification}.

\section{Conclusion} \label{sec: 7_conclusion}

We propose a novel batched bandit framework that models reward functions using a semi-parametric single-index structure. 
By estimating a shared projection direction across arms, the BIDS algorithm reduces dimensionality and guides adaptive binning and successive arm elimination. We derive a lower bound for the {$K$-armed} GMABC problem and establish an upper bound that matches the lower bound {(for fixed $K$)} when the index parameter is available with sufficient accuracy, or, in its absence, when the margin parameter and batch size fall within a certain range of values {in terms of time horizon $T$}. Empirically, our method outperforms nonparametric baselines while offering substantial gains in interpretability and computational efficiency.

To the best of our knowledge, this is the first study to explore a single-index framework in contextual batched bandits, opening avenues for future research. An immediate open question is whether one can design an algorithm with a minimax-optimal upper bound that holds across all parameter regimes and batch sizes when the index is unavailable.
In this regard, one could draw on insights from recent work in transfer learning, specifically, by leveraging data collected from `source' bandits to estimate the index direction prior to initiating learning in the `target' bandit. Another promising direction is to estimate the index direction adaptively across batches by exploiting the margin condition, especially when the number of batches is moderate to large. Since interpretability is a key motivation of our work, developing formal inference procedures for the estimated index direction would further enhance practical utility.  In summary, our framework and proposed methodology bridge interpretability and flexibility in batched contextual bandits, offering both strong theoretical guarantees and practical gains. 

\section{Acknowledgment}
HS gratefully acknowledges partial support from NSF DMS-2311141.

\bibliographystyle{siamplain}
\bibliography{ref}
\ifnum\pageoption=2
\newpage
\setcounter{page}{1}
\renewcommand{\thepage}{SM\arabic{page}} 
\renewcommand{\thesection}{SM\arabic{section}}  
\renewcommand{\thetable}{SM\arabic{table}}  
\renewcommand{\thefigure}{SM\arabic{figure}}
\renewcommand{\theequation}{SM-\arabic{equation}}
\renewcommand{\theremark}{SM-\arabic{remark}}
\renewcommand{\thelemma}{SM-\arabic{lemma}}
\renewcommand{\theproposition}{SM-\arabic{proposition}}
\renewcommand{\thecorollary}{SM-\arabic{corollary}}
\setcounter{equation}{0}
\setcounter{section}{0}
\setcounter{table}{0}

\begin{center}
{\Large\bf Supplement Material}\\
\vspace{1em}
{\large Sakshi Arya and Hyebin Song}\\
\end{center}

\section{A summary table of notations}
First, to enhance readability, in Table \ref{tab:notations_long}, we provide a table of notations that are used in the paper and the proofs presented in this section.
\begin{table}[htbp]
\footnotesize
    \centering
     \resizebox{.85\textwidth}{!}{
    \begin{tabular}{|l|c|l|}
        \hline
        \textbf{Category} & \textbf{Notation} & \textbf{Description} \\
        \hline
        \multirow{2}{*}{\textbf{Problem setup}} 
        &  $T$  & Total time horizon \\
        & $M$ & Number of batches\\
        & $\mathcal{X}$ & Covariate space in $\mathbb{R}^d$\\
        & $\mathcal{G}$ & Partition of $\{1,\hdots, T\}$ in $M$ batches\\
        & $\{t_0, t_1,\dots,t_M\}$ & Batch end points\\
        & $R_T(\pi)$  & Cumulative regret of $\pi$\\
        & $\mathcal{R}_T(\pi)$  & Expected cumulative regret of $\pi$\\
         & $\angle u,v$ & Principal angle between $u$ and $v$: $\cos^{-1}(|u^\top v|)$\\
        \hline
        \multirow{2}{*}{\textbf{ Parameters}}  
        & $\beta_0$ & Index parameter \\
         & $\alpha$ & Margin parameter\\
         & $\{\omega_k\}_{k=1}^K$ & Weights for the average estimator\\
        \hline
        \multirow{2}{*}{\textbf{Algorithmic and Theory}}  
        & $\pi$ & Proposed BIDS algorithm\\
         & $\beta$ & Working direction\\
         & $\mathcal{I}_\beta: = [L_\beta, U_\beta]$ & Interval of projected covariates along $\beta$\\
        &  $t_{\rm{init}}$  & Initial batch size used when pilot unknown \\
        & $\hat{\beta}^{(k)}$ & Single index estimate for $k$th arm\\
        & $\hat{\beta}$ &  Initial index estimate of $\beta_0$\\
        & $\mathcal{T}$ & Tree of depth $M$\\
        & $\mathcal{A}_i$ & Partition of $\mathcal{I}_\beta = [L_\beta, U_\beta]$ at layer $i$\\ 
        & $w_i = |\mathcal{I}_\beta|/n_i$ & Bin width for $i$th layer\\
        & $b_l$ & Number of splits in layer $l$\\
        & $n_i$ & Number of equal width intervals in layer $i$\\
        & $\mathcal{T}_\mathcal{A}$ & $\cup_{i=1}^M \mathcal{A}_i$\\
        & $\mathcal{B}_i$ & Partition of $\mathcal{X}$ induced by $\mathcal{A}_i$\\
        & $C = C_A(\beta)$ & Bin in $\mathcal{X}$ corresponding to $A \in \mathcal{T}_\mathcal{A}$\\
        & $|C|_{\mathcal{T}}$ & width of $A$ for $C = C_A(\beta)$\\
        & $p(C) = p(C_A(\beta))$ & Parent bin of $C$ defined by $A$\\
        & $\text{child}(C)$ & Child bin of $C$ defined by $A$\\
        & $\mathcal{L}_t, \mathcal{L}^{(i)}$ & Set of active bins at time $t$/at batch $i$\\
       &  $\mathcal{J}_t$ &  $\cup_{s \leq t} \mathcal{L}_s$\\
        & $\mathcal{I}_C$ & Set of active arms in bin $C$\\
        & $\mathcal{I}^\prime_{C}$ & Set of active arms post arm-elimination in $C$ \\
        & $\underline{\mathcal{I}}_{C}, \overline{\mathcal{I}}_{C}, \mathcal{S}_C, \mathcal{G}_C$ & Sets defined in \eqref{def: I_C}, \eqref{def: GC}, \eqref{def: SC}\\
        & $U(m,T,C)$ & Threshold for arm elimination\\
        & $m_{C,i}$ & number of $X_t$'s falling in C during batch i \\
        & $m^*_{C,i}$ & $\E[m_{C,i}]$\\
        & SIR & Single-index regression\\
        & $\xi$, $c_B, R_X, \bar{c}_X, \underline{c}_X, L_0, D_0$ & Constants independent of $T$.\\
        \hline
    \end{tabular}}
    \caption{Summary of notations used in the paper}
    \label{tab:notations_long}
\end{table}

\section{Proofs for Section \ref{sec: 2_problem_setup}} \label{appendix: proofsofSection2}
\subsection{Proof for Lemma \ref{lem: tmtvn_condX}}
\begin{proof}
For any $v$, the density of $X^\top v$ is given by 
\begin{align*}
    f_{X^\top v}(u) = \begin{cases}
        \frac{1}{Z(v,\Sigma)} \exp\{-\frac{u^2}{2v^\top\Sigma v}\}  & x\in \T_v\\
        0 & \text{otherwise}
    \end{cases}
\end{align*}
where we define $\T_v:= \{x^\top v; v\in \mathcal{H}\}$ and $Z(v,\Sigma):= \int_{u \in \T_v} \exp\{-\frac{u^2}{2v^\top\Sigma v}\}du$. 

Let a unit vector $v$ be given such that $\|v\|_2=1$. First of all, we observe that $\T_v$ is an interval in $\R$. Note that $\mathcal{H}$ is a closed, convex set in $\R^d$. We can find $x_0(v), x_1(v) \in \mathcal{H}$ such that $x_0(v)^\top v = \min_{x\in \mathcal{H}} x^\top v:=L_0(v)$ and $x_1(v)^\top v = \max_{x\in \mathcal{H}} x^\top v:=L_1(v)$. Moreover, since the dual of the 
$\ell_{\infty }$-norm is the $\ell_1$-norm, $L_0(v) = -\|v\|_1$ and $L_1(v) = \|v\|_1$. Now we show for any $u \in [L_0(v), L_1(v)]$, $u \in \T_v$. Since $u \in [L_0(v), L_1(v)]$, we can find $t\in[0,1]$ such that $u = tL_0(v) + (1-t)L_1(v)$. Then $u = tx_0(v)^\top v + (1-t)x_1(v)^\top v = \{tx_0(v) +(1-t)x_1(v)\}^\top  v $. By convexity of $\mathcal{H}$, $tx_0(v) +(1-t)x_1(v) \in \mathcal{H}$, and therefore $u \in \T_v$, which shows that $\T_v = [L_0(v), L_1(v)] \subseteq \R$. 

Now let $R_0 = \|\beta_0\|_1 /(2\sqrt{d})$. Let $v \in \bB_2(R_0;\beta_0)$ be given such that $\|v\|_2=1$. We show that for any $u \in \T_v$, the density $f_{X^\top v}(u)$ is bounded below and above by constants $\underline{c}_X$ and  $\overline{c}_X$, which depend on model parameters $\beta_0$ and $\Sigma$, but independent of $v$. Recall that $L_0(v) = -\|v\|_1$ and $L_1(v) = \|v\|_1$. Since $|\|v\|_1 - \|\beta_0\|_1| \le \|v - \beta_0\|_1 \le \sqrt{d}R_0$, $|L_0(v) - L_0(\beta_0)| \le \sqrt{d}R_0$. Similarly, $|L_1(v) - L_1(\beta_0)| \le \sqrt{d}R_0$. In particular, $[L_0(\beta_0)/2, L_1(\beta_0)/2] \subseteq [L_0(v),L_1(v)] \subseteq [1.5L_0(\beta_0), 1.5L_1(\beta_0)].$
We let $$\underline{\T_0}:=[L_0(\beta_0)/2, L_1(\beta_0)/2], \overline{\T_0}:=[(3/2)L_0(\beta_0), (3/2)L_1(\beta_0)],$$ so that \[\underline{\T_0} \subseteq \T_v \subseteq \overline{\T_0}.\]
Since $\|v\|_2=1$, $\Lambda_{\rm min} (\Sigma) \le v^\top \Sigma v \le \Lambda_{\rm max} (\Sigma)$.
First, recall $$Z(v,\Sigma)= \int_{u \in \T_v} \exp\{-\frac{u^2}{2v^\top \Sigma v}\}du.$$ We have,
\begin{align*}
    Z(v,\Sigma)&= \int_{u \in \T_v} \exp\left\{-\frac{u^2}{2v^\top \Sigma v}\right\}du \ge \int_{u \in \underline{\T_0}} \exp\left\{-\frac{u^2}{2 \Lambda_{\rm min}(\Sigma)}\right\}du:= \underline{c}_{Z}
\end{align*}
Similarly, we have 
\begin{align*}
    Z(v,\Sigma) \le \int_{u \in \overline{\T_0}} \exp\left\{-\frac{u^2}{2 \Lambda_{\rm max}(\Sigma)}\right\}du:=\overline{c}_Z
\end{align*}
Then for $u \in \T_{v}$,
\begin{align*}
   \frac{1}{\overline{c}_Z} \inf_{u \in \overline{\T_0}}\exp\left\{-\frac{u^2}{2 \Lambda_{\rm min}(\Sigma)} \right\}&\le \frac{1}{Z(v,\Sigma)}\exp\left\{-\frac{u^2}{2v^\top \Sigma v}\right\}\\
   & \quad \quad \le \frac{1}{\underline{c}_Z}\sup_{u \in \overline{\T_0}} \exp\left\{-\frac{u^2}{2 \Lambda_{\rm max}(\Sigma)}\right\},
\end{align*}
and we can take, \begin{align*}
    \underline{c}_X = \frac{1}{\overline{c}_Z} \inf_{u \in \overline{\T_0}}\exp\{-\frac{u^2}{2 \Lambda_{\rm min}(\Sigma)} \}, \overline{c}_X = \frac{1}{\underline{c}_Z}\sup_{u \in \overline{\T_0}} \exp\{-\frac{u^2}{2 \Lambda_{\rm max}(\Sigma)}\}.
\end{align*}
\end{proof}

\section{Proofs for Section \ref{sec: 4_regret_bounds}}
\subsection{Proof of Theorem \ref{thm: lower_limit}}\label{sec: lower_bound_proof}
\begin{proof}
Recall the definition of the expected cumulative regret of $\pi$:
\begin{align*}
    \mathcal{R}_T(\pi) &= \E\left[\sum_{t=1}^T \max_{k\in{[K]}}g^{(k)}(X_t) - g^{(\pi_t(X_t))}(X_t)\right]\\
    &= \E\left[\sum_{t=1}^T \max_{k\in{[K]}}f^{(k)}(X_t^\top\beta_0) - f^{(\pi_t(X_t))}(X_t^\top \beta_0)\right]
\end{align*}
To make explicit the dependence of $ \mathcal{R}_T(\pi)$ on the reward functions $f^{(k)}$, direction $ \beta_0$, and covariate distribution $\P_X$, we write $$\mathcal{R}_T(\pi) = {\mathcal{R}_T(\pi;\,\, g^{(1)}(x) = f^{(1)}(x^\top\beta_0), \dots, g^{(K)}(x) = f^{(K)}(x^\top \beta_0), \P_X).}$$

We want to lower-bound:
\[
\inf_\pi \sup_{\substack{{f^{(k)} \in \mathcal{F}(\alpha; (\beta_0,\P_X)),\,\forall k\in [K],}\\
\beta_0 \in \mathbb{S}^{d-1},\,\P_X \in \mathcal{P}_X}} \mathcal{R}_T(\pi;\,\, g^{(1)}(x) = f^{(1)}(x^\top\beta_0), g^{(2)}(x) = f^{(2)}(x^\top \beta_0),\P_X).
\]
We first choose $\mathbb{P}_X$ and $\beta_0$. Let $\beta_0=[1,0,\dots,0]$ be given, and let $\P_X = N_T(0,I_n;\mathcal{H})$ be a truncated normal distribution with $\mathcal{H} = \prod_{j=1}^d 1\{|x_j|\le 0.5\}$, whose density is given by 
\begin{align*}
    f_{X}(x) = \begin{cases}
        \frac{1}{Z} \exp\{-\frac{1}{2}x^\top x\}   & x\in \mathcal{H}\\
        0 & \text{otherwise}
    \end{cases},
\end{align*}
with $Z =\int_{x\in \R^d} e^{-\frac{1}{2}x^\top x} 1\{x\in \mathcal{H}\} dx $.  Define $U = X^\top \beta_0$ for $X \in \mathcal{X}$.   By Lemma \ref{lem: tmtvn_condX}, we have $\P_X = N_T(0,I_n;\mathcal{H}) \in \mathcal{P}_X$.  Since $U = X^\top \beta_0 = X_1 $, we have $U \sim N(0,1)$ truncated to $[-0,5,0.5]$, with the density
\begin{align*}
    p_{U}(u) = \frac{\phi(u)}{\Phi(0.5)-\Phi(-0.5)}
\end{align*}
for $u\in [-0.5,0.5]$ and $0$ elsewhere, where $\phi(\cdot)$ and $\Phi(\cdot)$ are the pdf and cdf of the standard normal distribution.
In particular, 
$\underbar{c}\leq p_U(u) \leq \bar{c}$ 
for $u \in [-0.5,0.5]$, where $\underbar{c} = \phi(0.5)/(\Phi(0.5)-\Phi(-0.5))$ and $\bar{c} = \phi(0)/(\Phi(0.5)-\Phi(-0.5))$.

With these choices, we have
\begin{align*}
    &\sup_{\substack{{f^{(k)} \in \mathcal{F}(\alpha; (\beta_0,\P_X)),\,\forall k\in [K],}\\
\beta_0 \in \mathbb{S}^{d-1},\,\P_X \in \mathcal{P}_X}} \mathcal{R}_T(\pi; \,\,{g^{(1)}(x) = f^{(1)}(x^\top \beta_0),\dots, g^{(K)}(x) = f^{(K)}(x^\top \beta_0)})
    \\&\qquad \ge \sup_{{f^{(k)} \in \mathcal{F}(\alpha; (\beta_0,N_T(0,I_n;\mathcal{H})),\,\forall k\in[K]}} \E\left[\sum_{t=1}^T \max_{k\in{[K]}}f^{(k)}(X_{t,1}) - f^{(\pi_t(X_t))}(X_{t,1})\right],
\end{align*}
where the expectation is taken with respect to a measure under which the distribution of $X_{t,1}$ is $N_T(0,1; [-0.5,0.5])$.
For notational convenience, we abuse notation slightly and define
\begin{align*}
    \mathcal{F}(\alpha) = \mathcal{F}(\alpha;\,[1,\dots,0], N_T(0,I_n;\mathcal{H}))
\end{align*}
and for any $t \leq  T$,
\begin{align*}\mathcal{R}_t(\pi;\,\,{f^{(1)}, \dots,f^{(K)}}):=\E\left[\sum_{s=1}^t \max_{k\in{[K]}}f^{(k)}(X_{s,1}) - f^{(\pi_t(X_s))}(X_{s,1})\right], 
\end{align*}
which is the cumulative expected regret up to time $t$ with the choice of $\beta_0 = [1,0,\dots,0]$ and $\P_X=N_T(0,I_n;\mathcal{H})$.


To further lower bound $\mathcal{R}_t$, we define the inferior sampling rate up to time $t$,  denoted as $S_t$, following \citesupp{perchet2013multi_supp} and present Lemma \ref{lem: perchet_lemma_S_R} which connects $\mathcal{R}_t$ and $S_t$. {Lemma~\eqref{lem: perchet_lemma_S_R} is an $K$ arm extension of Lemma 3.1 of \citesupp{rigollet2010nonparametric_supp} which proved the same result for $K=2$. The proof is presented at the end of this section.}
\begin{definition}
[Inferior sampling rate]
For algorithm $\pi$ and any $1\le t\le T$, define the inferior sampling rate up to time $t$ as
\begin{align}\label{def: inferior_sampling_rate}
    S_t(\pi) = \E\left[\sum_{s=1}^t 1\{f^{(\pi_s(X_s))}(X_s^\top \beta_0) < \max_{k\in{[K]}}f^{(k)}(X_s^\top \beta_0)\} \right].
\end{align}
\end{definition}

\begin{lemma} \label{lem: perchet_lemma_S_R}
Under the margin condition \ref{assum: Margin} with any $\alpha>0$, there exists a constant $C_0 >0$ such that 
    \begin{align}
        \mathcal{R}_t(\pi) \geq C_0 (S_t(\pi))^{\frac{\alpha + 1}{\alpha}} t^{-1/\alpha}.
    \end{align}
\end{lemma}
Note that the worst-case regret over the time horizon $T$ is larger than the worst-case regret over the first $i$ batches. Therefore, for any $i=1,\dots,M$,
\begin{align}
    \sup_{{f^{(1)}, \dots,f^{(K)}} \in \mathcal{F}(\alpha)} \mathcal{R}_T(\pi;\; {f^{(1)}, \dots,f^{(K)}}) &\geq \max_{1\leq i \leq M} \sup_{{f^{(1)}, \dots,f^{(K)}}\in \mathcal{F}(\alpha)} \mathcal{R}_{t_i}(\pi;{f^{(1)}, \dots,f^{(K)}}) \nonumber\\
    &\geq C_0 \max_{1\leq i \leq M} t_i^{-1/\alpha} \left[\sup_{{f^{(1)}, \dots,f^{(K)}} \in \mathcal{F}(\alpha)} S_{t_i}(\pi; {f^{(1)}, \dots,f^{(K)}})\right]^{\frac{1 + \alpha}{\alpha}},  \label{eq: regret_batched_lower}
\end{align}
where \eqref{eq: regret_batched_lower} follows from Lemma \ref{lem: perchet_lemma_S_R}. 

Now, we focus on lower bounding $\sup_{{f^{(1)}, \dots,f^{(K)}} \in \mathcal{F}(\alpha)} S_{t_i}(\pi; {f^{(1)}, \dots,f^{(K)}})$ by creating specific families of hard instances for reward functions in $\mathcal{F}(\alpha)$ targeting different batch indices $i$.
First, fix $i \in \{1,\dots,M\}$. All constructions that follow are for this fixed batch index $i$, but we suppress the dependence of the construction on $i$ for notational simplicity.
Split $[-0.5,0.5]$ into $n = 1/h$ equal-width intervals $\mathcal{I}_j$, each with width $h$, where $0<h\le 1$ is to be chosen later. Let $u_1,\hdots,u_{n}$ be the center of each interval $\mathcal{I}_j, \, j = 1,\hdots,n$. Let $D = \lceil n^{1-\alpha} \rceil =  \lceil h^{-(1-\alpha)}\rceil$, i.e., the largest integer corresponding to $n^{1-\alpha}$. 
For each bit vector $v \in {\{1,\dots,K\}}^D$, $0<h\le 1$ and {$C_f = \min\{\tau, 1/2\}$, define for each arm $k\in [K]$},
\begin{equation*}
{f^{(k)}_{v,h}(u) 
    = \frac{1}{2} 
    + C_f h \sum_{j=1}^D 
    \mathbf{1}\{k = v_j\} 
    \K\!\left(\frac{u-u_j}{h}\right)}
\end{equation*}
where $\K(u) = (1-|2u|)1\{|u|\le 0.5\}$. Each $\K((u-u_j)/h)$ is a ``bump" function supported on the interval $u_j \pm0.5h$. The coefficient $v_j$ determines whether a bump is added at interval $j$, {while all other arms remain at the baseline level $1/2$. Thus, given $v$, the arm $k=v_j$ is the best arm on interval $j$ by construction.} 

Define the class of functions
\begin{align*}
   { \mathcal{F}_{v,h}(\alpha) = \{(f^{(1)}_{v,h},\dots, f^{(K)}_{v,h});\; v\in \{1,\dots,K\}^D\}.}
\end{align*}
The following Lemma \ref{lem: F_vh} shows that the constructed family $\mathcal{F}_{v,h}(\alpha)$ is contained in our function class $\mathcal{F}(\alpha)$.  The proof of Lemma \ref{lem: F_vh} is presented at the end of this proof. 
\begin{lemma}\label{lem: F_vh}
    For any $0 \le \alpha \le 1$ and $h>0$, we have $\mathcal{F}_{v,h}(\alpha) \subseteq \mathcal{F}(\alpha)$ .
\end{lemma}

Then, for any $i = 1,\hdots, M$,
\begin{align}\label{eq: lb-reduction1}
 \sup_{{f^{(1)}, \dots,f^{(K)}} \in \mathcal{F}(\alpha)} \mathcal{S}_{t_i}(\pi;\; {f^{(1)}, \dots,f^{(K)}}) \geq   \sup_{{f^{(1)}, \dots,f^{(K)}}\in \mathcal{F}_{v,h}(\alpha)} \mathcal{S}_{t_i}(\pi;f^{(1)} = f_{v,h},f^{(2)} = 1/2).
\end{align}
Recall from our construction $X_t^\top \beta_0 = X_{t,1}$. For $X_{t,1} \in \mathcal{I}_j$, we have ${\arg\max_{a \in [K]} f^{(a)}_{v,h}(X_t^\top \beta_0) = v_j}$. Also recall that the inferior sampling rate up to time $t_i$ is defined as
\begin{align}\label{eq: S_ti}
    S_{t_i}(\pi) =\sum_{t=1}^{t_i}  \E\left[1\{f^{(\pi_t(X_t))}(X_t^\top \beta_0) < \max_{k\in{[K]}}f^{(k)}(X_t^\top \beta_0)\} \right].
\end{align}

For each $t=1,2,\dots$, let $\P^t_{v}$ denote the joint distribution of the collection of pairs $(X_j,\,Y_j^{(\pi_j(X_j))})_{1\le j \le t}$, where the mean reward functions are given by ${(f^{(1)}, \dots,f^{(K)})=(f_{v,h}^{(1)}, \dots,f_{v,h}^{(K)})}$, and let $\E^t_{v}$ denote the expectation with respect to this distribution. Note that the expectation of the term at time $t \in [t_{l-1}+1,t_l]$ ($l$th batch) in \eqref{eq: S_ti} is taken with respect to the product measure $\P_v^{t_{l-1}} \otimes \P_X$. This is because in the batched setting, $\pi_t$ depends on information available up to time $t_{l-1}$, while $X_t$ is sampled independently from $\P_X$. For notational simplicity, denote $\P_v^t = \P_v^{t_{l-1}} \otimes \P_X$ for $t \in [t_{l-1}+1,t_l]$.
We have,
\begin{align}
    \sup_{{f^{(1)}, \dots,f^{(K)}}\in \mathcal{F}_{v,h}(\alpha)} \mathcal{S}_{t_i}(\pi;{f^{(1)}, \dots,f^{(K)}})  
    &= \sup_{v \in {[K]}^D}  \sum_{t=1}^{t_i} \P_v^t \left[ \pi_t(X_t) \neq \arg\max_{a \in {[K]}} f^{(a)}(X_t^\top \beta_0) \right] \nonumber\\
&= \sup_{v \in {[K]}^D} \sum_{j=1}^D \sum_{t=1}^{t_i} \P_v^t \left[ \pi_t(X_t) \neq v_j,\, X_{t,1} \in \mathcal{I}_j  \right] \nonumber\\
&\geq \frac{1}{{K}^D} \sum_{j=1}^D \sum_{t=1}^{t_i} \sum_{v \in {[K]}^D} \P_v^t \left[ \pi_t(X_t) \neq v_j,\, X_{t,1} \in \mathcal{I}_j \right]. \label{eq: inferior_sampling}
\end{align}

Denote $v_{[-j]}= (v_1,\dots,v_{j-1},v_{j+1},\dots,v_D)$ and $v_{[-j]}^k= (v_1,\dots,v_{j-1},k,v_{j+1},\dots,v_D)$. Decomposing the last summation, for any $j$:
\begin{align}
    \sum_{v \in {[K]}^D} \P_v^t \left[ \pi_t(X_t) \neq v_j,\, X_{t,1} \in \mathcal{I}_j \right] 
    &= \sum_{v_{[-j]} \in {[K]}^{D-1}}\sum_{k \in {[K]}} \P^t_{v_{[-j]}^k} \left[ \pi_t(X_t) \neq k,\, X_{t,1} \in \mathcal{I}_j \right] \nonumber\\
    &= \sum_{v_{[-j]} \in {[K]}^{D-1}}\sum_{k \in {[K]}} \P^t_{v_{[-j]}^k} \left[ \pi_t(X_t) \neq k\,| \,X_{t,1} \in \mathcal{I}_j \right]\P_X[X_{t,1} \in \mathcal{I}_j] \nonumber\\
    &\ge \underline{c}h \sum_{v_{[-j]} \in {[K]}^{D-1}}\sum_{k \in {[K]}} \P^t_{v_{[-j]}^k} \left[ \pi_t(X_t) \neq k\,| \,X_{t,1} \in \mathcal{I}_j \right].\label{eq: inf_samp_rate2}
\end{align}
{
Plugging \eqref{eq: inf_samp_rate2} into \eqref{eq: inferior_sampling} above yields
\begin{align}
&\sup_{f^{(1)},\dots,f^{(K)}\in \mathcal{F}_{v,h}(\alpha)} 
\mathcal{S}_{t_i}(\pi;f^{(1)},\dots,f^{(K)})
\nonumber\\
&\ge
\frac{\underline c h}{K^D}
\sum_{j=1}^D \sum_{t=1}^{t_i}
\sum_{v_{[-j]}\in [K]^{D-1}}
\sum_{k=1}^K
\P^t_{v_{[-j]}^k}\!\left(
\pi_t(X_t)\neq k \,\middle|\, X_{t,1}\in\mathcal I_j
\right) \nonumber\\
&\ge \underline c h\sum_{j=1}^D\sum_{l=1}^i \sum_{t = t_{l-1}+1}^{t_l} \left\lbrace\frac{1}{K^{D-1}}\sum_{v_{[-j]}\in [K]^{D-1}}
\left[\frac{1}{K}\sum_{k=1}^K \P^t_{v_{[-j]}^k}\!\left(
\pi_t(X_t)\neq k \,\middle|\, X_{t,1}\in\mathcal I_j\right)\right]\right\rbrace.\label{eq: inf_samp_rate2K_simplified}
\end{align}
We then relate 
\begin{equation}\label{eq: simplifyKD}
    \frac{1}{K}\sum_{k=1}^K \P^t_{v_{[-j]}^k}\!\left( \pi_t(X_t)\neq k \,\middle|\, X_{t,1}\in\mathcal I_j\right)
\end{equation}
to {a $K$-arm hypothesis testing} problem. }
{
Fix the interval index $j\in\{1,\dots,D\}$ and fix the remaining coefficients 
$v_{[-j]}\in[K]^{D-1}$. Define
$\P_X^j(\cdot)=\P_X(\cdot\mid X_1\in\mathcal I_j)$.
For $t\in[t_{l-1}+1,t_l]$, conditioning on $X_{t,1}\in\mathcal I_j$,
consider the collection of probability measures:
\[
\left\{
\P_{v_{[-j]}^k}^{\,t_{l-1}}\otimes\P_X^j
\right\}_{k=1}^K.
\]
Introduce a random variable $V_j$ uniformly distributed on
$\{1,\dots,K\}$ and interpret $V_j=k$ as the hypothesis under which
the data are generated according to
$\P_{v_{[-j]}^k}^{\,t_{l-1}}\otimes\P_X^j$.
Under hypothesis $k$, the optimal arm on interval $\mathcal I_j$
is $k$. }
%

{We now consider two cases: $K\ge 3$ and $K=2$. We handle $K=2$ case separately, since Fano's inequality becomes trivial when $K=2$ since the right-hand side of \eqref{eq: Fano} is non-positive.}

\paragraph{Case I: $K\ge 3$}
{Applying Fano's inequality (Lemma 2.10 in \citesupp{tsybakov2008introduction_supp}, and writing in terms of mutual information \citesupp{scarlett2021introductory_supp}), 
together with the bound $H(x)\le \log 2$ for the binary entropy function,
we obtain
\begin{align}\label{eq: Fano}
\frac{1}{K}\sum_{k=1}^{K} \P^t_{v_{[-j]}^k}
\!\left[\pi_t(X_t) \neq k\,\middle|\,
X_{t,1} \in \mathcal{I}_j\right]\;\ge\; 1 - \frac{I(V_j ; \mathcal{H}_{t_{l-1}})+ \log 2}{\log K},
\end{align}
where $\mathcal{H}_{t_{l-1}}$ denotes the $\sigma$-field generated by the history up to time $t_{l-1}$.
}

{
Since $V_j$ is uniform on $\{1,\dots,K\}$, we have the identity
\[
I(V_j ; \mathcal H_n)
=
\frac{1}{K}
\sum_{k=1}^K
\mathrm{KL}\!\left(
\P^n_{v_{[-j]}^k},
\bar P_n
\right),
\]
where $\bar P_n = \frac{1}{K}\sum_{k=1}^K \P^n_{v_{[-j]}^k}$.
By the convexity of the KL divergence in its second argument (cf, Page 113 in \citesupp{tsybakov2008introduction_supp}),
\begin{align*}
I(V_j ; \mathcal H_n)
&\le
\frac{1}{K^2}
\sum_{k=1}^K \sum_{k'=1}^K
\mathrm{KL}\!\left(
\P^n_{v_{[-j]}^k},
\P^n_{v_{[-j]}^{k'}}
\right).
\end{align*}
Using the chain rule decomposition of KL divergence together with the KL bound
assumption in \eqref{eq: KL_bound}, for any $1\le n \le T$,
\begin{align*}
\mathrm{KL}\!\left(
\P^n_{v_{[-j]}^k},\,
\P^n_{v_{[-j]}^{k'}}
\right)
&\le
\frac{h^2}{4\kappa^2}
\E_{v_{[-j]}^k}
\!\left[
\sum_{t=1}^n 
1\{\pi_t(X_t)\in\{k,k'\},\,
X_t \in \mathcal{I}_j \}
\right],
\end{align*}
where we used that
\[
\{f_{v_{[-j]}^k,h}(X_t) - f_{v_{[-j]}^{k'},h}(X_t)\}^2
\le C_f^2 h^2 \le \frac{h^2}{4}.
\]
By the law of total probability,
\begin{align*}
\E_{v_{[-j]}^k}
\!\left[
1\{\pi_t(X_t)\in\{k,k'\},\, X_{t,1} \in \mathcal{I}_j\}
\right]
&=
\P_X(X_{t,1}\in\mathcal I_j)
\,
\P_{v_{[-j]}^k}\!\left(
\pi_t(X_t)\in\{k,k'\}
\mid X_{t,1}\in\mathcal I_j
\right).
\end{align*}
Since $0 \le \P(\pi_t(X_t)\in\{k,k'\}\mid X_{t,1}\in\mathcal I_j)\le 1$
and
\[
\P_X(X_{t,1}\in\mathcal I_j)
=
\int_{\mathcal I_j} p_U(u)\,du
\le \bar ch,
\]
we obtain
\begin{align*}
\mathrm{KL}\!\left(
\P^n_{v_{[-j]}^k},\,
\P^n_{v_{[-j]}^{k'}}
\right)
\le
\frac{1}{4\kappa^2}\,\bar c\,h^3 n
=:\tilde c h^3 n,
\end{align*}
where $\tilde{c}:=\bar c/4\kappa^2$.
Therefore,
\begin{align*}
I(V_j ; \mathcal H_n)
&\le
\frac{1}{K^2}
\sum_{k,k'}
\tilde c h^3 n
\le
\tilde c h^3 n.
\end{align*}
Consequently, for each $j$ and each $t \in [t_{l-1}+1,t_l]$,
\begin{align*}
\frac{1}{K}\sum_{k=1}^K 
\P^t_{v_{[-j]}^k}
\!\left(
\pi_t(X_t) \neq k \,\middle|\, X_{t,1} \in \mathcal{I}_j
\right)
\;\ge\;
\left\lbrace 1 - \frac{\tilde c h^3 t_{l-1} + \log 2}{\log K}\right\rbrace =:c_l.
\end{align*}
Note that this lower bound does not depend on $j$. Therefore, we get:
\begin{equation}\label{eq: KL_bound_pf}
\frac{1}{K^{D-1}}\sum_{v_{[-j]}\in[K]^{D-1}}
\left[
\frac{1}{K}\sum_{k=1}^K
\P^t_{v_{[-j]}^k}\!\left(\pi_t(X_t)\neq k \,\middle|\, X_{t,1}\in\mathcal I_j\right)
\right]
\ge c_l.
\end{equation}
}

{
Finally, since $t_{l-1} \le t_{i-1}$ for all $1 \le l \le i$ and
the function $x \mapsto 1 - \frac{\tilde c h^3 x + \log 2}{\log K}$
is decreasing in $x$, we have $c_l \ge c_i$, for all $l \le i.$
Using this bound in \eqref{eq: inf_samp_rate2K_simplified},
\begin{align*}
     \sup_{f^{(1)}, \dots,f^{(K)}\in \mathcal{F}_{v,h}(\alpha)} \mathcal{S}_{t_i}(\pi)   
     &\ge  \sum_{j=1}^D\sum_{l=1}^{i}\sum_{t=t_{l-1}+1}^{t_l}\{\underline{c}h c_l\}\ge \underline{c}hDc_i\sum_{l=1}^{i}(t_l - t_{l-1})\ge \underline{c}c_ih^{\alpha}t_i
\end{align*}
where for the last inequality we use $D = \lceil h^{-1+\alpha} \rceil  \ge h^{-1+\alpha}$ and $\sum_{l=1}^{i}(t_l - t_{l-1}) = t_i-t_0 = t_i$.
}

{
We now choose $h$. Recall that $K\ge 3$. 
When $i=1$, set $h=h_1 = 1$, so $c_1 = 1-\frac{\log 2}{\log K} \in (0,1)$.
When $i>1$, choose $h$ so that $\tilde{c}h^3t_{i-1} = (\log K)/3$, that is, 
\begin{equation*}
    h = h_i = \tilde{C}\left(\frac{\log K}{t_{i-1}}\right)^{1/3},
\end{equation*}
where $\tilde{C} = (3\tilde{c})^{-1/3}$. Then,
\begin{align*}
    c_i &= \left(1-\frac{\log 2}{\log K}\right) - \frac{\tilde{c}h^3t_{i-1}}{\log K}\ge \left(1-\frac{\log 2}{\log 3}\right) - \frac{\tilde{c}h^3t_{i-1}}{\log K} \ge \frac{2}{3}-\frac{\log 2}{\log 3}> 0.
\end{align*}
}

{
Therefore, 
\begin{align}\label{eq: S_ti_multi}
  \sup_{f^{(1)},\dots, f^{(K)}\in \mathcal{F}_{v,h}(\alpha)} \mathcal{S}_{t_i}(\pi)   \geq \begin{cases}
        c_* (\log K)^{\alpha/3}\dfrac{t_i}{t_{i-1}^{\alpha/3}} & \ \text{when} \ i > 1\\
        c_* t_1 & \ \text{when} \ i = 1
    \end{cases},
\end{align}
for some $c_* > 0$, which depends on $\overline{c},\underline{c}, \kappa$, and other universal constants. 
}

\paragraph{Case II: $K=2$} 
{We now bound \eqref{eq: simplifyKD} using Le Cam's method (ref. Theorem 2.2 in \citesupp{tsybakov2008introduction_supp}). We have
\begin{align*}
\frac{1}{2}\sum_{k=1}^{2} \P^t_{v_{[-j]}^k}
\!\left[\pi_t(X_t) \neq k\,\middle|\,
X_{t,1} \in \mathcal{I}_j\right]\;
&\ge\;\frac{1}{4} \exp\left(-\mathrm{KL}\left(\mathbb{P}_{v_{[-j]}^1}^{t_{l-1}}\otimes \P_X^j,\, \mathbb{P}_{v_{[-j]}^2}^{t_{l-1}}\otimes \P_X^j\right)\right)\\
&\ge \frac{1}{4}\exp\left(-\mathrm{KL}\left(\mathbb{P}_{v_{[-j]}^1}^{t_{l-1}},\, \mathbb{P}_{v_{[-j]}^2}^{t_{l-1}}\right)\right)\\
&\ge \frac{1}{4} \exp(-\tilde{c}h^3 t_{l-1}).
\end{align*}
where the last inequality follows from the KL bound in \eqref{eq: KL_bound_pf}. Then, 
\begin{align*}
    \frac{1}{K^{D-1}}\sum_{v_{[-j]}\in [K]^{D-1}} \left[\frac{1}{K}\sum_{k=1}^K \P^t_{v_{[-j]}^k}\!\left(
\pi_t(X_t)\neq k \,\middle|\, X_{t,1}\in\mathcal I_j\right)\right]  \ge \frac{1}{4} \exp(-\tilde{c}h^3 t_{l-1}).
\end{align*}
Therefore, from \eqref{eq: inf_samp_rate2K_simplified},
\begin{align*}
\sup_{f^{(1)},\dots,f^{(K)}\in \mathcal{F}_{v,h}(\alpha)} 
\mathcal{S}_{t_i}(\pi)
&\ge \frac{\underline c h}{4}\sum_{j=1}^D\sum_{l=1}^i \sum_{t = t_{l-1}+1}^{t_l} \exp(-\tilde{c}h^3 t_{l-1}) \\
&\ge \frac{\underline c h^{\alpha} }{4}\sum_{l=1}^i (t_l - t_{l-1})\exp(-\tilde{c}h^3 t_{i-1})\\
&\ge \frac{\underline c h^{\alpha} }{4}t_i\exp(-\tilde{c}h^3 t_{i-1}),
\end{align*}
where for the second inequality we use $D = \lceil h^{-1+\alpha} \rceil  \ge h^{-1+\alpha}$ and $t_{l-1} \le t_{i-1}$ for $1\le l \le i$.
}

{
Now, choosing $h = h_1 = 1$ when $i = 1$ and $h = h_i = (t_{i-1})^{-1/3}$ when $i > 1$. Then, for $i>1$,
\begin{align*}
h^\alpha t_i \exp(-\tilde{c}h^3 t_{i-1}) = t_{i-1}^{-\alpha/3} t_i\exp(-\tilde{c}).
\end{align*}
Therefore,
\begin{align}\label{eq: S_ti_2}
  \sup_{f^{(1)},\dots, f^{(K)}\in \mathcal{F}_{v,h}(\alpha)} \mathcal{S}_{t_i}(\pi)   \geq \begin{cases}
        \tilde{c}_* \dfrac{t_i}{t_{i-1}^{\alpha/3}} & \ \text{when} \ i > 1\\
        \tilde{c}_* t_1 & \ \text{when} \ i = 1
    \end{cases},
\end{align}
for some constant $\tilde{c}_* > 0$ depending only on the model parameters. 
}

{
Combining \eqref{eq: S_ti_multi} and \eqref{eq: S_ti_2} with the previous arguments in \eqref{eq: regret_batched_lower}:
Now, combining this with \eqref{eq: regret_batched_lower}, we obtain
\begin{align*}
\sup_{f^{(1)},\dots,f^{(K)} \in \mathcal{F}(\alpha)}
\mathcal{R}_T(\pi)
&\ge
C_0 \max_{1\le i\le M} t_i^{-1/\alpha}
\left[ \sup_{f^{(1)},\dots,f^{(K)}\in \mathcal{F}_{v,h}(\alpha)}
\mathcal{S}_{t_i}(\pi) \right]^{\frac{1+\alpha}{\alpha}} \\
&\ge C_1 \max \left\{t_1,\;
a_K^\gamma \dfrac{t_2}{t_1^\gamma},
a_K^\gamma \dfrac{t_3}{t_2^\gamma},
\dots, a_K^\gamma \dfrac{T}{t_{M -1}^\gamma}\right\}\\
&\ge C_1 \min_{t_1,\dots,t_{M-1}}\max \left\{t_1,\;
a_K^\gamma \dfrac{t_2}{t_1^\gamma},
a_K^\gamma \dfrac{t_3}{t_2^\gamma},
\dots, a_K^\gamma \dfrac{T}{t_{M -1}^\gamma}\right\}
\end{align*}
where
\begin{equation*}
    \gamma = \frac{1+\alpha}{3}, \quad a_K:= \begin{cases}
        1,&K=2\\\log K, &K\ge 3,
    \end{cases}
\end{equation*}
and $C_1>0$ depends only on the model parameters. Define
\[
f(t_1,\dots,t_{M-1})
=
\max
\left\{
t_1,\; a_K^{\gamma} \dfrac{t_2}{t_1^{\gamma}},\, a_K^{\gamma} \dfrac{t_3}{t_2^{\gamma}},\dots,
\,a_K^{\gamma}\dfrac{T}{t_{M-1}^{\gamma}}\right\}.
\]
We know that the minimum is achieved when $\tilde{t}_1 = a_K^\gamma(\tilde{t}_2/\tilde{t}_1^{\gamma} )= \dots =a_K^\gamma (T/\tilde{t}_{M-1}^{\gamma})$ (as altering any of the terms will increase min-max).
Let
\[
l_T:=\min_{t_1,\dots,t_{M-1}}
f(t_1,\dots,t_{M-1}) = f(\tilde{t}_1,\dots,\tilde{t}_{M-1}).
\]
Solving recursively,
\begin{align*}
l_T &= \tilde t_1, \\
\tilde t_2 &= \frac{l_T^{1+\gamma}}{a_K^{\gamma}}, \\
\tilde t_3 &= \frac{l_T^{1+\gamma+\gamma^2}}
{a_K^{\gamma(1+\gamma)}}, \\
&\;\;\vdots \\
T &= \frac{l_T^{1+\gamma+\dots+\gamma^{M-1}}}
{a_K^{\gamma(1+\gamma+\dots+\gamma^{M-2})}}.
\end{align*}
Therefore, we obtain
\[
T
=
l_T^{\frac{1-\gamma^M}{1-\gamma}}
a_K^{-\frac{\gamma(1-\gamma^{M-1})}{1-\gamma}}.
\]
Hence,
\[
l_T
=
T^{\frac{1-\gamma}{1-\gamma^M}}
a_K^{\frac{\gamma(1-\gamma^{M-1})}{1-\gamma^M}}.
\]
In particular,
\begin{align*}
\sup_{f^{(1)},\dots,f^{(K)} \in \mathcal{F}(\alpha)}
\mathcal{R}_T(\pi)
\ge
C_1
T^{\frac{1-\gamma}{1-\gamma^M}}
a_K^{\frac{\gamma(1-\gamma^{M-1})}{1-\gamma^M}}.
\end{align*}
which proves the result.
}

\end{proof}
\begin{proof}[Proof of Lemma \ref{lem: F_vh}]
{
We verify Assumptions \ref{assum: Smoothness} and \ref{assum: Margin} for $f^{(k)}(x) = f_{v,h}^{(k)}(x)$, $k=1,\dots,K$, defined by
\begin{equation*}
f^{(k)}_{v,h}(u) 
    = \frac{1}{2} 
    + C_f h \sum_{j=1}^D 
    \mathbf{1}\{k = v_j\} 
    \K\!\left(\frac{u-u_j}{h}\right)
\end{equation*}
for $v \in \{1,\dots,K\}^D$.
}

\paragraph{Lipschitz condition} First, note that the kernel $\K$ is $2$-Lipschitz; that is, for all $u_1, u_2 \in \mathbb{R}$,
\[
|\K(u_1) - \K(u_2)| \leq 2|u_1-u_2|.
\]
{Let $k \in [K]$ be given.} We analyze the difference {$|f_{v,h}^{(k)}(u_1) - f_{v,h}^{(k)}(u_2)|$} in three cases, using the fact that each bump has support of width $h$.\\
\textbf{Case 1:} When both $u_1$ and $u_2$ belong to the same bump:
In this case, there exists $j^*$ such that $u_1, u_2 \in [u_{j^*} - h/2,\, u_{j^*} + h/2]$, and for all $j \neq j^*$, $\K((u_1-u_j)/h) = \K((u_2-u_j)/h) = 0$. Thus,
\begin{align*}
{|f_{v,h}^{(k)}(u_1) - f_{v,h}^{(k)}(u_2)|} &= C_fh \, {\mathbf{1}\{k = v_{j^*}\} } \left| \K\left( \frac{u_1-u_{j^*}}{h} \right) - \K\left( \frac{u_2-u_{j^*}}{h} \right) \right| \\
&\leq C_f h \cdot 2 \left| \frac{u_1-u_2}{h} \right| = 2C_f|u_1-u_2|.
\end{align*}
\textbf{Case 2:} $u_1$ and $u_2$ belong to adjacent bumps, with $|u_1-u_2| < h$.
Suppose $u_1 \in [u_{j_1} - h/2,\, u_{j_1} + h/2]$ and $u_2 \in [u_{j_2} - h/2,\, u_{j_2} + h/2]$ with $|j_1-j_2|=1$. Without loss of generality, suppose $j_2>j_1$. 

If $1\le j_1 < j_2 \le D$, 
\begin{align*}
{|f_{v,h}^{(k)}(u_1) - f_{v,h}^{(k)}(u_2)|}
&= C_fh \left| {\mathbf{1}\{k = v_{j_1}\} }\K\left( \frac{u_1-u_{j_1}}{h} \right) - {\mathbf{1}\{k = v_{j_2}\} }K\left( \frac{u_2-u_{j_2}}{h} \right) \right| \\
& \leq C_f h\{\left| \K\left( \frac{u_1-u_{j_1}}{h} \right) - \K\left( \frac{u_2-u_{j_1}}{h} \right) \right| + \left| \K\left( \frac{u_1-u_{j_2}}{h} \right) - \K\left( \frac{u_2-u_{j_2}}{h} \right) \right|\}\\
&\leq C_f h \cdot 4 \left| \frac{u_1-u_2}{h} \right| = 4C_f|u_1-u_2|,
\end{align*}
where we use the fact that $\K$ is 2-Lipschitz continuous.
Note, that $\K((u_2 - u_{j_1})/h)$ and $\K((u_1 - u_{j_2})/h)$ are 0, since $\mathcal{I}_{j_1}$ and $\mathcal{I}_{j_2}$ are disjoint by construction.

If $1\le j_1 \le D < j_2$, then $u_2$ lies outside all bumps so $f_{v,h}^{(k)}(u_2)=1/2$, giving:
\begin{align*}
    {|f_{v,h}^{(k)}(u_1) - f_{v,h}^{(k)}(u_2)|}
&= C_fh \left| \K\left( \frac{u_1-u_{j_1}}{h} \right) \right|  = C_fh \left| \K\left( \frac{u_1-u_{j_1}}{h}  \right)-\K\left( \frac{u_2-u_{j_1}}{h}  \right) \right| \le 2C_f|u_1-u_2|.
\end{align*}

Finally, if $D < j_1 , j_2$, both points lie outside all bumps so ${|f_{v,h}^{(k)}(u_1) - f_{v,h}^{(k)}(u_2)|} = 0$.

\noindent \textbf{Case 3:} $|u_1-u_2| \geq h$ (points separated by at least the bump width), with $u_1 \in \mathcal{I}_{j_1}$ and $u_2 \in \mathcal{I}_{j_2}$, $j_1 \neq j_2$, respectively.
 Then, using the fact that $\K(\cdot)$ is uniformly bounded by 1,
\begin{align*}
    {|f_{v,h}^{(k)}(u_1) - f_{v,h}^{(k)}(u_2)|} &= C_fh \left| {\mathbf{1}\{k = v_{j_1}\} }\K\left( \frac{u_1-u_{j_1}}{h} \right) -{\mathbf{1}\{k = v_{j_2}\} }\K\left( \frac{u_2-u_{j_2}}{h} \right) \right|\\
    &\leq C_fh\{\left| \K\left( \frac{u_1-u_{j_1}}{h} \right)\right| + \left| \K\left( \frac{u_2-u_{j_2}}{h} \right) \right|\}\\
    & \leq 2C_fh\\
    & \leq 2C_f |u_1 - u_2|.
\end{align*}
Similarly, if $1\le j_1\le D < j_2$, \[{|f_{v,h}^{(k)}(u_1) - f_{v,h}^{(k)}(u_2)|} \le C_fh|\K(u_1-u_{j_1})/h| \le C_fh \le C_f|u_1-u_2|,\] and if $D<j_1,j_2$, ${|f_{v,h}^{(k)}(u_1) - f_{v,h}^{(k)}(u_2)|} = 0$, therefore trivially satisfies Lipschitz condition.
Hence we have shown that 
\[{|f_{v,h}^{(k)}(u_1) - f_{v,h}^{(k)}(u_2)| \leq  4C_f|u_1 - u_2|.}\]

\paragraph{Margin condition} 
{
Recall $\Delta(x):=g^{(*)}(x)-g^{(\#)}(x)$ is the mean reward gap of the best arm and the next best arm given $x$. 
}
{
In each interval $\I_j$, only one arm with $k=v_j$ has a positive bump while all other arms have $f_{v,j}^{(k)} = 1/2$. Outside all bumps, $\Delta(x) =0$, so those regions contribute nothing. Therefore, for $\delta>0$:
\begin{align*}
    \P(0 < \Delta(X) \leq \delta) 
    &=\sum_{j=1}^D \P(0 < f_{v,h}^{(v_j)}(X_1)-1/2 \leq \delta, \, X_1 \in \I_j)\\
    & \leq D \bar{c} \int_{\mathcal{I}_1} 1\left\{ \K\left(\dfrac{x_1 - u_1}{h}\right) \leq \dfrac{\delta/C_f}{h}\right\} dx_1 \nonumber\\
    & = D\bar{c}h \int_{[-0.5, 0.5]} 1\{\K(t) \leq (\delta/C_f) h^{-1}\} dt
\end{align*}}
where we used the boundedness of the projected density and a change-of-variable ($t = (x_1-u_1)/h$ and $|t| \leq 0.5$).\\
If $(\delta/C_f) h^{-1} > 1$, note that since $\K$ is non-negative and uniformly bounded by $1$, we get:
\begin{align*}
 \int_{[-0.5, 0.5]} 1\{\K(t) \leq (\delta/C_f) h^{-1}\} dt = 1.   
\end{align*}
If $(\delta/C_f) h^{-1} \leq 1$, for $\K(t)$ constructed as above, observe that for any $t \in [-0.5,0.5]$, note that $0 < \K(t) \le (\delta/C_f) h^{-1}$ implies $|t| \in [1/2 - (\delta/C_f)/(2h), 1/2]$, an interval of length $\delta/h$. Therefore,
\begin{align*}
    \int_{[-0.5, 0.5]} 1\{\K(t) \leq (\delta/C_f) h^{-1}\} dt &= \int_{[-0.5, 0.5]} 1\left\{|t| \in \left[\frac{1}{2} - \frac{\delta/C_f}{2h}, \frac{1}{2}\right]\right\}  dt = (\delta/C_f) h^{-1}.
\end{align*}
Now, combining the two cases, we get that:
{
\begin{align*}
    \P(0 < \Delta(X) \leq \delta) 
    &\leq D\bar{c}h \left[1\{(\delta/C_f) h^{-1} > 1\} +  (\delta/C_f) h^{-1} 1\{\delta h^{-1} \leq C_f\}\right].
\end{align*}
}

Note $0<h\le 1$ and $-1+\alpha \le 0$, $h^{-1+\alpha} \ge 1$, and $D := \lceil h^{-1+\alpha}\rceil \le 2h^{-1+\alpha}$.  We have on $h< (\delta/C_f)$, $h^\alpha < (\delta/C_f)^\alpha$, and on $h \ge (\delta/C_f)$, $h^{-1+\alpha} \le (\delta/C_f)^{-1+\alpha}$.
Therefore,
\begin{align*}
    \P(0 < { \Delta(X)} \leq \delta) 
    &\leq 2h^{\alpha}\bar{c} 1\{(\delta/C_f) h^{-1} > 1\} +  2h^{-1+\alpha}\bar{c}(\delta/C_f)  1\{\delta h^{-1} \leq C_f\}\\
    &\leq 2(\delta/C_f)^{\alpha}\bar{c} 1\{(\delta/C_f) h^{-1} > 1\} + 2 (\delta/C_f)^{-1+\alpha} \bar{c}(\delta/C_f)  1\{\delta h^{-1} \leq C_f\}\\
    &\le 2\bar{c}(\delta/C_f)^{\alpha}.
\end{align*}
Hence, the margin condition holds with exponent $\alpha$ and $D_0 = 2\bar{c}/C_f^\alpha$. 
\end{proof}
{
\begin{lemma}[Proof of Lemma~\ref{lem: perchet_lemma_S_R}]
Define
\begin{align*}
 s_t(\pi) &:= \sum_{s=1}^t 1\{f^{(\pi_s(X_s))}(X_s^\top \beta_0) < \max_{k\in[K]}f^{(k)}(X_s^\top \beta_0)\}   \\
&= \sum_{s=1}^t 1\{\pi_s(X_s)\ne \pi^*(X_s)\}
\end{align*}
so that $S_t(\pi) = \E[s_t(\pi)]$, where $\pi^*(x_s) := \arg\max_{k\in [K]} f^{(k)}(x_s^\top \beta_0).$
We have 
\begin{align*}
    R_t(\pi) &= \sum_{s=1}^t \{\max_{k\in[K]}f^{(k)}(X_s^\top \beta_0)- f^{(\pi_s(X_s))}(X_s^\top \beta_0) \}\ge \sum_{s=1}^t \{ g^{(*)}(X_s) - g^{(\#)}(X_s)\}1(\pi_s(X_s) \ne \pi^*(X_s)),
\end{align*}
where the inequality holds because when $\pi_s(X_s) \ne \pi^*(X_s)$, the regret from the chosen arm $f^{(\pi_s(X_s))}(X_s^\top\beta_0)\le g^{(\#)}(X_s^\top\beta_0)$. Continuing, for any $\delta >0$,
\begin{align*}
\sum_{s=1}^t \{ g^{(*)}(X_s) - g^{(\#)}(X_s)\}1(\pi_s(X_s) \ne \pi^*(X_s))
    &\ge\delta \sum_{s=1}^t 1\{ g^{(*)}(X_s) - g^{(\#)}(X_s)>\delta, \pi_s(X_s) \ne \pi^*(X_s)\}\\
    &\ge \delta [s_t(\pi) -\sum_{s=1}^t 1\{ 0<g^{(*)}(X_s) - g^{(\#)}(X_s)\le \delta\}].
\end{align*}
Taking expectations on both sides and using the margin condition \ref{assum: Margin}, we obtain
\begin{align}\label{eq:lem:r_t_lb}
    \mathcal{R}_t(\pi) \ge \delta [S_t(\pi) - t D_0 \delta^\alpha].
\end{align}
Optimize the RHS of \eqref{eq:lem:r_t_lb} with respect to $\delta$, set
\begin{align*}
    \delta = (\frac{cS_t(\pi)}{(\alpha+1)D_0 t})^{1/\alpha}
\end{align*}
with $0<c<\alpha+1$ chosen small enough that $\delta \le \delta_0$. Then,
\begin{align*}
    \mathcal{R}_t(\pi) \ge \left(\frac{cS_t(\pi)}{(\alpha+1)D_0 t}\right)^{1/\alpha}\left(1-\frac{c}{\alpha+1}\right)S_t(\pi) = C_0 S_t(\pi)^{1+1/\alpha}t^{-1/\alpha}
\end{align*}
for some constant $C_0$ depending on $c,\alpha$ and $D_0$.    
\end{lemma}
}

\subsection{Proof of Theorem \ref{thm: regret_bound}} \label{sec: regret_proof_Thm1}
\begin{proof}
First we construct two events to capture the elimination process. 
Let the batch index $i=1,\dots,M$ be fixed.  For each bin $C  \in \mathcal{B}_i$, we define a ``good batch elimination event'', $\SC$, associated with $C$. 
Note that $C$ may or may not have been born at the beginning of batch $i$, and only undergoes the unique batch elimination event if it was born in the beginning of batch $i$, i.e., when $C \in  \L^{(i)}$ (also ref. Remark \ref{rmk: batch_elimination}). If $C\notin \L^{(i)}$, simply let $\SC = \Omega$ where $\Omega$ is the whole probability space.  
When $C\in \L^{(i)}$, let $\mathcal{I}_{C}$ and $\mathcal{I}^\prime_{C}$ denote the set of active arms associated with $C$ during batch $i$ and end of batch $i$ after batch elimination process, respectively. Note $|\mathcal{I}^\prime_{C}|>1$ will trigger splitting $C$ into its children sets. Define 
\begin{align}\label{def: I_C}
    \underline{\mathcal{I}}_{C} &= \left\{k \in {\I_C}: \sup_{x \in C} \{f^{(*)}(x^\top\beta_0) -f ^{(k)}(x^\top\beta_0)\} \leq c_0 |C|_\T \right\},\\
     \overline{\mathcal{I}}_{C} &= \left\{k \in {\I_C}: \sup_{x \in C} \{f^{(*)}(x^\top\beta_0) -f ^{(k)}(x^\top\beta_0) \} \leq c_1 |C|_\T \right\},
\end{align}
for $c_0 = 4L_0 + 1$ with ${L_0=L(2^{3/2}R_X+ 1)}$, $c_1 = 8c_0\gamma_X^{1/2}$ where $\gamma_X = \overline{c}_X/\underline{c}_X$, and $$f^{(*)}(x^\top\beta_0) = \max_{k \in {[K]}} f^{(k)}(x^\top \beta_0).$$ Note that, $\underline{\mathcal{I}}_{C} \subseteq  \overline{\mathcal{I}}_{C}$. Define a `good event':
\begin{equation}\label{def: SC}
    \mathcal{S}_{C} = \{\underline{\mathcal{I}}_{C} \subseteq \mathcal{I}_{C}^\prime \subseteq \overline{\mathcal{I}}_{C}\}.
\end{equation}
This is a good event because it says that all good arms (with small regret) survive the stage $i$ elimination, and all survived arms in $\mathcal{I}_{C}^\prime$ have not so large regret.
In addition, define 
\begin{equation}\label{def: GC}
\mathcal{G}_{C} = \cap_{C^\prime\in \mathcal{P}(C)} \mathcal{S}_{C'},
\end{equation}
which is the event where the elimination processes were ``good'' for all ancestors of $C$. In the special case when $C$ has no parent since $C \in \mathcal{B}_1$, simply let $\mathcal{G}_{C} = \Omega$.


We decompose the regret into three terms. Recall that $\mathcal{L}_t$ is the set of active bins at $t$. {Also, we define $\mathcal{J}_t := \cup_{s \leq t} \mathcal{L}_s$ for all the bins that were alive at some time point $s \le t$.}
First for a bin $C \in \T$, we define:
\begin{align*}
r_T^l(C) := \sum_{t=1}^T \{g^{(*)}(X_t) - g^{(\pi_t(X_t))}(X_t)\} 1(X_t \in C) 1(C \in \mathcal{L}_t),
\end{align*}
which is the amount of regret on $C$ when $C$ is ``alive'', and also define:
\begin{align*}
    r_T^b(C) := \sum_{t=1}^T (g^{(*)}(X_t) - g^{(\pi_t(X_t))}(X_t)) 1(X_t \in C) 1(C \in \mathcal{J}_t),
\end{align*}
which is the amount of regret on $C$ since $C$ was ``born''. 

There exists a recursive relationship between $r_T^l(C)$ and $r_T^b(C)$, as introduced in \citesupp{perchet2013multi_supp}. We present this relationship as Lemma \ref{lem: r_t^b_decomp} for the convenience of readers and provide a proof in Section \ref{sec: proof for Lemma rtbdecomp}.
\begin{lemma}\label{lem: r_t^b_decomp}
For $C \in \mathcal{B}_i$, for $i=1,\dots,M$, we have 
\begin{align}\label{eq: r_T^b_decomp}
    r_T^b(C) = r_T^l(C)+  \sum_{C'\in\text{child}(C)}r_T^b(C'),
\end{align}
where we adopt the convention that $\sum_{C\in\emptyset}r_T^b(C)=0$. In particular,  \begin{equation*} \sum_{C'\in\text{child}(C)}r_T^b(C') = 0 \  \text{if}  \ C \in \mathcal{B}_M. \end{equation*}
\end{lemma}
From Lemma \ref{lem: r_t^b_decomp}, trivially we obtain,
\begin{align}
    r_T^b(C) &=\left \{r_T^l(C)+  \sum_{C'\in\text{child}(C)}r_T^b(C')\right\}1(\mathcal{S}_{C}) + r_T^b(C) 1(\mathcal{S}_{C}^c) \nonumber\\
    &=  r_T^l(C)1(\mathcal{S}_{C})+ r_T^b(C) 1(\mathcal{S}_{C}^c)+ \sum_{C'\in\text{child}(C)}r_T^b(C')1(\mathcal{S}_{C}) \label{eq: regret_iterative}  
\end{align}
Additionally, we can have the following iterative relationship:
\begin{align}
    &{\sum_{C\in \mathcal{B}_i} \sum_{C'\in \text{child}(C)} r_T^b(C')1({\mathcal{S}_C})1({\mathcal{G}_C})}\label{eq: regret_iterative2}\\
 &=\sum_{C\in \mathcal{B}_i} \sum_{C'\in \text{child}(C)} \Big\{r_T^l(C')1(\mathcal{S}_{C'}) + r_T^b(C') 1(\mathcal{S}_{C'}^c) + \sum_{C''\in\text{child}(C')}r_T^b(C'')1(\mathcal{S}_{C'})\Big\} 1({\mathcal{S}_C})1({\mathcal{G}_C})\nonumber\\
&=\sum_{C'\in \mathcal{B}_{i+1}}\{ r_T^l(C')1({\mathcal{S}_{C'}})+ r_T^b(C')1({\mathcal{S}_{C'}^c})\}1({\mathcal{G}_{C'}}) + {\sum_{C'\in \mathcal{B}_{i+1}} \sum_{C''\in\text{child}(C')}r_T^b(C'')1({\mathcal{S}_{C'}})1({\mathcal{G}_{C'}})}\nonumber
\end{align}
using the fact that $1({\mathcal{S}_C})1({\mathcal{G}_C})=  1({\mathcal{G}_{C'}})$ for $C' \in \text{child}(C)$.

Using \eqref{eq: regret_iterative} and applying \eqref{eq: regret_iterative2} iteratively, and using the fact that $\mathcal{G}_C=\Omega$ for $C \in \mathcal{B}_1$, we have:
\begin{align*}
    R_T(\pi) 
    &= \sum_{C \in \mathcal{B}_1} r_T^b(C)\\
    &= \sum_{C \in \mathcal{B}_1} r_T^l(C) 1(\mathcal{S}_C)1(\mathcal{G}_C) +   \sum_{C \in \mathcal{B}_1}  r^b_T(C)1(\mathcal{S}_C^c)1(\mathcal{G}_C)+\sum_{C \in \mathcal{B}_1}\sum_{C' \in \text{child}(C)} r_T^b(C')1(\mathcal{S}_{C}) 1(\mathcal{G}_C)\\
     &= \sum_{i=1}^2 \sum_{C \in \mathcal{B}_i} \{r_T^l(C) 1(\mathcal{S}_C) +    r^b_T(C)1(\mathcal{S}_C^c)\}1(\mathcal{G}_C)+\sum_{C \in \mathcal{B}_2}\sum_{C' \in \text{child}(C)}  r_T^b(C')1(\mathcal{S}_{C})1(\mathcal{G}_{C})  \\
     & \dots =\sum_{i=1}^{M-1} \sum_{C \in \mathcal{B}_i} \{r_T^l(C) 1(\mathcal{S}_C) +    r^b_T(C)1(\mathcal{S}_C^c)\}1(\mathcal{G}_C) +\sum_{C \in \mathcal{B}_{M-1}}\sum_{C' \in \text{child}(C)}  r_T^b(C')1(\mathcal{S}_{C})1(\mathcal{G}_{C})  \\
     &=\sum_{i=1}^{M-1} \sum_{C \in \mathcal{B}_i} \{r_T^l(C) 1(\mathcal{S}_C) +    r^b_T(C)1(\mathcal{S}_C^c)\}1(\mathcal{G}_C)+\sum_{C \in \mathcal{B}_{M}} r_T^b(C)1(\mathcal{G}_{C}).
\end{align*}

Define the event that we obtain sufficient samples for all $C$ in $\mathcal{B}_i$ for $1\le i \le M-1$:
\begin{align}\label{def: E_nsamples_event}
\mathcal{E}:=\{\forall C \in \cup_{i=1}^{M-1} \mathcal{B}_i, m_{C,i} \in[m_{C,i}^*/2,\,3m_{C,i}^*/2]\}.
\end{align}
We have
\begin{align*}
    R_T(\pi) = R_T(\pi)1(\mathcal{E}^c) + R_T(\pi)1(\mathcal{E}).
\end{align*}

Moreover, for a set $C\in \T$, if $C$ has never been born (i.e., if $C \notin \mathcal{J}_T \iff C \notin \L_t$ for all $1\le t \le T$), $r_T^l(C) = r_T^b(C) = 0$.
Therefore,
\begin{align}
    R_T&(\pi)1(\mathcal{E}) \\
    &= \sum_{i=1}^{M-1} \sum_{C \in \mathcal{B}_i \cap \mathcal{J}_T} r_T^l(C) 1(\mathcal{S}_C \cap \mathcal{G}_C\cap \mathcal{E}) +   \sum_{i=1}^{M-1} \sum_{C \in \mathcal{B}_i \cap \mathcal{J}_T}  r^b_T(C)1(\mathcal{S}_C^c\cap  \mathcal{G}_C\cap \mathcal{E}) +\sum_{C \in \mathcal{B}_{M}\cap \mathcal{J}_T} r_T^b(C)1(\mathcal{G}_{C}\cap \mathcal{E})\nonumber\\
      &\le  \sum_{i=1}^{M-1} \sum_{C \in \mathcal{B}_i \cap \mathcal{J}_T} r_T^l(C) 1(\mathcal{S}_C \cap \mathcal{G}_C) +   \sum_{i=1}^{M-1} \sum_{C \in \mathcal{B}_i \cap \mathcal{J}_T}  r^b_T(C)1(\mathcal{S}_C^c\cap  \mathcal{G}_C\cap \mathcal{E}) +\sum_{C \in \mathcal{B}_{M}\cap \mathcal{J}_T} r_T^b(C)1(\mathcal{G}_{C}). \nonumber
\end{align}
Let, for $i=1,\dots,M-1$,
\begin{align*}
    U_i:= &\sum_{C \in \mathcal{B}_i\cap \mathcal{J}_T} r_T^l(C) 1(\mathcal{S}_{C} \cap \mathcal{G}_{C}), \ \   \ V_i := \sum_{C \in \mathcal{B}_i\cap \mathcal{J}_T} r_T^b(C) 1(\mathcal{S}_{C}^c \cap \mathcal{G}_{C}\cap \mathcal{E}),
\end{align*}
and $W_M =: \sum_{C \in \mathcal{B}_M\cap \mathcal{J}_T} r_T^b(C) 1(\mathcal{G}_C) $ so that 
\begin{equation}\label{eq: regret_breakdown}
    R_T(\pi)1(\mathcal{E}) \le \sum_{i=1}^{M-1} (U_i + V_i) + W_M.
\end{equation}
Next, we bound these three terms, namely, $U_i, V_i$ and $W_M$ separately.

\paragraph{Controlling $U_i$}\label{sec: U_i}
Let us fix some batch $i$, $1 \leq i \leq M-1$, and some bin $C \in \mathcal{B}_i\cap \mathcal{J}_T$. 
Recall that by definition of $\mathcal{B}_i$, $C=C_{A}(\beta)$ for some $A \in \mathcal{A}_i$, where $A \subseteq [L_\beta,U_\beta]$ is an interval of length $w_i$. 
By definition of $r_T(C)$, 
\begin{align*}
    &\E[r_T^l(C) 1(\mathcal{G}_{C} \cap \mathcal{S}_{C})] \\
    &= \E\left[\sum_{t=1}^T \{g^{(*)}(X_t) - g^{(\pi_t(X_t))}(X_t)\} 1(X_t \in C) 1(C \in \mathcal{L}_t) 1(\GC \cap \SC)\right].
\end{align*}
We show that the summand is non-zero only for $t \in [t_{i-1}+1 , t_i]$:
First, since $C \in \mathcal{B}_i$, $C \notin \L_t$ for $t\le t_{i-1}$, i.e., $1(C \in \mathcal{L}_t) =0$ for $t \le t_{i-1}$. This is because $C \in \mathcal{B}_i$ can only born at the beginning of batch $i$, that is when $t = t_{i-1}+1$. 
Now consider $t > t_{i}$. At the end of batch $i$, there are two possibilities:\\
{
1. $|\mathcal{I}'_{C}|>1$: in this case, $C$ is split into its children, and $C \notin \L_t$ for $t >  t_{i}$.\\
2. $|\mathcal{I}'_{C}|=1$: we argue that on $\SC$, the remaining arm is optimal for all $x \in C$, and therefore $g^{(*)}(x) - g^{(\pi_t(x))}(x) = 0$ for $t >t_i$, where we recall that $\pi_t(x)$ is the arm chosen for $x$ by the algorithm.}
{Let $k_1 \in {[K]}$ be any of the eliminated arms and $k_2 \in {[K]}$ be the remaining arm.} On $\SC$, we have $\underline{\mathcal{I}}_{C} \subseteq \mathcal{I}_{C}^\prime = \{k_2\} \subseteq \overline{\mathcal{I}}_{C}$, therefore $k_1 \notin \underline{\mathcal{I}}_{C}$. Then, there exists $x_0 \in  C$ such that $g^{(k_2)}(x_0) - g^{(k_1)}(x_0) > c_0 |C|_\T$.  
For any $x\in C$, 
\begin{align*}
    g^{(k_2)}(x) - g^{(k_1)}(x)  \ge g^{(k_2)}(x_0) - g^{(k_1)}(x_0) - \sum_{k\in {\{k_1,k_2\}}} |g^{(k)}(x) - g^{(k)}(x_0)|.
\end{align*}

By Lemma \ref{lem: barg_min_g}, for sufficiently large $T$,  $|g^{(k)}(x)-g^{(k)}(x_0)| \le L_0 w_i$ for $k\in{[K]}$, 
and therefore
\begin{align*}
g^{(k_2)}(x) - g^{(k_1)}(x) \ge (c_0-2L_0)w_i = (2L_0+1)w_i>0,
\end{align*}
recalling that $c_0 = 4L_0+1$. Therefore $k_2$ is the optimal arm for all $x\in C$. In particular, regret is not incurred for $t > t_i$, i.e.,  $g^{(*)}(X_t) - g^{(\pi_t(X_t))}(X_t)=0$ for $X_t \in C, \,t>t_i$.

Therefore,
\begin{align*}
    &\E[r_T^l(C) 1(\GC \cap \SC)] \\
    &= \E\left[\sum_{t=t_{i-1}+1}^{t_i} \{g^{(*)}(X_t) - g^{(\pi_t(X_t))}(X_t)\} 1(X_t \in C) 1(C \in \mathcal{L}_t) 1(\GC \cap \SC)\right].
\end{align*}
On the event $\GC$, we have that $\mathcal{I}^\prime_{p(C)} \subseteq \overline{\mathcal{I}}_{p(C)}$, that is, for any $k \in \mathcal{I}^\prime_{p(C)}$,
\begin{align*}
    \sup_{x \in p(C)} \{g^{(*)}(x) - g^{(k)}(x) \}\leq c_1 |p(C)|_\T.
\end{align*}
Moreover, regret is only incurred at points where {$\Delta(x):=g^{(*)}(x) - g^{(\#)}(x) > 0$. }
Therefore, on $\GC$, for any $x \in C$ and $k \in \mathcal{I}^\prime_{p(C)}$,
\begin{align}
    g^{(*)}(x) - g^{(k)}(x)  \leq c_1 |p(C)|_\T 1(0 < {\Delta(x)} \leq c_1 |p(C)|_\T).\nonumber 
\end{align}
In particular, for any $X_t \in C$, the inequality
\begin{align}
    g^{(*)}(X_t) - g^{(\pi_t(X_t))}(X_t)  \leq c_1 |p(C)|_\T 1(0 < {\Delta(X_t)} \leq c_1 |p(C)|_\T)\label{eq: eq1_Vi}
\end{align}
holds on $\GC$ when $t > t_{i-1}$, since for $t > t_{i-1}$, 
$\pi_t(X_t)$ can be selected from the (subset of) active arms after the $i-1$ batch elimination, and therefore $\pi_t(X_t) \in  \mathcal{I}'_{p(C)}$. 
Therefore, we obtain,
\begin{align*}
    &\E\left[\sum_{t=t_{i-1}+1}^{t_i} (g^{(*)}(X_t) - g^{(\pi_t(X_t))}(X_t)) 1(X_t \in C) 1(C \in \mathcal{L}_t) 1(\GC \cap \SC)\right]\\
 &\le\sum_{t=t_{i-1}+1}^{t_i} c_1 |p(C)|_\T\E\left[ 1(0 < {\Delta(X_t)} \leq c_1 |p(C)|_\T) 1(X_t \in C) 1(C \in \mathcal{L}_t) 1(\GC \cap \SC)\right]\\
 &\le\sum_{t=t_{i-1}+1}^{t_i} c_1 |p(C)|_\T \P\left( 0 < {\Delta(X_t)} \leq c_1 |p(C)|_\T,\, X_t \in C \right)\\
 &=(t_i - t_{i-1}) c_1 |p(C)|_\T \P\left( 0 < {\Delta(X)} \leq c_1 |p(C)|_\T,\, X \in C \right),
\end{align*}
where the last equality is due to the fact that $X_t\sim \P_X$ iid. 
Finally,
\begin{align*}
    \E [U_i]
    &= \sum_{C\in \mathcal{B}_i \cap \mathcal{J}_T} \E[r_T^l(C) 1(\GC \cap \SC)] \\
    &\le\sum_{C \in \mathcal{B}_i \cap \mathcal{J}_T} (t_i - t_{i-1}) c_1 |p(C)|_\T \P( 0 < {\Delta(X)} \leq c_1 |p(C)|_\T, \, X \in C )\\& \le (t_i - t_{i-1}) c_1|p(C)|_\T \P( 0 < {\Delta(X)} \leq c_1 |p(C)|_\T),
\end{align*}
where for the last equality we use the fact that $\mathcal{B}_i$ is the partition of $\mathcal{X}$. Since $|p(C)|_\T= w_{i-1}$ by the set-up and $\P\left( 0 < {\Delta(X)} \leq c_1 |p(C
)|_\T\right) \le D_0\{c_1|p(C)|_\T\}^\alpha$ by the margin condition in Assumption \ref{assum: Margin}, for $1\le i \le M-1$,
\begin{align} \label{eq: Ui_bound}
    \E[U_i] \le (t_i-t_{i-1}) D_0\{c_1w_{i-1}\}^{1+\alpha}.
\end{align}

\paragraph{Controlling $V_i$}\label{sec: V_i}
Similarly, choose some $1 \leq i \leq M-1$ and bin $C \in \mathcal{B}_i \cap \mathcal{J}_T$. We have $C= C_A(\beta)$ for some $A\in\mathcal{A}_i$. 
We have from the definition of $r_T^b(C)$,
\begin{align}
    \E&[r_T^b(C) 1(\GC \cap \SC^c  \cap \mathcal{E})] \nonumber\\
    &= \E\Bigg[\sum_{t=1}^T (g^{(*)}(X_t) - g^{(\pi_t(X_t))}(X_t)) 1(X_t \in C) 1(C\in \mathcal{J}_t) 1(\GC \cap \SC^c \cap \mathcal{E}) \Bigg] \nonumber\\
     &= \E\Bigg[\sum_{t=t_{i-1}+1}^T (g^{(*)}(X_t) - g^{(\pi_t(X_t))}(X_t)) 1(X_t \in C) 1(C\in \mathcal{J}_t) 1(\GC \cap \SC^c \cap \mathcal{E}) \Bigg]\nonumber\\
     &\le c_1 |p(C)|_\T \E\Bigg[\sum_{t=t_{i-1}+1}^T   1(0 < {\Delta(X_t)} \leq c_1 |p(C)|_\T, X_t \in C) 1(\GC \cap \SC^c \cap \mathcal{E}) \Bigg], \label{eq: eq2_Vi}
 \end{align}
where for the second equality we use the fact that $C \notin \mathcal{J}_t$ for $t \le t_{i-1}$, since $C \in \mathcal{B}_i$ can be born only at batch $i$ and we use \eqref{eq: eq1_Vi} for the last inequality.

We note that $\GC \cap \SC^c \cap \mathcal{E}$ is independent of $\{X_t; t > t_i\}$. 
This is because $\GC = \cap_{C \in \mathcal{P}(C)}\SC$, therefore it only depends on (random) batch elimination events up to $i-1$ batch, i.e., $\GC $ only depends on $\{(X_t,Y_t); 1 \le t\le t_{i-1}\}$, and $\SC$ depends on batch elimination event at the end of batch $i$, and therefore depends on $\{(X_t,Y_t); t_{i-1}+1 \le t\le t_{i}\}$. 
Therefore,
\begin{align*}
&\E\left[\sum_{t=t_{i-1}+1}^T   1(0 < {\Delta(X_t)} \leq c_1 |p(C)|_\T, X_t \in C )  1(\GC \cap \SC^c \cap \mathcal{E}) \right]    \\
&=\sum_{t=t_{i-1}+1}^{t_i} \E\left[  1(0 < {\Delta(X_t)} \leq c_1 |p(C)|_\T, X_t \in C)  1(\GC \cap \SC^c \cap \mathcal{E}) \right] \\
&\quad +\sum_{t=t_{i}+1}^T\P\left[  0 < {\Delta(X_t)} \leq c_1 |p(C)|_\T, X_t \in C\right]  \P (\GC \cap \SC^c \cap \mathcal{E}) \\
&\le \sum_{t=t_{i-1}+1}^{t_i} \P\left[   0 < {\Delta(X_t)} \leq c_1 |p(C)|_\T, X_t \in C \right] \\
&\quad +\sum_{t=t_{i}+1}^T\P\left[  0 < {\Delta(X_t)} \leq c_1 |p(C)|_\T, X_t \in C\right]  \P (\GC \cap \SC^c \cap \mathcal{E}), 
\end{align*}
where for the last inequality we  use $ 1(\GC \cap \SC^c \cap \mathcal{E}) \le 1$ a.s.
Therefore, using this in \eqref{eq: eq2_Vi} we obtain,
\begin{align*}
      \E&[r_T^b(C) 1(\GC \cap \SC^c \cap \mathcal{E})] \\
      &\le c_1 |p(C)|_\T \{ (t_i-t_{i-1}) +(T-t_i)\P (\GC \cap \SC^c \cap \mathcal{E}) \}  \P\left[   0 < {\Delta(X)} \leq c_1 |p(C)|_\T, X\in C \right].
\end{align*}

From Lemma \ref{lem: bound_Gc_cap_SC^c}, we have that $P(\GC \cap \SC^c \cap \mathcal{E}) \leq \dfrac{3 m_{C,i}^*}{2T|C|_\T}$. 
Recalling the definition $m_{C,i}^*=\E[\sum_{t=t_{i-1}+1}^{t_i} 1\{X_t \in C\}] = (t_i - t_{i-1}) P_X(C)\le (t_i - t_{i-1})\overline{c}_X |C|_\T$ for a sufficiently large $T$, 
we have 
\begin{align*}
    (T-t_i)\P (\GC \cap \SC^c \cap \mathcal{E}) \le \frac{ (T-t_{i-1})\{3\overline{c}_X (t_i - t_{i-1}) |C|_\T\} }{2T|C|_\T} \le 3\overline{c}_X (t_i - t_{i-1}).
\end{align*}
Then, using the fact that $\mathcal{B}_i$ is the partition of $\mathcal{X}$, and Assumption \ref{assum: Margin}, we obtain:
\begin{align}
    \E[V_i] &=\sum_{C\in \mathcal{B}_i\cap \mathcal{J}_T}  \E[r_T^b(C) 1(\GC \cap \SC^c \cap \mathcal{E})]\nonumber\\
    &\le c_1 |p(C)|_\T  (t_i-t_{i-1}) (3\overline{c}_X+1) \P\left[   0 < {\Delta(X)} \leq c_1 |p(C)|_\T \right]\nonumber\\
    &\le D_0  \{c_1 w_{i-1}\}^{1+\alpha} (3\overline{c}_X +1)(t_i-t_{i-1}).\label{eq: Vi_bound}
\end{align}

\paragraph{Controlling $W_M$} \label{sec: W_M}
Finally, for $C = C_A(\beta) \in \mathcal{B}_M\cap \mathcal{J}_T$ with $A \in \mathcal{A}_M$, since $C \in \mathcal{J}_t$ only for $t>t_{M-1}$,
\begin{align*}
 \E&[r_T^b(C) 1(\GC) ]\\
  &=  \E[\sum_{t=1}^T \{g^{(*)}(X_t) - g^{(\pi_t(X_t))}(X_t)\}1(X_t \in C) 1(C \in \mathcal{J}_t)1(\GC) ]\\
  &= \E[ \sum_{t=t_{M-1}+1}^T \{g^{(*)}(X_t) - g^{(\pi_t(X_t))}(X_t)\} 1(X_t \in C) 1(\GC) ]\\
  &\le \E\Big[ \sum_{t=t_{M-1}+1}^T  c_1 |p(C)|_\T  1(0 < {\Delta(X_t)} \leq c_1 |p(C)|_\T,\,  X_t \in C) 1(\GC) \Big]\\
   &\le \sum_{t=t_{M-1}+1}^T c_1 |p(C)|_\T \P(   0 < {\Delta(X_t)} \leq c_1 |p(A)| \,, X_t \in C), 
\end{align*}
where the first inequality is due to \eqref{eq: eq1_Vi}.
In particular,
\begin{align}
 \E[W_M] &=  \sum_{C \in \mathcal{B}_M\cap \mathcal{J}_T} \E[r_T^b(C) 1(\GC)]\nonumber\\
 &\le (T-t_{M-1}) c_1|p(C)|_\T \P(0 < {\Delta(X)} \leq c_1 |p(C)|_\T) \nonumber \\
 &\le (T-t_{M-1}) D_0 \{c_1 w_{M-1}\}^{1+\alpha}. \label{eq: WM_bound}
\end{align}

\paragraph{Regret upper bound.}
Putting the results from \eqref{eq: Ui_bound}, \eqref{eq: Vi_bound} and \eqref{eq: WM_bound} together in \eqref{eq: regret_breakdown}, we get, 
\begin{align*}
    \E[R_T(\pi) 1(\mathcal{E})]
    & \le\sum_{1 \leq  i \le M-1} \{\E[U_i] +\E[V_i]\} + \E[W_M]\\
    & \le \sum_{1\le i\le M-1}   D_0 (3\overline{c}_X +2) \{c_1 w_{i-1}\}^{1+\alpha}(t_i-t_{i-1})+  D_0 \{c_1 w_{M-1}\}^{1+\alpha} (T-t_{M-1}).
\end{align*}
{
First, note that with the choice of $b_0 \asymp \{T/(K\log T)\}^{\frac{1-\gamma}{3(1-\gamma^M)}}$ and $K=O(\log T)$,
\begin{align*}
    KTw_i \asymp KTb_0^{-\frac{1-\gamma^i}{1-\gamma}}\asymp KT\left(\frac{T}{K\log T}\right)^{-\frac{1-\gamma^i}{3(1-\gamma^M)}}\lesssim T.
\end{align*}
By the choice of the bind widths and batch sizes in \eqref{eq: wi_formula} and\eqref{eq: batch_size}, we have,
\begin{align*}
t_i-t_{i-1}&\asymp Kw_i^{-3} {\log{(KT w_i)}}\lesssim Kw_i^{-3}\log(T),\\
w_i^{-3}w_{i-1}^{ (1+\alpha)} &\asymp b_0^{3(\frac{1-\gamma^i}{1-\gamma})-3\gamma(\frac{1-\gamma^{i-1}}{1-\gamma})}  = b_0^3.
\end{align*}
Therefore for $1\le i \le M-1$, we have
\begin{align*}
    (t_i -t_{i-1} ) w_{i-1}^{ (1+\alpha)} \lesssim K\log (T) w_i^{-3} w_{i-1}^{ (1+\alpha)} \lesssim K\log(T)\left(\frac{T}{K\log T}\right)^{\frac{1-\gamma}{1-\gamma^M}}=(K\log T)^{1-\beta_M}T^{\beta_M}
\end{align*}
where we recall $\beta_M:=\frac{1-\gamma}{1-\gamma^M}$.
For the last term, since
\begin{align*}
    w_{M-1}^{1+\alpha} \asymp b_0^{-3\gamma(\frac{1-\gamma^{M-1}}{1-\gamma})}\asymp\left(\frac{T}{K \log T}\right)^{-(\frac{\gamma-\gamma^M}{1-\gamma^M})},
\end{align*}
we obtain,
\begin{align*}
   (T-t_{M-1}) w_{M-1}^{(1+\alpha)}  \lesssim   T^{1-(\frac{\gamma-\gamma^M}{1-\gamma^M})}(K\log T)^{\frac{\gamma-\gamma^M}{1-\gamma^M}} =(K\log T)^{1-\beta_M}T^{\beta_M}.
\end{align*}
Therefore,
\begin{align*}
     \E[R_T(\pi) 1(\mathcal{E})] \lesssim  M (K\log T)^{1-\beta_M}T^{\beta_M}.
\end{align*}
}
On the other hand, since we have $|Y_i| \le 1$, 
\begin{align*}
    \E[R_T(\pi)1(\mathcal{E}^c)] \le 2T \P(\mathcal{E}^c) \le 2,
\end{align*}
by Lemma \ref{lem:mCA_bound}. Therefore, we prove the result of Theorem \ref{thm: regret_bound}. 
\end{proof}
\subsubsection{Proof for Lemma \ref{lem: r_t^b_decomp}} \label{sec: proof for Lemma rtbdecomp}
\begin{proof}
There exists three cases for $C \in \mathcal{B}_i$ for $i=1,\dots,M-1$.
\begin{enumerate}
\setlength\itemsep{0em}
    \item $C$ is not born at the beginning of batch $i$,
    \item $C$ is born at the beginning of batch $i$, and is not split into its children sets after the batch elimination at the end of batch $i$, and
    \item $C$ is born at the beginning of batch $i$, and is split into its children sets after the batch elimination at the end of batch $i$.
\end{enumerate}
In case 1, $C$ is never born, i.e., $C \notin \L_t$ for all $1\le t \le T$, as a set $C \in \B_i$ can be born only at batch $i$ by the set up of the algorithm. Moreover, since $C$ is not born, its child $C' \in \text{child}(C)$ will not be born. Therefore $r_T^b(C)=r_T^l(C) =r_T^b(C')= 0$, and equation \eqref{eq: r_T^b_decomp} is trivially true.
In case 2, $C \notin \mathcal{J}_t$ for $t \le t_{i-1}$ (before batch $i$) and $C \in \L_t$ for $t \ge t_{i-1}+1$ (batch $i$ and onward). Therefore, 
\begin{align*}
    r_T^b(C) &= \sum_{t=t_{i-1}+1}^{T} \{g^{(*)}(X_t) - g^{(\pi_t(X_t))}(X_t)\} 1(X_t\in C)1(C\in\mathcal{J}_t)\\
    &= \sum_{t=t_{i-1}+1}^{T} \{g^{(*)}(X_t) - g^{(\pi_t(X_t))}(X_t)\} 1(X_t\in C)1(C\in\mathcal{L}_t)= r_T^l(C).
\end{align*}
Since $\text{child}(C) \notin \mathcal{J}_t$ for all $t$ ($C$ is not split), $r_T^b(C') =0$  for any $C'\in \text{child}(C)$, and therefore equation \eqref{eq: r_T^b_decomp} holds. 
In the last case,
\begin{align*}
    r_T^b(C) &= \sum_{t=t_{i-1}+1}^{t_i} \{g^{(*)}(X_t) - g^{(\pi_t(X_t))}(X_t)\}1(X_t\in C)1(C\in\mathcal{L}_t)
    \\ &\qquad+ \sum_{t=t_{i}+1}^{T} \{g^{(*)}(X_t) - g^{(\pi_t(X_t))}(X_t)\}1(X_t\in C)1(C\in\mathcal{J}_t)\\
    &=\sum_{t=t_{i-1}+1}^{t_i}  \{g^{(*)}(X_t) - g^{(\pi_t(X_t))}(X_t)\}1(X_t\in C)1(C\in\mathcal{L}_t)
    \\ & \qquad + \sum_{t=t_{i}+1}^{T} \sum_{C'\in\text{child}(C)}\{g^{(*)}(X_t) - g^{(\pi_t(X_t))}(X_t)\}1(X_t\in C')1(C'\in\mathcal{J}_t)\\
    &=r_T^l(C)+  \sum_{C'\in\text{child}(C)}r_T^b(C'),
\end{align*}
where the second equality is due to the fact that $C = \cup_{C'\in\text{child}(C)} C'$ and children sets are disjoint, and $1(C \in \mathcal{J}_t) = 1(C'\in \mathcal{J}_t) = 1$ for $t_i+1 \le t \le T$.
Therefore, \begin{align*}
    r_T^b(C) = r_T^l(C)+  \sum_{C'\in\text{child}(C)}r_T^b(C').
\end{align*}
The equation \eqref{eq: r_T^b_decomp} is also true for $i=M$, where only the first two cases happen, and we treat $\sum_{C'\in \text{child}(C)} r_T^b(C') = \sum_{C'\in\emptyset} r_T^b(C') = 0$. 
\end{proof}
\subsection{Proof of Lemma \ref{lem: estimated_beta_bound}} \label{sec: proof_lem_estimated_beta_bound}
\begin{proof}
Let {$\hat{\beta}^{(1)},\,\dots,\hat{\beta}^{(K)}$} be the estimated index vectors. Let $n_k$ be the number of samples used for $\hat{\beta}^{(k)}$ for $k\in{[K]}$. 
{
By the setup of Algorithm \ref{algorithm: Initial_Dir_Estimation}, for a sufficiently large $t_{\rm init}$, we have,
\begin{equation*}
    \frac{t_{\rm init}}{2K} < n_k < \frac{2t_{\rm init}}{K}
\end{equation*}
and therefore by Assumption \ref{assum: index_rate}, for any $\delta \in (0, K/(2t_{\rm init}))$, the following inequality holds for all $k\in [K]$ with probability at least $1-K\delta$:
\begin{align}\label{eq: sin_bk_bound}
     \sin \angle \hat{\beta}^{(k)},\beta_0 \le  C_{\rm idx}\frac{\textrm{polylog}(2t_{\rm init}/(K\delta))}{\sqrt{t_{\rm init}/(2K)}}\le  C'_{\rm idx}\frac{\sqrt{K}\textrm{polylog}(t_{\rm init}/(K\delta))}{\sqrt{t_{\rm init}}}.
\end{align}
for another constant $C'_{\rm idx}$.
}

Note for any $u,v$ such that $\|u\|_2 = \|v\|_2=1$,
\begin{equation}\label{eq: proj_sin_theta}
    \|uu^\top - vv^\top \|_F^2 = 2-2(u^\top v)^2 = 2(\sin \angle u,v)^2,
\end{equation}
since $\cos (\angle u,v) = |u^\top v|$ by the definition of the principal angle between $u$ and $v$. 

Then, for $\hat{\mathcal{P}} = \sum_{k=1}^K \omega_k \hat{\beta}^{(k)}(\hat{\beta}^{(k)})^\top$ with $\sum_k \omega_k = 1$, { the following inequality holds with probability at least $1-K\delta:$}
\begin{align}
    \|\hat{\mathcal{P}} - \mathcal{P}_0\|_F &= \|\sum_{k=1}^{{K}} \omega_k 
    \{\hat{\beta}^{(k)}(\hat{\beta}^{(k)})^\top - \beta_0\beta_0^\top\} \|_F \nonumber\\
    &\le  \sum_{k=1}^{{K}} \omega_k \| \hat{\beta}^{(k)}(\hat{\beta}^{(k)})^\top - \beta_0\beta_0^\top \|_F \nonumber\\
    & \le { \sqrt{2}C'_{\rm idx} K^{1/2}\frac{\textrm{polylog}(t_{\rm init}/(K\delta))}{\sqrt{t_{\rm init}}}.}\label{eq: Phat_bound}
\end{align}
Then by a variant of the Davis-Kahan inequality (Theorem 2 in \citesupp{yu2015useful}) with $r=s=1$ and the bound \eqref{eq: Phat_bound}, we have,
\begin{equation*}
    \sin \angle \hat{\beta},\beta_0 = 2 \|\hat{\mathcal{P}} - \mathcal{P}_0\|_F \le 2^{3/2}{C'_{\rm idx} K^{1/2}\frac{\textrm{polylog}(t_{\rm init}/(K\delta))}{\sqrt{t_{\rm init}}}}.
\end{equation*}
Taking $\tilde{C}_{\rm idx} = 2^{3/2}C'_{\rm idx}$, we obtain the first inequality.

For the second inequality, note that for any $u,v$ such that $\|u\|_2 = \|v\|_2=1$, if $u^\top v \ge 0$, we have
\begin{equation}\label{eq: uv_ineq1}
    \|u -v\|_2^2 = 2(1-u^\top v) \le 2(1-(u^\top v)^2) = 2 (\sin\angle u,v)^2.
\end{equation}
On the other hand, if $u^\top v \le 0$, we have,
\begin{align}\label{eq: uv_ineq2}
    \|u +v\|_2^2 \le \|uu^\top - vv^\top \|_F^2 = 2(\sin \angle u,v)^2,
\end{align}
which can be obtained by replacing $v$ with $-v$ in \eqref{eq: uv_ineq1}. In particular, there exists $\hat{o} = \mathrm{sgn}(\hat{\beta}^\top \beta_0) \in \{-1,1\}$ such that
\begin{equation*}
    \| \hat{\beta}\cdot \hat{o} - \beta_0\|_2 \le \sqrt{2} \sin \angle \hat{\beta}, \beta_0 \le 2^{1/2}{\tilde{C}_{\rm idx} K^{1/2}\frac{\textrm{polylog}(t_{\rm init}/(K\delta))}{\sqrt{t_{\rm init}}}}.
\end{equation*}
\end{proof}

\subsection{Proof of Theorem \ref{thm: thm2}}\label{sec: regret_proof_thm2}
\begin{proof}
We know from \eqref{eq: regret_phase1_bd_phase2} that,
\begin{align*}
    \mathcal{R}_T(\pi) \le 2t_{\rm init} + \mathcal{R}_{T-t_{\rm init}}(\pi;\beta).
\end{align*}
Define $\mathcal{E}_\beta$ to  be the event that the inequality \eqref{eq: estimation_rate_for_beta_wi} holds for all $k \in {[K]}$, which holds with probability at least {$1- K\delta$} under Assumption \ref{assum: index_rate}, {provided that $\delta \in (0, K/(2t_{\rm init})]$}. 
{
Choose $\delta = 1/(KT).$
Since $t_{\rm init} \asymp  K^{1/3}{\rm polylog} (T)T^{2/3}$ and $K=O(\log T)$, for a sufficiently large $T$, $\delta \le K/(2t_{\rm init})$ since
\begin{equation*}
    \frac{1}{KT} \lesssim \frac{K}{K^{2/3}{\rm polylog} (T)T^{2/3}}.
\end{equation*}
Therefore, on $\mathcal{E}_\beta$, by Lemma \ref{lem: estimated_beta_bound} and the choice of $t_{\rm init}$ (ref. Equation \ref{eq: tinit_bound}), 
\begin{align*}
    \sin \angle \hat{\beta}, \beta_0 \le   \tilde{C}_{\rm idx} \frac{\sqrt{K}\textrm{polylog}(t_{\rm init}\cdot T)}{\sqrt{t_{\rm init}}}\lesssim \left( \frac{T}{K\log T}\right)^{-1/3}.
\end{align*}
}
Therefore, we have,
\begin{align*}
    \mathcal{R}_T(\pi) 
    &\le 2t_{\rm init} + \E[R_{T-t_{\rm init}}(\pi;\beta)1(\mathcal{E}_\beta) + R_{T-t_{\rm init}}(\pi;\beta)1(\mathcal{E}_\beta^c)]\\
    &\le 2t_{\rm init} + \E[R_{T-t_{\rm init}}(\pi;\beta)1(\mathcal{E}_\beta)] + (T-t_{\rm init}){K\delta}.
\end{align*}
{
Then by Theorem \ref{thm: regret_bound},
\begin{align*}
    \E[R_{T-t_{\rm init}}(\pi;\beta)1(\mathcal{E}_\beta)] \lesssim (M-1)\log\{K(T-t_{\rm init})\}^{1-\beta_M} (T-t_{\rm init})^{\beta_M}.
\end{align*}
Then,
\begin{align}
    \mathcal{R}_T(\pi) 
    &\lesssim  K^{1/3}{\rm polylog} (T)T^{2/3} + M\{K\log(T)\}^{1-\beta_M} T^{\beta_M} + TK \cdot (1/KT)\label{eq: thm2_RT}\\
    &\lesssim  {\rm polylog} (T) \max\{K^{1/3}T^{2/3} , K^{1-\beta_M}T^{\beta_M}\}\nonumber,
\end{align}
where we use the fact that the first term dominates the third term in \eqref{eq: thm2_RT}.}
\end{proof}
\subsection{Supporting Lemmas}

\begin{lemma} \label{lem: Chernoff_bound}
\textbf{Multiplicative Chernoff Bound: }
 Suppose $X_1, ..., X_n$ are independent random variables taking values in $\{0, 1\}$. Let $X$ denote their sum and let $\mu = \E[X]$ denote the sum's expected value. Then for any  $\delta > 0$,
\begin{align*}
    \P(|X - \mu| \geq \delta \mu) \leq 2 e^{-\delta^2 \mu/3}.
\end{align*}
\end{lemma}
More details on multiplicative Chernoff bound and its extensions can be found in \citesupp{kuszmaul2021multiplicative}. Next, we use the multiplicative Chernoff bound to provide a concentration result on the number of covariates falling in a bin contained in the tree $\mathcal{T}$.
\begin{lemma}\label{lem:mCA_bound}
Suppose Assumption \ref{assum: cond_X} holds. Suppose $M \le C_1 \log T$ for some $C_1 >0$. Suppose Assumption \ref{assum: initial_beta_rate} holds.
For a sufficiently large $T$, for all $1 \le i \le M-1$ and $C \in \mathcal{B}_i$, we have $m_{C,i} \in[m_{C,i}^*/2,3m_{C,i}^*/2]$ with probability at least $1/T$, i.e.,
\begin{align*}
    \P(\forall C \in \cup_{i=1}^{M-1} \mathcal{B}_i, \,\,m_{C,i} \in[m_{C,i}^*/2,\, 3m_{C,i}^*/2]) \ge 1-\frac{1}{T}
\end{align*}
where we define $m_{C,i} = \sum_{t=t_{i-1}+1}^{t_i} 1\{X_t \in C\}$ as the number of times $X_t$ visits $C$ during batch $i$, and $m^*_{C,i} = \E[m_{C,i}]$.
\end{lemma}
\begin{proof}
Let $i\in \{1,\dots,M-1\}$ be given, and choose a set $C \in \mathcal{B}_i$. We have $C = C_A(\beta)$ with $A \in \mathcal{A}_i$. 
Let $\mathcal{E}_C$ be the event that $m_{C,i} \in[m_{C,i}^*/2,\, 3m_{C,i}^*/2]$.  Using the multiplicative Chernoff bound from Lemma \ref{lem: Chernoff_bound}, using $\delta = \frac{1}{2}$, we get:
\begin{align*}
    \P(|\sum_{t=t_{i-1}+1}^{t_i} 1\{X_t \in C_A(\beta)\} - m_{C,i}^* | \ge \frac{m_{C,i}^*}{2} ) \le 2 \exp(-\frac{m_{C,i}^*}{12}).
\end{align*}
as each $1\{X_t \in C_A(\beta)\}  \in [0,1]$ a.s.  Note since $(X_t)$ are iid, 
\begin{align*}m_{C,i}^* =  \sum_{t=t_{i-1}+1}^{t_i} \P(X_t\in C_A(\beta)) = (t_i- t_{i-1}) \P_X(C_A(\beta)).\end{align*} Also, note that $\P_X(C_A(\beta)) = \P(X^\top\beta \in A) = \P(X^\top (-\beta) \in -A).$
Defining $A_{sgn} = A$ if $\beta_{sgn} = \beta$ and $-A$ otherwise, we have $\P_X(C_A(\beta)) = P(X^\top\beta_{sgn} \in A_{sgn}) = \int_{u \in A_{sgn}} f_{x^\top\beta_{sgn} }(u) du$. 
{Suppose $T$ is sufficiently large so that $\beta_{sgn} \in \mathbb{B}_2(R_0; \beta_0)$ for $R_0>0$ defined in Assumption \ref{assum: cond_X}. Then,}
\begin{align}\label{eq: P_X(C)_bounds}
   \underline{c}_X|A|\le  \P_X(C_A(\beta)) \le \overline{c}_X |A|
\end{align}
by Assumption \ref{assum: cond_X}. Therefore, $m^*_{C,i} \ge \underline{c}_X (t_i- t_{i-1})|A|$, and 
\begin{align*}
    P(\mathcal{E}_C^c) \le 2 \exp(-m_{C,i}^*/12)\le 2\exp(-\{(t_i- t_{i-1})\underline{c}_X |A|\}/12).
\end{align*}
For $1\leq i \leq M-1,$ {$t_i - t_{i-1} = \lfloor c_B  \,  K w_i^{-3} \log{(2KTw_i)}\rfloor\asymp K|A|^{-3}\log (KT|A|) $}, since $|A| = w_i$ and $c_B$ do not depend on $T$. Also, recall that $|A|^{-1} = w_i^{-1} =(b_0b_1\cdots b_{i-1})/(U_\beta-L_\beta)$ for $(b_i)_{i=1}^{M-1}$ defined in \eqref{eq: bi_formula}. In particular, for sufficiently large $T$, $b_i \ge 1$ for all $i$, and 
{
\begin{align}\label{eq: mstar_Ci_order}
    \frac{\underline{c}_X}{12}(t_i- t_{i-1}) |A| \asymp K|A|^{-2}\log (KT|A|) \gtrsim |A|^{-2} \gtrsim b_0^{2} \asymp \{T/(K\log T)\}^{(\frac{1-\gamma}{1-\gamma^M})(\frac{2}{3})}.
\end{align}
Since $\frac{2(1-\gamma)}{3(1-\gamma^M)} > 0$, the right-hand side grows polynomially in $T/(K\log T)$, and therefore dominates $\log T$ for sufficiently large $T$. }Therefore, for a sufficiently large $T$, $\frac{\underline{c}_X}{12}(t_i- t_{i-1}) |A| \ge 3\log (T)$, and $P(\mathcal{E}_i^c) \le 2/T^3$.

Now we obtain a union bound over all sets in $\cup_{i=1}^{M-1}\B_i$ . Recall the number of sets in $\B_i$ is $n_i = \prod_{l=0}^{i-1} b_l$, and thus the total number of sets in $\cup_{i=1}^{M-1}\B_i$ is $\sum_{i=1}^{M-1} n_i = \sum_{i=1}^{M-1} \prod_{l=0}^{i-1}b_l \le M \prod_{l=0}^{M-2}b_l$.  Therefore, we have
\begin{align*}
    \P(\exists C \in \cup_{i=1}^{M-1} \mathcal{B}_i \mbox{ s.t. } m_{C,i} \notin[m_{C,i}^*/2,3m_{C,i}^*/2]) \le \sum_{C \in \cup_{i=1}^{M-1} \mathcal{B}_i} P(\mathcal{E}_C^c) \le \frac{2M}{T^3} \prod_{l=0}^{M-2}b_l.
\end{align*}
Since $\prod_{l=0}^{M-2}b_l = b_0^{1+\gamma+\cdots+ \gamma^{M-3}} = b_0^{\frac{1-\gamma^{M-2}}{1-\gamma}} \asymp  {\{T/(K\log T)\}}^{(\frac{1-\gamma^{M-2}}{1-\gamma^M})(\frac{1}{3})} \lesssim T$ and $M = O(\log T)$, 
\begin{align*}
    P(\exists C \in \cup_{i=1}^{M-1} \mathcal{B}_i \mbox{ such that } m_{C,i} \notin[m_{C,i}^*/2,\,3m_{C,i}^*/2]) \lesssim \frac{\log T}{T^2}  \le \frac{1}{T},
\end{align*}
when $T$ is sufficiently large.
\end{proof}

\begin{lemma} \label{lem: barg_min_g}
For $i=1,\dots,M-1$, choose $C \in \mathcal{B}_i$. Suppose Assumptions \ref{assum: Smoothness} and \ref{assum: cond_X} hold. 
Also assume Assumption \ref{assum: initial_beta_rate} {holds with $\rho_{\rm pilot} \le C_0\{T/(K\log T)\}^{-1/3}$ for some constant $C_0>0$.}
%
For each $k \in {[K]}$, define $ \bar{g}_C^{(k)} = \frac{1}{\P_X(C)} \int_{x\in C} g^{(k)} (x) d\P_X(x)$.
For any $x,y \in C$, $k\in{[K]}$, the following hold {for a sufficiently large $T$:}
\begin{enumerate}
    \item $|g^{(k)}(x) - g^{(k)}(y)| \le L\{2^{3/2}R_X{\rho_{\rm pilot}}+w_i\} $ and
    \item $|\bar{g}_{C}^{(k)}-g^{(k)}(x)|  \leq L\{2^{3/2}R_X {\rho_{\rm pilot}} + w_i\}$.
\end{enumerate}
In particular, for a sufficiently large $T$,
\begin{equation*}
    |g^{(k)}(x)-g^{(k)}(y)|\le L_0 w_i \quad \mbox{and}\quad
|\bar{g}_{C}^{(k)}-g^{(k)}(x)| 
\leq  L_0 w_i 
\end{equation*}
for $L_0:=L(2^{3/2}R_X+ 1)$.
\end{lemma}
\begin{proof}
We have $C = C_A(\beta)$ for an $A \in \A_i$. We have
\begin{align*}
\left|\bar{g}_{C}^{(k)}-g^{(k)}(x)\right| 
&=\left|\frac{1}{\P_X(C)} \int_{y \in C} g^{(k)} (y)-g^{(k)}(x) d \P_X(y)\right|
\end{align*}
by definition. 

Since for any $x,y\in C$,  we have $x^\top\beta \in A$ and $y^\top\beta \in A$ by the set-up of $C$. In particular, $|x^\top \beta - y^\top \beta |  = |x^\top \beta_{sgn} - y^\top \beta_{sgn} | \le |A|$. 
For any $x,y\in C$ we have,
\begin{align}
|g^{(k)}(x) - g^{(k)}(y)| &= 
  |  f^{(k)}(x^\top\beta_0)-f^{(k)}(y^\top\beta_0) |\nonumber \\
  &\le L|x^\top\beta_0-y^\top\beta_0| \nonumber\\
  &\le L\{|(x-y)^\top\beta_{sgn}|+|(x-y)^\top(\beta_{sgn}-\beta_0)|\} \nonumber\\
  & \le L\{|A| + \|x-y\|_2 \|\beta_{sgn} - \beta_0\|_2 \} \nonumber\\
  &\le L\{|A| + 2^{3/2}R_X{\rho_{\rm pilot}}\},  \nonumber
\end{align}
where we use the smoothness condition of $f^{(k)}$ in Assumption \ref{assum: Smoothness}, Assumption \ref{assum: cond_X} to bound $\|y-x\|_2 \le 2R_X$, and Assumption \ref{assum: initial_beta_rate} to bound $\|\beta_{sgn} - \beta_0\|_2 {\le 2^{1/2}\rho_{\rm pilot}}$.  Therefore,
\begin{align*}
    |\bar{g}_{C}^{(k)}-g^{(k)}(x)|  
    &\leq \frac{1}{\P_X(C)} \int_{y \in C} L\{w_i +2^{3/2}R_X {\rho_{\rm pilot}} \} d \P_X(y)\\
    &\leq L\{w_i + 2^{3/2}R_X {\rho_{\rm pilot}}\}
\end{align*}
which establishes items $1$ and $2$. For the final simplification, recall from \eqref{eq: wi_formula} that 
{
\begin{equation*}
    w_i \asymp \left(\frac{T}{K\log T}\right)^{-\frac{1-\gamma^i}{3(1-\gamma^M)}},\,\,i=1,\dots,M-1.
\end{equation*}
Since $\rho_{\rm pilot} = O(\{T/(K\log T)\}^{-1/3})\lesssim w_i$, for a sufficiently large $T$, $\rho_{\rm pilot} \le w_i$ for all $i=1,\dots,M-1$. For such $T$,
}
 \begin{align}\label{eq: gkx-gky_diff}
    |\bar{g}_{C}^{(k)}-g^{(k)}(x)|  \le \sup_{x,y\in C} |g^{(k)}(y)-g^{(k)}(x)|
    &\leq   L(2^{3/2}R_X+ 1) w_i = L_0 w_i.
\end{align}

\end{proof}

\begin{lemma}
\label{lem: bound_Gc_cap_SC^c}
Let $C \in \cup_{l=1}^{M-1}\mathcal{B}_l$ be given. We have $i\in \{1,\dots,M-1\}$ such that $C = C_A(\beta) \in \B_i$  and $A  \in \A_i$. Suppose Assumptions \ref{assum: Smoothness} and \ref{assum: cond_X} hold. 
{For a sufficiently large $T$}, we have,
    \begin{align*}
        \P(\mathcal{E} \cap \G_C \cap \S_C^c) \leq \dfrac{3 m_{C,i}^*}{2T|C|_\T},
    \end{align*}
where,
\begin{align*}
\mathcal{E}&=\{\forall C \in \cup_{i=1}^{M-1} \mathcal{B}_i, m_{C,i} \in[m_{C,i}^*/2,3m_{C,i}^*/2]\},\\
\mathcal{S}_{C} &= \{\underline{\mathcal{I}}_{C} \subseteq \mathcal{I}_{C}^\prime \subseteq \overline{\mathcal{I}}_{C}\},\\
\mathcal{G}_{C} &= \cap_{C^\prime\in \mathcal{P}(C)} \mathcal{S}_{C'},
\end{align*}
and we recall the definition of $\underline{\mathcal{I}}_{C}$ and $\overline{\mathcal{I}}_{C}$ as
\begin{align*}
    \underline{\mathcal{I}}_{C} &= \left\{k \in {\I_C}: \sup_{x \in C} \{f^{(*)}(x^\top\beta_0) -f ^{(k)}(x^\top\beta_0)\} \leq c_0 |C|_\T \right\},\nonumber\\
     \overline{\mathcal{I}}_{C} &= \left\{k \in {\I_C}: \sup_{x \in C} \{f^{(*)}(x^\top\beta_0) -f ^{(k)}(x^\top\beta_0) \} \leq c_1 |C|_\T \right\}
\end{align*}
for $c_0 = 4L_0 + 1$ with ${L_0 := L(2^{3/2}R_X + 1)}$ and $c_1 = 8c_0
\gamma_X^{1/2}$.
\end{lemma}
\begin{proof}
Since $\S_C = \{\underline{\mathcal{I}}_{C} \subseteq \mathcal{I}_{C}^\prime  \subseteq \overline{\mathcal{I}}_{C}\}$, we have $\S_C^c = \{\underline{\mathcal{I}}_{C} \not\subseteq \mathcal{I}_{C}^\prime\} \cup [\{\underline{\mathcal{I}}_{C} \subseteq \mathcal{I}_{C}^\prime\} \cap \{  \mathcal{I}_{C}^\prime \not\subseteq \overline{\mathcal{I}}_{C}\}] $.
Therefore,
\begin{align*}
\P(\mathcal{E} &\cap \G_C \cap \S_C^c)\\
& = \P(\mathcal{E} \cap \G_C \cap  \{\underline{\mathcal{I}}_{C} \not\subseteq \mathcal{I}_{C}^\prime\}) + \P(\mathcal{E} \cap \G_C \cap \{\underline{\mathcal{I}}_{C} \subseteq \mathcal{I}_{C}^\prime\} \cap \{  \mathcal{I}_{C}^\prime \not\subseteq \overline{\mathcal{I}}_{C}\}).
\end{align*}
Also, suppose for now that the following inequalities
\begin{align}
    2c_0|C|_\T &\le U(m_{C,i},T,C) \le \frac{2}{3}(c_1-2L_0)|C|_\T\label{eq: U_ineq}
\end{align}
hold on $\mathcal{E}$, which we later will show. Here, we recall that $|C|_\T = |A|$ for $C = C_A(\beta)$.

For the first term, since $\underline{\I}_{C}  \not\subseteq \I'_C$, there exists an arm $k_1 \in \underline{\mathcal{I}}_C$ such that $k_1 \notin \I_C'$, i.e., $k_1$ was eliminated at the end of batch $i$ within the bin $C$. By the arm elimination mechanism, $\exists k_2 \in \mathcal{I}_{p(C)}$ such that,
\begin{align}
\bar{Y}_{C,i}^{(k_2)} - \bar{Y}_{C,i}^{(k_1)} > U(m_{C,i}, T, C). \label{eq: eq1_S1capG}
\end{align}
We argue that this implies that there exists $k \in {\{k_1,k_2\} \subseteq \I_C}$ such that $|\bar{Y}_{C,i}^{(k)} - \bar{g}_C^{(k)}| >\frac{1}{4} U(m_{C,i}, T, C)$. 
We have,
\begin{align*}
  \bar{g}_{C}^{(k_2)} - \bar{g}_{C}^{(k_1)} 
  &=   \frac{1}{\P_X(C)} \int_{x\in C} \{g^{(k_2)}(x) -g^{(k_1)}(x)\} d\P_X(x)\\
  &\le   \frac{1}{\P_X(C)} \int_{x\in C} \{g^{(*)}(x) -g^{(k_1)}(x)\} d\P_X(x),
\end{align*}
and since $k_1 \in \underline{\I}_C$, $\sup_{x \in C} \{g^{(*)}(x) - g^{(k_1)}(x)\} \le c_0 |A|$, 
and thus 
\begin{align*}
    \bar{g}_{C}^{(k_2)} - \bar{g}_{C}^{(k_1)}  \le c_0|A|.
\end{align*}
Then, if both $k\in \{k_1,k_2\}$ satisfy $|\bar{Y}_{C,i}^{(k)} - \bar{g}_C^{(k)}| \le \frac{1}{4} U(m_{C,i}, T, C)$, then 
\begin{align*}
    \bar{Y}_{C,i}^{(k_2)} - \bar{Y}_{C,i}^{(k_1)}
    &=\bar{Y}_{C,i}^{(k_2)}- \bar{g}_C^{(k_2)}+ \bar{g}_C^{(k_2)}- \bar{g}_C^{(k_1)} + \bar{g}_C^{(k_1)} - \bar{Y}_{C,i}^{(k_1)}\\
    &\le |\bar{Y}_{C,i}^{(k_2)} - \bar{g}_C^{(k_2)}|+\{\bar{g}_C^{(k_2)}-\bar{g}_C^{(k_1)}\}+|\bar{Y}_{C,i}^{(k_1)} - \bar{g}_C^{(k_1)}|\\
    &\le \frac{1}{2}U(m_{C,i},T,C) + c_0|A| \\
    &\le U(m_{C,i},T,C), 
\end{align*} which is  a contradiction, and therefore on $\mathcal{E}$, there exists $k \in {\{k_1,k_2\} \subseteq \I_C}$ such that $|\bar{Y}_{C,i}^{(k)} - \bar{g}_C^{(k)}| > \frac{1}{4} U(m_{C,i}, T, C)$.
In particular, we can bound the first term as follows:
\begin{align*}
 \P(\mathcal{E} \cap \G_C \cap & \{\underline{\mathcal{I}}_{C} \subseteq \mathcal{I}_{C}^\prime\}^c)\\
 &\le \P\left(\mathcal{E} \cap \left\{\exists k\in {[K]} \,\text{s.t.}\,{m_{C,i}^{(k)}\ge 1},\, |\bar{Y}_{C,i}^{(k)} - \bar{g}_C^{(k)}| >\frac{1}{4} U(m_{C,i}, T, C)\right\}\right)
\end{align*}
{where we recall that $m_{C,i}^{(k)}$ is the number of times arm $k$ is pulled in bin $C$ during batch $i$.}

For the second term where $\{\underline{\mathcal{I}}_{C} \subseteq \mathcal{I}_{C}^\prime\} \cap \{  \mathcal{I}_{C}^\prime \not\subseteq \overline{\mathcal{I}}_{C}\}$, there exists $k_1 \in \I'_C$ such that $k_1 \notin \overline{\I}_C$. By the definition of $\overline{\I}_C$, there exists $x_0 \in C$ such that 
\begin{equation}
    g^{(k_2)}(x_0) - g^{(k_1)}(x_0) > c_1 |A| \label{eq: x0_choice}
\end{equation}  for $k_2\ne k_1$. Then, for any $x \in C$,
\begin{align}
    g^{(k_2)}(x) - g^{(k_1)}(x) 
    &\ge g^{(k_2)}(x_0) - g^{(k_1)}(x_0) -\sum_{k\in{\{k_1,k_2\}}} |g^{(k)}(x) - g^{(k)}(x_0)| \nonumber\\
    &\ge c_1 |A| - 2L_0|A|  = (c_1-2L_0)|A| >0, \label{eq: eq2_gk2_sup}
\end{align}
where the last inequality is due to the fact that for a sufficiently large $T$, $|g^{(k)}(x) - g^{(k)}(x_0)| \le L_0|A|$ by \eqref{eq: gkx-gky_diff} in Lemma~\ref{lem: barg_min_g}, and 
\begin{align}\label{eq: c1_2L0_ineq}
    c_1 - 2L_0 \ge 8c_0 \gamma_X^{1/2} - c_0 = c_0(8\gamma_X^{1/2}-1) \ge 7c_0 \gamma_X^{1/2} >0,
\end{align}
since $c_1 = 8c_0\gamma_X^{1/2}$, $c_0 = 4L_0+1 \ge 2L_0$, and $\gamma_X \ge 1$.

Note the bound \eqref{eq: eq2_gk2_sup} implies that $k_2$ is universally better than $k_1$ on $C$. In particular, $k_2 \in \underline{\I}_C \subseteq \I_C'$ as well. Since both $k_1,k_2 \in \I_C'$, 
\begin{align*}
    |\bar{Y}_{C,i}^{(k_1)}-\bar{Y}_{C,i}^{(k_2)}| \le U(m_{C,i},T,C).
\end{align*}
We argue that on $\mathcal{E}$, when $T$ is sufficiently large, this implies that there exists $k \in {\{k_1,k_2\}}$ such that $|\bar{Y}_{C,i}^{(k)} - \bar{g}_C^{(k)}| >\frac{1}{4} U(m_{C,i}, T, C)$.  We have
\begin{align*}
    \bar{g}_{C}^{(k_2)}&\ge g^{(k_2)}(x_0) - |\bar{g}_{C}^{(k_2)}-g^{(k_2)}(x_0) | \\
&\ge   g^{(k_2)}(x_0) - L_0 |A|\\
& >  \{ g^{(k_1)}(x_0)+c_1|A|\} - L_0 |A|,
\end{align*}
where the second inequality is due to Lemma \ref{lem: barg_min_g}, and the third inequality is due to the choice of $x_0$ in \eqref{eq: x0_choice}.  Applying Lemma \ref{lem: barg_min_g} again,
\begin{align*}
    \bar{g}_{C}^{(k_2)} 
    &>  g^{(k_1)}(x_0)+c_1|A| - L_0 |A|\\
    &>  \{\bar{g}_C^{(k_1)} - |\bar{g}_C^{(k_1)} - g^{(k_1)}(x_0)|\}+c_1|A| - L_0 |A|\\
     &>  \bar{g}_C^{(k_1)} +(c_1-2L_0)|A|\\
     & > \bar{g}_C^{(k_1)} +\frac{3}{2}U(m_{C,i},T,C), 
\end{align*}
where for the last inequality we use \eqref{eq: U_ineq}.
On the other hand, 
\begin{align*}
    |  \bar{g}_{C}^{(k_2)}  -   \bar{g}_{C}^{(k_1)} | & \le  |  \bar{g}_{C}^{(k_2)}  -   \bar{Y}_{C,i}^{(k_2)} | + |  \bar{Y}_{C,i}^{(k_2)}  -   \bar{Y}_{C,i}^{(k_1)} | + |  \bar{g}_{C}^{(k_2)}  -   \bar{Y}_{C,i}^{(k_1)} | \\
   & \le  |  \bar{g}_{C}^{(k_2)}  -   \bar{Y}_{C,i}^{(k_2)} | +U(m_{C,i},T,C) + |  \bar{g}_{C}^{(k_2)}  -   \bar{Y}_{C,i}^{(k_1)} |. 
\end{align*}
Therefore if both $k\in \{k_1,k_2\}$ satisfy $|\bar{Y}_{C,i}^{(k)} - \bar{g}_C^{(k)}| \le \frac{1}{4} U(m_{C,i}, T, C)$, then $|\bar{g}_{C}^{(k_2)} - \bar{g}_{C}^{(k_1)}|\le \frac{3}{2}U(m_{C,i},T,C)$, which is a contradiction. Therefore, 
\begin{align*}
 \P(\mathcal{E} \cap \G_C \cap & \{\underline{\mathcal{I}}_{C} \subseteq \mathcal{I}_{C}^\prime\} \cap \{  \mathcal{I}_{C}^\prime \not\subseteq \overline{\mathcal{I}}_{C}\}) \\
 &\le \P(\mathcal{E} \cap \{\exists k\in {[K]} \, \text{s.t.} \,{m_{C,i}^{(k)}\ge 1},\,|\bar{Y}_{C,i}^{(k)} - \bar{g}_C^{(k)}| >\frac{1}{4} U(m_{C,i}, T, C)\}).   
\end{align*}
Combining two inequalities and by Lemma \ref{lem: Ybar_gbar_concentration}, we have
\begin{align*}
\P(\mathcal{E} \cap \G_C \cap \S_C^c) & \le 2\P(\mathcal{E} \cap \{\exists k\in {[K]} \, \text{s.t.} \,\,{m_{C,i}^{(k)}\ge 1},\, |\bar{Y}_{C,i}^{(k)} - \bar{g}_C^{(k)}| >\frac{1}{4} U(m_{C,i}, T, C)\})   \\
& \leq \frac{3m_{C,i}^*}{2T|A|}.
\end{align*}

It remains to show \eqref{eq: U_ineq} on $\mathcal{E}$. Recall 
{
\begin{equation*}
    U(m,T,C)= 4\sqrt{\frac{2\log (2KT|C|_\T}{\lfloor m/K\rfloor \vee 1}}.
\end{equation*}
}
First we show that 
\begin{align}\label{eq: lem_bound_gc1}
    &c_0|A| \le \frac{1}{2}U(\frac{3}{2}m_{C,i}^*,T,C)\quad \text{and}\quad
     \frac{3}{2}U(\frac{1}{2}m_{C,i}^*,T,C)\le (c_1 - 2L_0)|A|.
\end{align}
Recall for $1\leq i \leq M-1$, $m^*_{C,i} = (t_i - t_{i-1})\P_X(C)$.
{Suppose $T$ is sufficiently large so that $\beta_{sgn} \in \mathbb{B}_2(R_0; \beta_0)$ for $R_0>0$ defined in Assumption \ref{assum: cond_X}. Then we have
$\underline{c}_X|A|\le \P_X(C) \le \overline{c}_X|A|$ (ref. Equation \eqref{eq: P_X(C)_bounds}) under the stated assumptions. 
Moreover, we have $t_i - t_{i-1}   = \lfloor c_B  \,  K w_i^{-3} \log{(2KT w_i)}\rfloor$ and $c_B = 4/(c_0^2\overline{c}_X) = 4(4L_0+1)^{-2}(\overline{c}_X)^{-1}$  in \eqref{eq: batch_size}. 
Therefore, we have
\begin{align*}
    \frac{1}{2}U(\frac{3}{2}m_{C,i}^*,T,C)  \ge 2\sqrt{\frac{2\log (2KT|A|)}{\{ 1.5 \bar{c}_Xc_B  \,   |A|^{-2} \log{(2KT |A|)}\}\vee 1 }}.
\end{align*}
For a sufficiently large $T$,
\begin{align}\label{eq:lem:wi-2order}
     1.5 \bar{c}_Xc_B  \,   |A|^{-2} \log{(2KT |A|)} \asymp (\frac{T}{K\log T})^{\frac{2(1-\gamma^i)}{3(1-\gamma^M)}}\log (KT|A|)\ge 1,
\end{align}
therefore,
\begin{align*}
    \frac{1}{2}U(\frac{3}{2}m_{C,i}^*,T,C)  \ge 2\sqrt{\frac{2\log (2KT|A|)}{ 1.5 \bar{c}_Xc_B  \,   |A|^{-2} \log{(2KT |A|)}}} \ge 2|A|\sqrt{\frac{4c_0^2\bar{c}_X} {3\bar{c}_X}}\ge c_0|A|.
\end{align*}
On the other hand,
\begin{align*}
    \frac{3}{2}U(\frac{1}{2}m_{C,i}^*,T,C) 
    &= 6\sqrt{\frac{2\log(2KT|A|)}{\lfloor (1/2)  \lfloor c_B K|A|^{-3}\log(2TK|A|)\rfloor\P_X(C)/K\rfloor \vee 1}}. 
\end{align*}
To upper-bound RHS,
\begin{align*}
    \lfloor c_B K|A|^{-3}\log(2TK|A|)\rfloor &\ge   c_BK |A|^{-3}\log(2TK|A|) - 0.5 \\
    &\ge (1-\delta) c_B K|A|^{-3}\log(2TK|A|)
\end{align*}
for sufficiently large $T$, for any given $\delta>0$, since $K|A|^{-3}\log(2TK|A|)$ grows with $T$ by \eqref{eq:lem:wi-2order}.
In particular, taking $\delta = 1/4$ and using $\P_X(C) \ge \underline{c}_X|A|$,
\begin{align*}
    \lfloor 0.5\lfloor c_B K|A|^{-3}\log(2TK|A|)\rfloor\P_X(C)/K\rfloor \vee 1 
    &\ge \lfloor (3/8)\underline{c}_X c_B|A|^{-2} \log (2TK|A|)\rfloor \vee 1\\
    &\ge (3/8)\underline{c}_X c_B|A|^{-2} \log (2TK|A|)
\end{align*}
for a sufficiently large $T$. Therefore, 
\begin{align*}
    \frac{3}{2}U(\frac{1}{2}m_{C,i}^*,T,C) 
    \le  6\sqrt{\frac{2\log(2KT|A|)}{(3/8)\underline{c}_X c_B|A|^{-2} \log (2TK|A|)}} \le \frac{12}{\sqrt{3}}|A|\sqrt{\frac{\overline{c}_Xc_0^2}{\underline{c}_X}}\le 7c_0 \sqrt{\gamma}_X|A| \le (c_1-2L_0)|A|,
\end{align*}
where for the last inequality we use \eqref{eq: c1_2L0_ineq}.
}

Finally, on $\mathcal{E}$, we have that $\frac{1}{2}m_{C,i}^* \le m_{C,i} \leq \frac{3}{2}m_{C,i}^*$, therefore, 
\begin{align}\label{eq: lem_bound_gc2}
    U(1.5 m^*_{C,i},T,C) \le U(m_{C,i},T,C)  \le U(0.5 m^*_{C,i},T,C).
\end{align}
By combining \eqref{eq: lem_bound_gc1} and \eqref{eq: lem_bound_gc2}, we obtain \eqref{eq: U_ineq}.
  \end{proof}

\begin{lemma} \label{lem: Y_independence}
Let $i\in \{1,\dots,M\}$ be given, and fix $C\in \B_i$. Let $\tau_{C,i}(s)$ be the sth time at which the sequence $X_t$ is in $C$ during $[t_{i},t_{i+1})$. Fix $k \in {[K]}$. Assume $|Y_t^{(k)}|\le 1$ almost surely for any $t,k$. Consider $\{Y^{(k)}_{\tau_{C,i}(s)}; s=1,\dots,N\}$ for some $N<\infty$. Then $\{Y^{(k)}_{\tau_{C,i}(s)}; s=1,\dots,N\}$ are independent random variables with expectation $\bar{g}_{C}^{(k)}$, where
\begin{align*}
\bar{g}_{C}^{(k)}:=\frac{1}{\P(X\in C)} \int_{x \in C} g^{(k)}(x) d \P_X(x)=\frac{1}{\P(X\in C)} \int_{x \in C} f^{(k)}(x^\top\beta_0) d \P_X(x).
\end{align*}
\end{lemma}
\begin{proof}

Recall that $\tau_{C,i}(s) = \inf\{n\ge \tau_{C,i}(s-1)+1; X_n \in C\}$ represents the time of the $s$th visit to the set $C$ from $t_{i-1}$, for $s=1,2,\dots$ and $\tau_{C,i}(0) = t_{i-1}$. Without loss of generality, assume $i=1$; otherwise we can redefine the sequence $X_{t_{i-1}+1},X_{t_{i-1}+2},\dots$ as $X_1,X_2,\dots$. Also, let $\tau_C(s) = \tau_{C,i}(s)$ for notational simplicity.

We note that for any $s$, $\tau_C(s)$ is a stopping time with respect to filtration $\mathcal{F}^X_t = \sigma(X_1,\dots,X_t)$, as for any $t \in \mathbb{N}$, $\{\tau_C(s) > t\} = \{\sum_{n=1}^t 1\{X_n \in C\} < s\}$ and therefore $\{\tau_C(s) > t\}$ is $\mathcal{F}_t^X$-measurable.

First, we compute $\E[Y^{(k)}_{\tau_C(s)}]$. First note that $1 = \sum_{t=s}^\infty 1\{\tau_C(s) = t\}$ almost surely and 
\begin{align*}
    &\{\tau_C(s) = t\} \\
    &= \bigcup_{\substack{(i_1,\dots,i_{s-1}) \subseteq\{1,\dots,t-1\}\\(j_1,\dots,j_{t-s}) \subseteq\{1,\dots,t-1\}\setminus (i_1,\dots,i_{s-1})}} \{X_{i_1} \in C,\dots,X_{i_{s-1}} \in C, X_{j_1}\in C^c,\dots,  X_{j_{n-s}} \in C^c\} \bigcap \{X_t \in C\}
\end{align*}
as $\{\tau_C(s) = t\}$ is the event where $X_n$ visits $C$ for $s-1$ times during $n=1,\dots,t-1$ and $X_t \in C$. For future reference, we define for $a< b$, and $s\in \{0, \dots,b-a\}$,
\begin{align*}
&\mathcal{E}_C(a,b,s)\\
&= \bigcup_{\substack{(i_1,\dots,i_{s}) \subseteq \{a+1,\dots,b\}\\(j_1,\dots,j_{b-a-s}) \subseteq \{a+1,\dots,b\}\setminus (i_1,\dots,i_{s})}} \{X_{i_1} \in C,\dots,X_{i_{s-1}} \in C, X_{j_1}\in C^c,\dots,X_{j_{b-a-s}} \in C^c\} 
\end{align*}
to be the event that during $n=a+1,\dots,b$, $X_n \in C$ for $s$ times.  With this notation,
\begin{align}
    \{\tau_C(s) = t\} = \mathcal{E}_C(0,t-1,s-1) \cap \{X_t \in C\}.
\end{align}
Since $(X_t)_{t\ge 1}$ are independent and identically distributed, we have,
\begin{align*}
&\P(\mathcal{E}_C(a,b,s)) =\binom{b-a}{s} \P(X_1 \in C^c)^{(b-a)-s} \P(X_1 \in C)^{s} 
\end{align*}
Therefore, we have, 
\begin{align*}
    \E[Y^{(k)}_{\tau_C(s)}]&= \E[\sum_{t=s}^\infty Y^{(k)}_{t} 1\{\tau_C(s) = t\}]\\
    &=\sum_{t=s}^\infty \E[ Y^{(k)}_{t} 1\{\tau_C(s) = t\}]\\
    &=\sum_{t=s}^\infty \E[ Y^{(k)}_{t} 1_{\mathcal{E}_C(0,t-1,s-1)}1\{X_t \in C\}]\\
    &=\sum_{t=s}^\infty \binom{t-1}{s-1} \P(X_1 \in C^c)^{t-s} \P(X_1 \in C)^{s-1} \E[ Y^{(k)}_{t} 1\{X_t \in C\} ]
\end{align*}
where for the second line we use the Fubini's theorem and the fact that $|Y_t^{(k)}|$ is bounded almost surely, and for the third line we use the independence between $(X_1,\dots,X_{t-1})$ and $(X_t, Y_t)$.  Since $\sum_{t=s}^\infty \binom{t-1}{s-1} \P(X_1 \in C^c)^{t-s} \P(X_1 \in C)^{s-1}  = \P(X_1 \in C)^{-1}$, we have
\begin{align*}
    \E[Y^{(k)}_{\tau_C(s)}]&= \frac{\E[ Y^{(k)}_{1} 1\{X_1 \in C\} ]}{\P(X_1\in C)} = \frac{1}{\P_X(C)} \int_{x \in C} g^{(k)}(x) d\P_X(x) =  \bar{g}_{C}^{(k)}
\end{align*}
where we note that $\E[ Y^{(k)}_{1} 1\{X_1 \in C\} ] = \E_{X_1}[\E_{\epsilon|X_1} [Y^{(k)}_{1}|X_1] 1\{X_1 \in C\} ]=\E_{X_1}[g^{(k)}(X_1) 1\{X_1 \in C\} ]$.

Now we show the independence of $\{Y^{(k)}_{\tau_C(s)}; s=1,\dots,N\}$. Fix $m\le N$.
Let $(i_1,\dots,i_m) \subseteq \{1,\dots,N\}$ be given such that $i_1<i_2<\dots<i_m$, as well as $B_1,\dots,B_m \in \mathscr{B}_{\mathbb{R}}$. It is sufficient to show $\P(Y^{(k)}_{\tau_C(i_1)} \in B_1, \dots,Y^{(k)}_{\tau_C(i_m)}\in B_m) = \prod_{j=1}^m \P(Y^{(k)}_{\tau_C(i_j)} \in B_j) $. 
\begin{align*}
    &\P(Y^{(k)}_{\tau_C(i_1)} \in B_1, \dots,Y^{(k)}_{\tau_C(i_m)}\in B_m)\\
    &=\sum_{n_1,n_2,\dots,n_m} \P(Y^{(k)}_{n_1} \in B_1, \dots,Y^{(k)}_{n_m}\in B_m, \tau_C(i_1)=n_1,\dots,\tau_C(i_m)=n_m)
    \end{align*}
Recall $\{\tau_C(i_1)=n_1,\dots,\tau_C(i_m)=n_m\}$ is the event that 
the time point for the $i_1$th visit = $n_1$, time point for the $i_2$th visit = $n_2$,$\dots$, and the time point for the $i_m$th visit = $n_m$.
Note that there are some restrictions in the possible values of $(n_1,\dots,n_m)$. For example, the earliest time $X_t$ can visit $C$ for $i_1$ times is $i_1$, when $X_t\in C$ for $1\le t \le i_1$, so $n_1 \ge i_1$. When $\tau_C(i_1) = n_1$, the earliest time that $X_t$ can visit $C$ for $i_2$ times is $n_1 + (i_2-i_1)$, so $n_2$ has to be at least $n_1+(i_2-i_1)$.  With this consideration, we have,
\begin{align*}
    &\P(Y^{(k)}_{\tau_C(i_1)} \in B_1, \dots,Y^{(k)}_{\tau_C(i_m)}\in B_m)\\
    &=\sum_{n_1,n_2,\dots,n_m} \P(Y^{(k)}_{n_1} \in B_1, \dots,Y^{(k)}_{n_m}\in B_m, \tau_C(i_1)=n_1,\dots,\tau_C(i_m)=n_m)\\
     &=\sum_{n_1=i_1}^\infty\sum_{n_2 = n_1+(i_2-i_1)}^\infty \cdots \sum_{n_m = n_{m-1}+(i_m-i_{m-1})}^\infty\\ 
     & \qquad \qquad \E(1\{\mathcal{E}_C(0,n_1-1,i_1-1) \cap \{X_{n_1} \in C,Y^{(k)}_{n_1} \in B_1\}  \cdots\\
     &\qquad \qquad\qquad \cap \mathcal{E}_C(n_{m-1},n_m-1,i_m-i_{m-1}-1)\cap \{X_{n_m} \in C, Y^{(k)}_{n_m} \in B_m\}\})\\
      &=\sum_{n_1=i_1}^\infty\sum_{n_2 = n_1+(i_2-i_1)}^\infty \cdots \sum_{n_m = n_{m-1}+(i_m-i_{m-1})}^\infty \prod_{j=1}^{m}\\
      & \qquad \qquad \P(\mathcal{E}_C(n_{j-1},\,n_j-1,\,i_j-i_{j-1} -1))\P(X_{n_j} \in C,Y^{(k)}_{n_j} \in B_j) 
\end{align*}
where we define $n_0 = 0, i_0=0$ , and use independence for the last equation. Since
\begin{align*}
   \P(\mathcal{E}_C(n_{j-1},\,n_j-1,i_j-i_{j-1} -1)) = \binom{n_{j}-n_{j-1}-1}{i_j-i_{j-1} -1} (1-p)^{(n_{j}-n_{j-1}) -(i_j-i_{j-1}) } p^{i_j-i_{j-1}-1},
\end{align*}
for $p = \P(X\in C)$, we have,
\begin{align}
   & \sum_{n_1=1}^\infty\sum_{n_2 = n_1+(i_2-i_1)}^\infty \cdots \sum_{n_m = n_{m-1}+(i_m-i_{m-1})}^\infty \prod_{j=1}^{m} \P(\mathcal{E}_C(n_{j-1},n_j-1,i_j-i_{j-1} -1))\P(X_{n_j} \in C,Y^{(k)}_{n_j} \in B_j) \nonumber\\
   &= \prod_{j=1}^{m} \lbrace \sum_{n_j=n_{j-1}+(i_j - i_{j-1})}^\infty \P(\mathcal{E}_C(n_{j-1},n_j-1,i_j-i_{j-1} -1))\P(X_{1} \in C,Y^{(k)}_{1} \in B_j)  \rbrace \nonumber \\
   &=\prod_{j=1}^{m}  \frac{\P(X_{1} \in C,Y^{(k)}_{1} \in B_j)}{\P(X_1 \in C)} \label{lem:eq_indep_1}
\end{align}
where for the last equality we use the fact that for any $j \in \{1,\dots,m\}$,
\begin{align}
&\sum_{n_j=n_{j-1}+(i_j - i_{j-1})}^\infty  \P(\mathcal{E}_C(n_{j-1},n_j-1,i_j-i_{j-1} -1))\nonumber\\
    &=\sum_{n_j=n_{j-1}+(i_j - i_{j-1})}^\infty \binom{n_{j}-n_{j-1}-1}{i_j-i_{j-1} -1} (1-p)^{(n_{j}-n_{j-1}) -(i_j-i_{j-1}) } p^{(i_j-i_{j-1})-1}\nonumber\\
   &=\sum_{k=i_j - i_{j-1}}^\infty \binom{k-1}{(i_j-i_{j-1}) -1} (1-p)^{k -(i_j-i_{j-1}) } p^{(i_j-i_{j-1})-1} = \frac{1}{p}.\label{lem:eq_indep_nb}
\end{align}
Here for the last equality, we use the following identity $\sum_{k=r}^\infty \binom{k-1}{r-1} p^k (1-p)^{n-r}=1$ with $r = i_j - i_{j-1}$.

On the other hand, for any $j\in \{1,\dots,m\}$, 
\begin{align}
    \P(Y^{(k)}_{\tau_C(i_j)} \in B_j) &= \sum_{n=i_j}^\infty \E[1\{Y^{(k)}_{n} \in B_j, \tau_C(i_j) = n\}] \nonumber\\
    &=\sum_{n=i_j}^\infty \E[1\{Y^{(k)}_{n} \in B_j, X_n \in C\}1\{\mathcal{E}_C(0,n-1, i_j-1)\}] \nonumber \\
    &=\P(Y^{(k)}_{1} \in B_j, X_1 \in C)\sum_{n=i_j}^\infty  \P(\mathcal{E}_C(0,n-1, i_j-1)) \nonumber\\
    &=\frac{\P(Y^{(k)}_{1} \in B_j, X_1 \in C)}{\P(X_1 \in C)}\label{lem:eq_indep_2}
\end{align}
where we use \eqref{lem:eq_indep_nb} with $j=1$ for the last equality. 

Therefore $\P(Y^{(k)}_{\tau_C(i_1)} \in B_1, \dots,Y^{(k)}_{\tau_C(i_m)}\in B_m) = \prod_{j=1}^m \P(Y^{(k)}_{\tau_C(i_j)} \in B_j) $ by \eqref{lem:eq_indep_1} and \eqref{lem:eq_indep_2} and the proof is complete.
\end{proof}

\begin{lemma} \label{lem: Ybar_gbar_concentration}
Fix $i \in \{1,\dots,M-1\}$ and $C \in \mathcal{B}_i$. 
Assume $|Y_t^{(k)}|\le 1$ almost surely for any $t,k$. 
Define
{
\begin{equation*}
    U(m,T,C)= 4\sqrt{\frac{2\log (2KT|C|_\T)}{\lfloor m/K\rfloor \vee 1}}.
\end{equation*}
}
We have 
\begin{align*}
    \P\left(\mathcal{E} \cap \left\lbrace \exists k \in {[K]}; \,\,{m_{C,i}^{(k)}\ge 1},\, |\bar{Y}^{(k)}_{C,i} -  \bar{g}^{(k)}_{C_A}| \geq \dfrac{1}{4} U(m_{C,i}, T, C)\right\rbrace \right) \leq \frac{3 m_{C,i}^*}{2T|C|_\T}.
\end{align*}
where $\mathcal{E}=\{\forall C \in \cup_{i=1}^{M-1} \mathcal{B}_i, m_{C,i} \in[m_{C,i}^*/2,3m_{C,i}^*/2]\}$ and
for $\bar{Y}_{C,i}^{(k)}$ defined in \eqref{eq: reward_def}.
\end{lemma}
\begin{proof}
We have 
\begin{align*}
    &\P\left(\mathcal{E} \cap \left\lbrace \exists k \in {[K]} \,\text{s.t.} \,\,{m_{C,i}^{(k)}\ge 1},\, |\bar{Y}^{(k)}_{C,i} -  \bar{g}^{(k)}_{C_A}| \geq \dfrac{1}{4} U(m_{C,i}, T, C)\right\rbrace \right) \\
    &\leq \P\left(\frac{1}{2}m_{C,i}^*\le m_{C,i} \leq \frac{3}{2} m_{C,i}^*, \,\exists k \in {[K]} \,\text{s.t.}\,\,{m_{C,i}^{(k)}\ge 1},\, |\bar{Y}^{(k)}_{C,i} -  \bar{g}^{(k)}_{C_A}| \geq \dfrac{1}{4} U(m_{C_A,i}, T, C)\right) \\
    &\leq \sum_{k=1}^{{K}} \P\left(\frac{1}{2}m^{*}_{C,i} \le m_{C,i} \leq \frac{3}{2} m_{C,i}^* , \,{m_{C,i}^{(k)}\ge 1},\,|\bar{Y}^{(k)}_{C,i} -  \bar{g}^{(k)}_{C_A}| \geq \dfrac{1}{4} U(m_{C,i}, T, C)\right) \\
    &\leq \sum_{k=1}^{{K}} \sum_{n=1}^{\lfloor 1.5m_{C,i}^* \rfloor} \P\left(m_{c,i}=n, \,{m_{C,i}^{(k)}\ge 1},\,|\bar{Y}^{(k)}_{C,i} -  \bar{g}^{(k)}_{C_A}| \geq \dfrac{1}{4} U(m_{C,i}, T, C)\right).
\end{align*}
 
For any $n>0$,  $\{Y_{\tau_{C,i}(s)}^{(k)}; 1\le s \le n\}$ consists of bounded independent random variables with mean $\bar{g}_C^{(k)}$ by Lemma \ref{lem: Y_independence}. 

{Note that on $\{m_{C,i} = n,\, m_{C,i}^{(k)} \ge 1\}$, conditioned on the history $\mathcal{H}_{t_{i-1}},$ we have $ \bar{Y}_{C,i}^{(k)}$ is the average of $m_{C,i}^{(k)} \ge \lfloor n/K\rfloor\vee 1$ bounded independent terms with mean $\bar{g}_C^{(k)}$. Therefore, by Hoeffding's inequality, conditioned on $\mathcal{H}_{t_{i-1}}$:
\begin{align*}
    &\P\left(m_{C,i}=n,\,m_{C,i}^{(k)}\ge 1,\, |\bar{Y}^{(k)}_{C,i} -  \bar{g}^{(k)}_{C}| 
    \geq \sqrt{\frac{2\log(2KT|C|_\T)}{1 \vee \lfloor n/K \rfloor }}\,\mid \mathcal{H}_{t_{i-1}} \right) \\
    &\le  \P\left(m_{C,i}=n,\,\,m_{C,i}^{(k)}\ge 1,\, |\bar{Y}^{(k)}_{C,i} -  \bar{g}^{(k)}_{C}| 
    \geq \sqrt{\frac{2\log(2KT|C|_\T)}{m_{C,i}^{(k)}}},\mid \mathcal{H}_{t_{i-1}}  \right) \\
    &\le 2\exp\left(-\frac{m_{C,i}^{(k)}}{2}\cdot \frac{2\log(2KT|C|_\T)}{m_{C,i}^{(k)}}\right) = \frac{1}{KT|C|_\T}.
\end{align*}
}
Then, by using the union bound,
\begin{align*}
    \P&\left(\mathcal{E} \cap \left\lbrace \exists k \in {[K]} \,\text{s.t.} \,\,{m_{C,i}^{(k)}\ge 1},\, |\bar{Y}^{(k)}_{C,i} -  \bar{g}^{(k)}_{C}| \geq \dfrac{1}{4} U(m_{C,i}, T, C)\right\rbrace \right)\\
    &\le {K}\cdot \lfloor 1.5 m_{C,i}^*\rfloor \frac{1}{{K}T|C|_\T}  \le \frac{3m^*_{C,i}}{{2}T|C|_\T}.
\end{align*}
\end{proof}
\section{Example of single index vector estimation using SADE}\label{sec: SADEappendix}

In this section, we present an example of constructing the initial vector $\hat{\beta}$ which satisfies Assumption \ref{assum: index_rate}.
We propose using the Sliced Average Derivative Estimator (SADE) introduced by \citesupp{babichev2018slic_supp}, which combines the Average Derivative Estimator \citesupp{newey1993efficiency} and Sliced Inverse Regression \citesupp{li1991sliced}. This approach offers provable improvements over non-sliced versions and provides non-asymptotic bounds for estimating a matrix whose column space lies within the effective dimension reduction (e.d.r) space. Using this bound and the Davis-Kahan inequality, we will derive a non-asymptotic bound for the initial vector that satisfies Assumption \ref{assum: index_rate}.

\paragraph{SADE algorithm}

We briefly describe the SADE algorithm and the non-asymptotic bound for the matrix whose column space belongs to the e.d.r of the model  by \citesupp{babichev2018slic_supp}. Consider for now a dataset with iid observations $(X_i, Y_i)_{i=1}^n$. \citesupp{babichev2018slic_supp} makes the following assumptions on the model and the distribution of $X$:
 \begin{enumerate}
     \item (A1) For all $x \in \mathbb{R}^d$, we have $f(x)=g\left(w^{\top} x\right)$ for a certain matrix $w \in \mathbb{R}^{d \times k}$ and a function $g: \mathbb{R}^k \rightarrow \mathbb{R}$. Moreover, $Y=f(X)+\varepsilon$ with $\varepsilon$ independent of $X$ with zero mean and finite variance.
     \item (A2) The distribution of $X$ has a strictly positive density $p(x)$ which is differentiable with respect to the Lebesgue measure, and such that $p(x) \to 0$ when $\|x\| \to \infty$. 
 \end{enumerate}
Note that when $k=1$ in (A1), the model corresponds to the single-index model.

Let $\mathcal{S}_1(x)$ be the negative derivative of the log density of $\P_X$, i.e., $\mathcal{S}_1(x)=-\nabla \log p(x)=\frac{-1}{p(x)} \nabla p(x)$ where $p(x)$ is the density function of $\P_X$ with respect to Lebesgue measure, which is assumed to be known. For example, if $X$ is normally distributed with mean vector $\mu$ and covariance matrix $\Sigma$, then $\mathcal{S}_1(x) = \Sigma^{-1}(x - \mu)$.

From Lemma 2 in \citesupp{babichev2018slic_supp}, under (A1)--(A2),  $\mathbb{E}(\mathcal{S}_1(X)|Y=y)$ belongs to the e.d.r space $\textrm{span}(w_1,\dots,w_k)$ for almost every (a.e.) $y$.
Then $\mathcal{V}_{1, \mathrm{cov}}=\mathbb{E} [\mathbb{E} (\mathcal{S}_1(X) | Y ) \mathbb{E}(\mathcal{S}_1(X) | Y  )^\top  ]  ={\rm Cov} [\mathbb{E} (\mathcal{S}_1(x) | Y  )]$ will be at most a rank-$k$ matrix whose eigenvectors corresponding to non-zero eigenvalues belong to $\textrm{span}(w_1,\dots,w_k)$. 
The process to estimate $\mathcal{V}_{1, \mathrm{cov}}$ given a data $(x_i,y_i)_{i=1}^n$ is summarized in Algorithm \ref{algorithm: SADE}.

\begin{algorithm}[ht]
\caption{SADE Algorithm to estimate $\beta_0$ for i.i.d. dataset}
\label{algorithm: SADE}
\begin{algorithmic}[1]
\State \textbf{Input:} Data $(x_i, y_i)_{i=1}^n$, score function $\mathcal{S}_1$, number of slices $H$
\State \textbf{Output:} $\beta =$ the scaled eigenvector corresponding to the largest eigenvalue of $\hat{\mathcal{V}}_{1, \mathrm{cov}}$
\State Slice $[0,1]$ into $H$ slices $I_1, \dots, I_H$
\State Let $\hat{p}_h$ be the empirical proportion of $y_i$ that fall in the slice $I_h$:
\[
\hat{p}_h = \frac{\sum_{i=1}^n 1\{y_i \in I_h\}}{n}
\]
\State Estimate $(\mathcal{S}_1)_h = \mathbb{E}[\mathcal{S}_1(x) \mid y \in I_h]$ by:
\[
(\hat{\mathcal{S}}_1)_h = \frac{1}{\sum_{i=1}^n 1\{y_i \in I_h\}} \sum_{i=1}^n 1\{y_i \in I_h\} \mathcal{S}_1(x_i)
\]
\State Estimate $\mathrm{Cov}(\mathcal{S}_1(x) \mid y \in I_h)$ by:
\[
(\hat{\mathcal{S}}_1)_{\mathrm{cov}, h} = \frac{1}{n \hat{p}_h - 1} \sum_{i=1}^n 1\{y_i \in I_h\} (\mathcal{S}_1(x_i) - (\hat{\mathcal{S}}_1)_h) (\mathcal{S}_1(x_i) - (\hat{\mathcal{S}}_1)_h)^\top
\]
\State Compute:
\[
\hat{\mathcal{V}}_{1, \mathrm{cov}} = \frac{1}{n} \sum_{i=1}^n \mathcal{S}_1(x_i) \mathcal{S}_1(x_i)^\top - \sum_{h=1}^H \hat{p}_h \cdot (\hat{\mathcal{S}}_1)_{\mathrm{cov}, h}
\]
\State Let $u$ be the eigenvector corresponding to the largest eigenvalue of $\hat{\mathcal{V}}_{1, \mathrm{cov}}$. 
\State If $u_1<0$, let $u \leftarrow  -u$.
\State \textbf{Return:} $\beta = u / \|u\|_2$
\end{algorithmic}
\end{algorithm}

\citesupp{babichev2018slic_supp} derive a non-asymptotic bound on $\|\mathcal{V}_{1, \mathrm{cov}}-\hat{\mathcal{V}}_{1, \mathrm{cov}}\|_{*}$, where $\|\cdot\|_*$ denotes the nuclear norm, under the additional assumptions (L1)--(L4) listed below.
\begin{enumerate}
    \item[(L1)] The function $m: \mathbb{R} \rightarrow \mathbb{R}^d$ such that $\mathbb{E}(\mathcal{S}_1(X) \mid Y=y)=m(y)$ is $L$-Lipschitz continuous.
    \item[(L2)] The random variable $Y \in \mathbb{R}$ is sub-Gaussian, i.e., such that \\
    $\mathbb{E} e^{t(Y-E y)} \le$ $e^{\tau_y^2 t^2 / 2}$, for some $\tau_y>0$.
\item[(L3)] The random variables $\mathcal{S}_{1j}(X) \in \mathbb{R}$ are sub-Gaussian, i.e., such that \\
$\mathbb{E} e^{t \mathcal{S}_{1j}(X)} \le$ $e^{\tau_{\ell}^2 t^2 / 2}$ for each component $j \in\{1, \ldots, d\}$, for some $\tau_{\ell}>0$.
\item[(L4)] The random variables $\eta_j=\mathcal{S}_{1j}(X)-m_j(Y) \in \mathbb{R}$ are sub-Gaussian, i.e., such that $\mathbb{E} e^{t \eta_j} \le e^{\tau_\eta^2 t^2 / 2}$ for each component $j \in\{1, \ldots, d\}$, for some $\tau_\eta>0$.
\end{enumerate}

Under (A1)--(A2) and (L1)--(L4), \citesupp{babichev2018slic_supp} proves the following bound in Theorem 1:   for any $\delta<\frac{1}{n}$, with probability not less than $1-\delta$:

\begin{align}\label{eq: sade_bound}
	\left\|\hat{\mathcal{V}}_{1, \operatorname{cov}}-\mathcal{V}_{1, \operatorname{cov}}\right\|_* & \le \frac{d \sqrt{d}\left(195 \tau_\eta^2+2 \tau_{\ell}^2\right)}{\sqrt{n}} \sqrt{\log \frac{24 d^2}{\delta}} \nonumber \\
+ & \frac{8 L^2 \tau_y^2+16 \tau_\eta \tau_y L \sqrt{d}+\left(157 \tau_\eta^2+2 \tau_{\ell}^2\right) d \sqrt{d}}{n} \log ^2 \frac{32 d^2 n}{\delta}.
\end{align}

\paragraph{Non-asymptotic bound for the estimated initial vector}

Now, combining the non-asymptotic bound for $\mathcal{V}_{1,\operatorname{cov}}$ and Davis-Kahan Theorem, we present the non-asymptotic bound for $\hat{\beta}^{(k)}$ where $\hat{\beta}^{(k)}$ is the estimated index vector using an i.i.d sample $(X_t, Y_t^{(k)})$ of size $n_k$ from the single index model \eqref{def:sim}.

\begin{theorem} \label{thm: SADE_estimation} 
Assume the single index model \eqref{def:sim} and Assumption \ref{assum: cond_X}, along with (L1)--(L4). For sufficiently large $n_k$, {for any $\delta \in (0,1/n_k)$, the following bound holds with probability at least $1-\delta$:
\begin{align*}
  \sin \angle \hat{\beta}^{(k)}, \beta_0 \le   c(d, \tau_\eta, \tau_\ell, \lambda_1) \frac{\log^2(n_k/\delta)}{\sqrt{n_k}}.
\end{align*}
}
Here $c(d, \tau_\eta, \tau_\ell, \lambda_1)$ is a constant which depends on model parameters \\
$d, \tau_\eta, \tau_\ell, \lambda_1$ but not on the sample size $n$.
\end{theorem}

\begin{proof}
Let $\hat{\mathcal{V}}_{1, \operatorname{cov}}^{(k)}$ be the estimated covariance matrix from Algorithm \ref{algorithm: SADE} using the dataset $\mathcal{D}^{(k)}_{\rm init}$ for $k=1,\dots,K$.  For $A \in \R^{d \times d}$ with singular values $\sigma_1,\dots,\sigma_d$, we have $\|A\|_*= \sum_{i=1}^d \sigma_i \le (\sum_{i=1}^d \sigma_i^2)^{1/2} (\sum_{i=1}^d 1)^{1/2}  =  d^{1/2}\|A\|_F$. Then from \eqref{eq: sade_bound}, for any $\delta < 1/n_k$, we have with probability at least $1-\delta$:
\begin{align*}
\left\|\hat{\mathcal{V}}_{1, \operatorname{cov}}^{(k)}-\mathcal{V}_{1, \operatorname{cov}}\right\|_F & \le \frac{{c_1}}{\sqrt{n_k}} \sqrt{\log \frac{{c_2}}{\delta}} +\frac{{c_3}}{n_k} \log ^2 \frac{{c_4} n_k}{\delta}
\end{align*}
{for constants $c_1,\dots,c_4>0$ depending on model parameters and $d$.}

Now, by applying a variant of Davis-Kahan inequality (ref. Theorem 2 in \citesupp{yu2015useful}) to this bound,
\begin{align*}
    \sin \angle \hat{\beta}^{(k)}, \beta_0 \le  \frac{2\|\hat{\mathcal{V}}_{1, \operatorname{cov}}^{(k)}-\mathcal{V}_{1, \operatorname{cov}}\|_{F}}{\lambda_1 - \lambda_2},
\end{align*}
where $\beta_0$ and $\hat{\beta}^{(k)}$ correspond to the first eigenvector of $\mathcal{V}_{1, \operatorname{cov}}$ and $\hat{\mathcal{V}}_{1, \operatorname{cov}}^{(k)}$ and $\lambda_1\ge \lambda_2 \ge \dots\lambda_d$ are eigenvalues of $\mathcal{V}_{1,\operatorname{cov}}$.  
Note since $k=1$, $\mathcal{V}_{1, \operatorname{cov}}$ should have only one non-zero eigenvalue, i.e., $\lambda_2 = 0$. Under condition where SADE is consistent, $\lambda_1>0$, {
and for a sufficiently large $n_k$, for any $\delta <1/n_k$, the following holds with probability at least $1-\delta$:
\begin{equation*}
    \sin \angle \hat{\beta}^{(k)}, \beta_0   \le \frac{2}{\lambda_1}\left\lbrace \frac{{c_1}}{\sqrt{n_k}} \sqrt{\log \frac{{c_2}}{\delta}} +\frac{{c_3}}{n_k} \log ^2 \frac{{c_4} n_k}{\delta} \right\rbrace\\
    \le c(\lambda_1,c_1,c_2,c_3,c_4) \frac{\log^2(n_k/\delta)}{\sqrt{n_k}}.
\end{equation*}
}
as the first term is the leading order term.

\end{proof}

\section{Addition simulation and real-data results} \label{appendix: additional_numerical}
\subsection{Additional simulation results} \label{sec: additional_simulations}
In addition to the simulation study in Section \ref{sec: 5_simulation}, we explore alternative covariate distributions beyond the truncated multivariate normal distribution. Specifically,  for $X_t \in \mathbb{R}^d$ for $t = 1, \hdots, T$, we consider: 1) $X_t \sim N(0, \Sigma_X)$, where $\Sigma_X = 5I$, where $I$ is the identity matrix, 2) $X_{ti} \sim \text{Unif}(-L,L)$ for $i = 1,\hdots,d$ and with $L = 3$. We consider Setting 2 from Section \ref{sec: 5_simulation} with $T = 10^6$. The true index vector $\beta_0$ and rewards are generated exactly as in Section \ref{sec: 5_simulation}. As before, we consider both the cases: 
\begin{itemize}
    \item When the pilot direction $\beta_0$ is available under varying degree of angular permutations $\theta$, i.e., we perturb $\beta_0$ by an angle $\theta$ ranging from $\{0.01, \hdots, \pi/2\}$ use the resulting perturbed direction in Algorithm \ref{algorithm: SIRBatchedBinning}.
    \item When the pilot direction is unknown and we use the initial $t_0 = T^{2/3}$ data to estimate using SADE algorithm \citesupp{babichev2018slic_supp} described in Algorithm \ref{algorithm: SADE} for each arm and then using Algorithm \ref{algorithm: Initial_Dir_Estimation} to construct the average index estimator. We consider varying level of model noise $\sigma$ and compare the performance of the proposed Algorithm \ref{algorithm: SIRBatchedBinning} with the nonparametric analogue, i.e., the BaSEDB algorithm of \citesupp{jiang2025batched_supp}. 
\end{itemize}
\begin{figure}[htbp]
    \centering
    \begin{tabular}{cc}
    (a) & (b)\\
         \includegraphics[width = 0.45\linewidth]{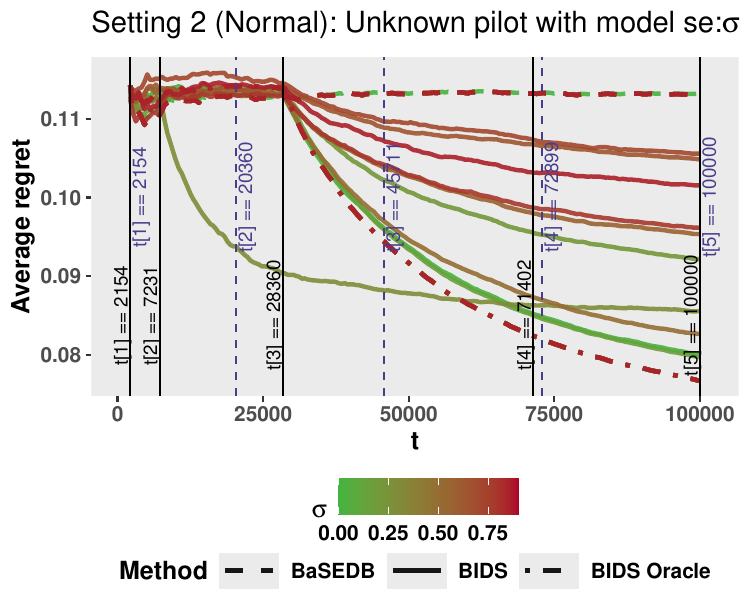}    & \includegraphics[width=0.45\linewidth]{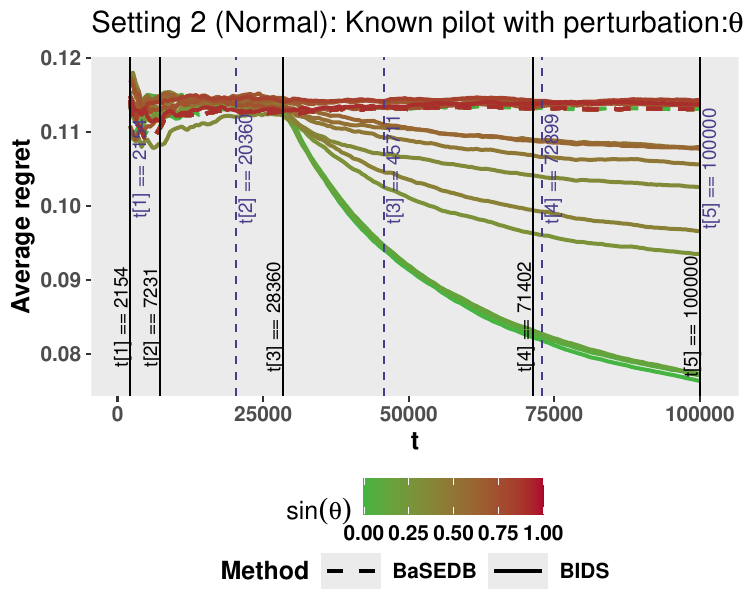} 
    \end{tabular}
\caption{Average regret ($(\mathcal{R}_t)_{t=1}^T$) with normally distributed covariates.  As the noise gets larger, the performance of the SIR batched bandit (solid) still beats the nonparametric analogue (dashed) but gets further way from the oracle (dashed-dotted).}
    \label{fig:simulation_normal}
\end{figure}
When the pilot direction is unknown and Algorithm \ref{algorithm: SADE} is employed with the initial index estimator as described in Algorithm \ref{algorithm: Initial_Dir_Estimation}, we note that for both Normal [Figure \ref{fig:simulation_normal}(a)] and Uniform covariates [Figure \ref{fig:simulation_uniform}(a)], the average regret for the proposed Algorithm \ref{algorithm: SIRBatchedBinning} decreases faster than for the nonparametric analogue (dashed lines). Nonetheless, its performance degrades as the model error grows from 0.1 to 0.8 (solid green to red lines), with the decline being more pronounced for Normally distributed covariates compared to Uniform ones.
\begin{figure}[htbp]
    \centering
    \begin{tabular}{cc}
         (a) & (b) \\
        \includegraphics[width = 0.45\linewidth]{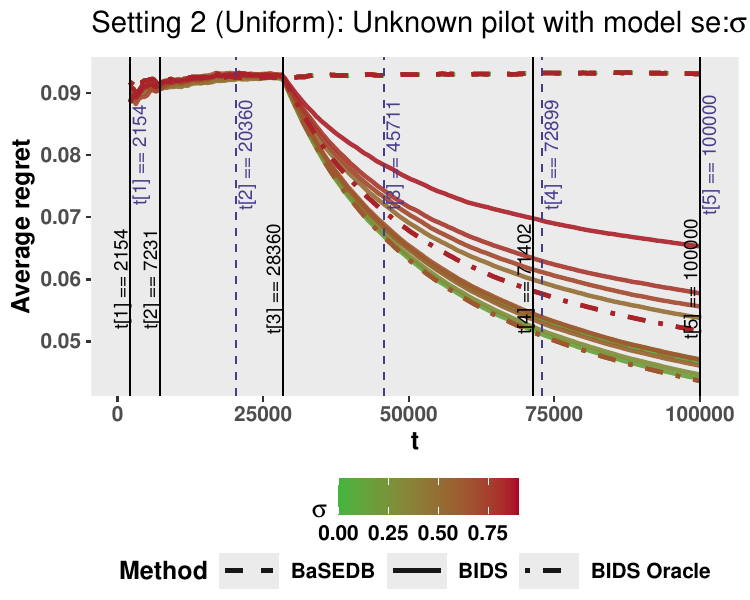} & \includegraphics[width = 0.45\linewidth]{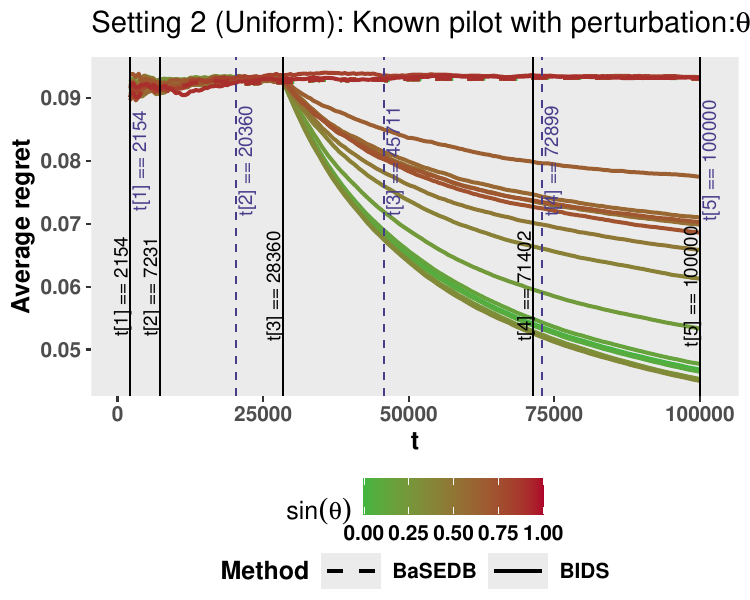}
    \end{tabular}
\caption{Average regret ($(\mathcal{R}_t)_{t=1}^T$) with uniformly distributed covariates with perturbed true direction $\beta_0$ by an angle $\theta$.  As the perturbation gets larger, the performance of the SIR batched bandit still beats the nonparametric analogue but gets further way from the oracle direction.}
    \label{fig:simulation_uniform}
\end{figure}
The average regret over 20 replications of each algorithm is shown in Figures \ref{fig:simulation_normal} and \ref{fig:simulation_uniform} for normally and uniformly distributed covariates, respectively. Note, the black solid and blue dashed vertical lines in all the four plots denote the $M = 5$ batches for BIDS and nonparametric analogue (BaSEDB), respectively, chosen according to the theory as described in Section \ref{sec: dynamic_binning}. Since the width of the BaSEDB algorithm depends on the covariate dimension $d$, we notice that the bins are much wider in the nonparametric setting as compared to the semiparametric GMABC setting. For the case where the pilot direction is available, both for Normally distributed covariates [Figure \ref{fig:simulation_normal}(b)] and Uniformly distributed covariates [Figure \ref{fig:simulation_uniform}(b)], we observe that as the perturbation, $\sin(\theta)$, increases from $0$ to $0.8$ (corresponds to $\theta \leq \pi/4$), the performance of the proposed algorithm deteriorates (solid green to solid red lines) and stops learning if the perturbation is larger, similar to the nonparametric analogue. However for $\theta \leq \pi/4$, it still outperforms the nonparametric analogue (dashed lines), where no arm elimination appears to occur. The decline in performance seems to be more pronounced for Normally distributed covariates compared to Uniform ones.

Finally, the performance of the proposed algorithm with the oracle direction (dashed-dotted lines) shows slight variation as model noise increases, but it remains consistently better than the other algorithms, as expected. This variation in the oracle's performance could be attributed to variability across different simulation runs of the decision-making process.

{\subsection{Additional simulations with multiple arms ($K>2$)}
\label{sec: simulation_multiarm}

To further investigate the behavior of the proposed BIDS algorithm in settings with more than two arms, we conduct additional experiments with $K\in\{3,5,8\}$. These experiments illustrate how the algorithm scales with the number of arms and how the batching structure adapts when more competing reward functions are present. 
We consider two reward configurations of increasing difficulty.

\paragraph{Simulation settings}
Both settings extend the two-arm design used earlier by generating rewards through a shared single-index model with projected covariate $\eta_t = X_t^\top \beta_0$.

\paragraph{Setting 1: single nonlinear arm}
One arm retains the nonlinear bump function $f(x)$ defined in \eqref{eq: fx_simulation}, while the remaining $K-1$ arms are assigned linear functions with varying slopes as shown in Figure~\ref{fig:Kgr2_rewardfuncs}:
\[
f^{(1)}(x) = f(x), \qquad 
f^{(k)}(x) = c + a_k(x-\eta_{\rm mid}), \quad k=2,\dots,K,
\]
where $\eta_{\rm mid}=(l+u)/2$ and the slopes $a_k$ are evenly spaced within a fixed range. 
This construction creates a collection of arms where the optimal arm varies across the projected covariate space while preserving the shared single-index structure.

\begin{figure}[htbp]
\centering
\begin{tabular}{ccc}
(a) & (b) & (c)\\
\includegraphics[width = 0.3\linewidth]{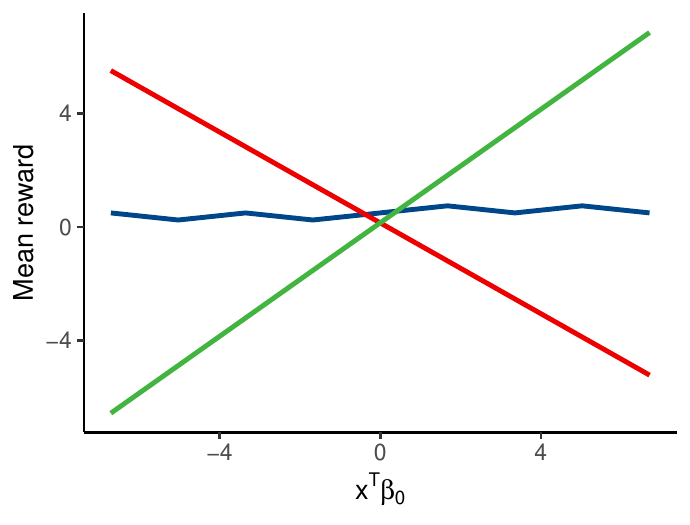} &
\includegraphics[width = 0.3\linewidth]{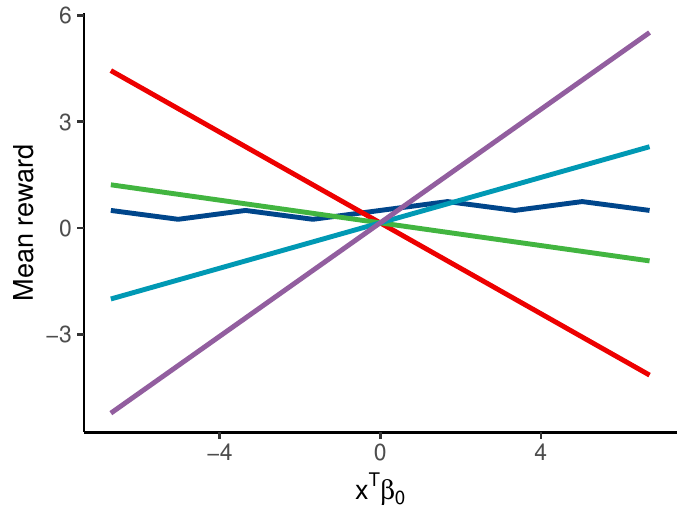} &
\includegraphics[width = 0.3\linewidth]{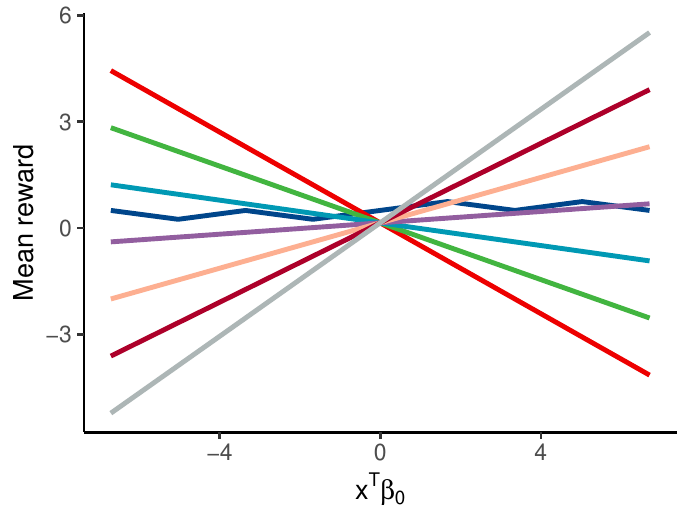}
\end{tabular}
\caption{Mean reward functions for $K=3,5,$ and $8$ arms in Setting~1.}
\label{fig:Kgr2_rewardfuncs}
\end{figure}

Rewards are generated according to
$
Y_t^{(k)} = f^{(k)}(X_t^\top \beta_0) + \epsilon_t,
\epsilon_t \sim N(0,\sigma^2),
$
with $\sigma=0.01$ and time horizon $T=10^5$.

\begin{figure}[htbp]
\centering
\begin{tabular}{ccc}
(a) & (b) & (c)\\
\includegraphics[width=0.32\linewidth]{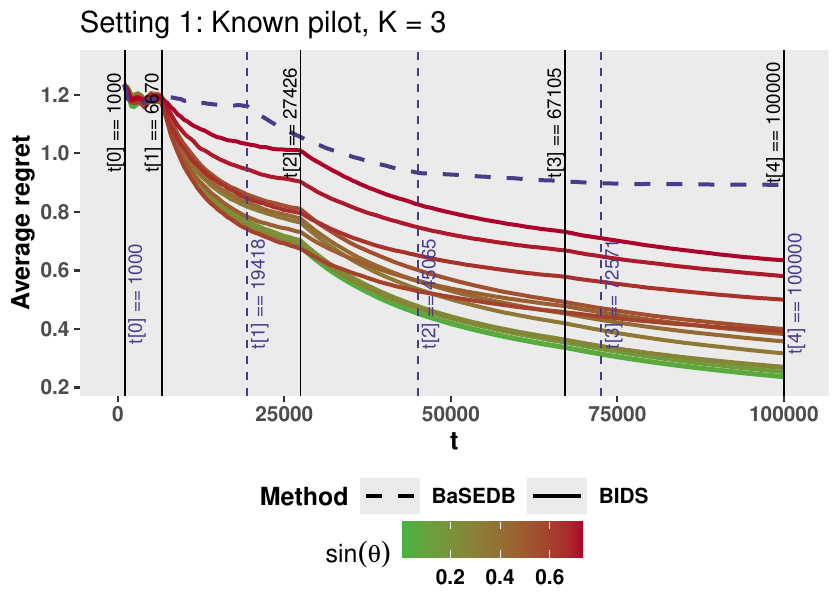} &
\includegraphics[width=0.32\linewidth]{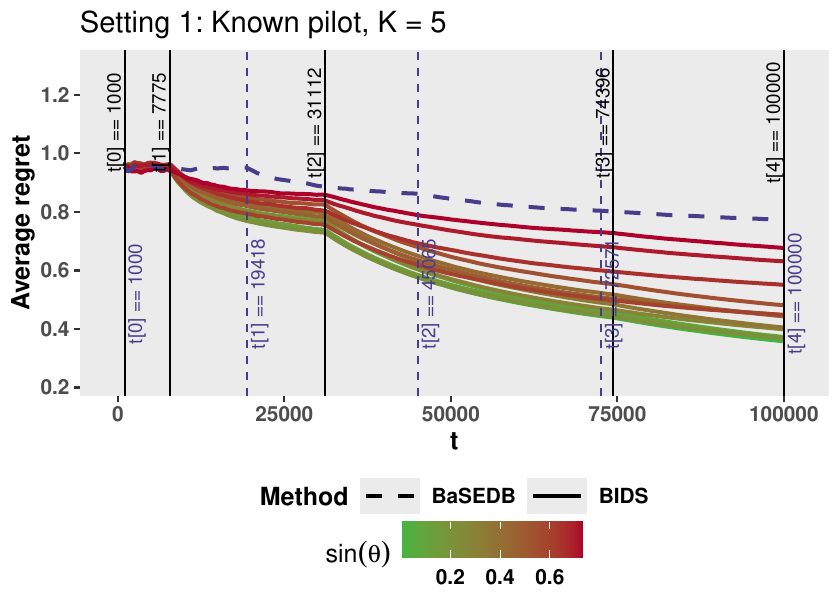} & 
\includegraphics[width=0.32\linewidth]{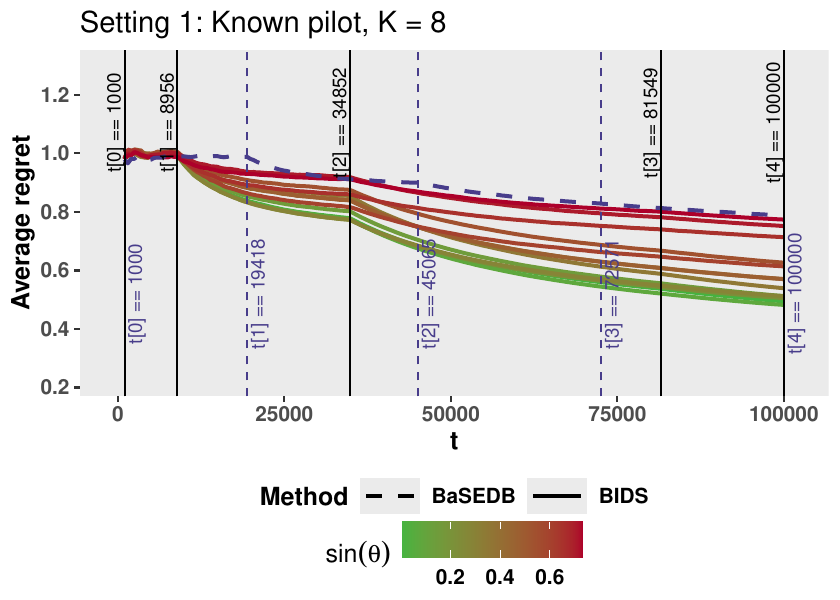}\\
\includegraphics[width=0.32\linewidth]{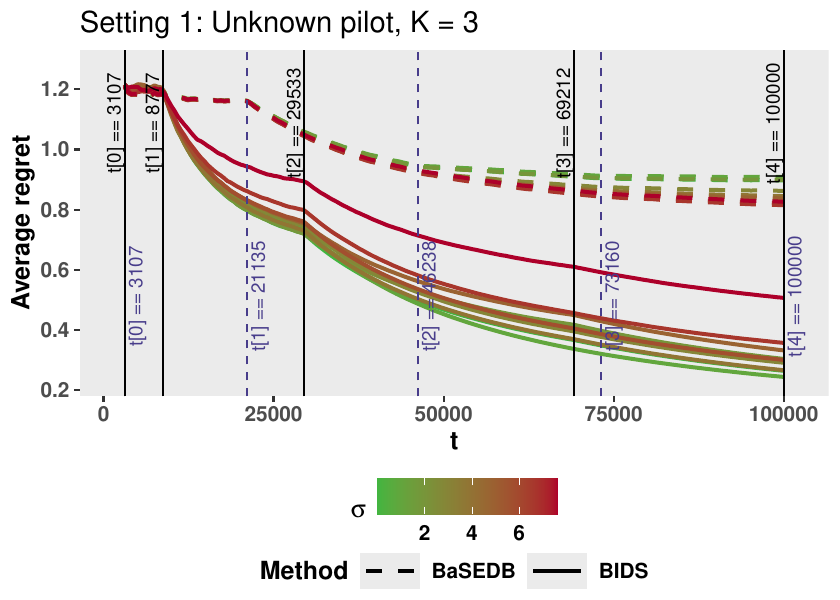} &
\includegraphics[width=0.32\linewidth]{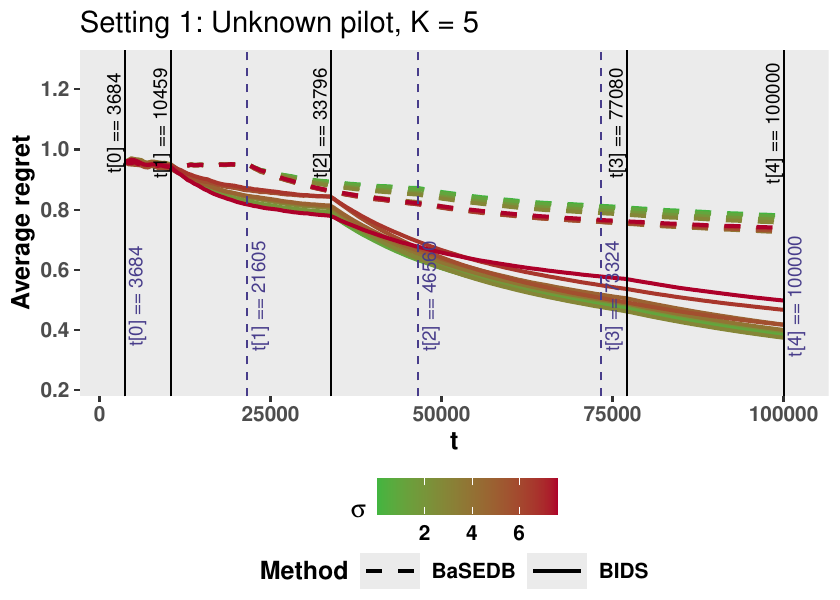} &
\includegraphics[width=0.32\linewidth]{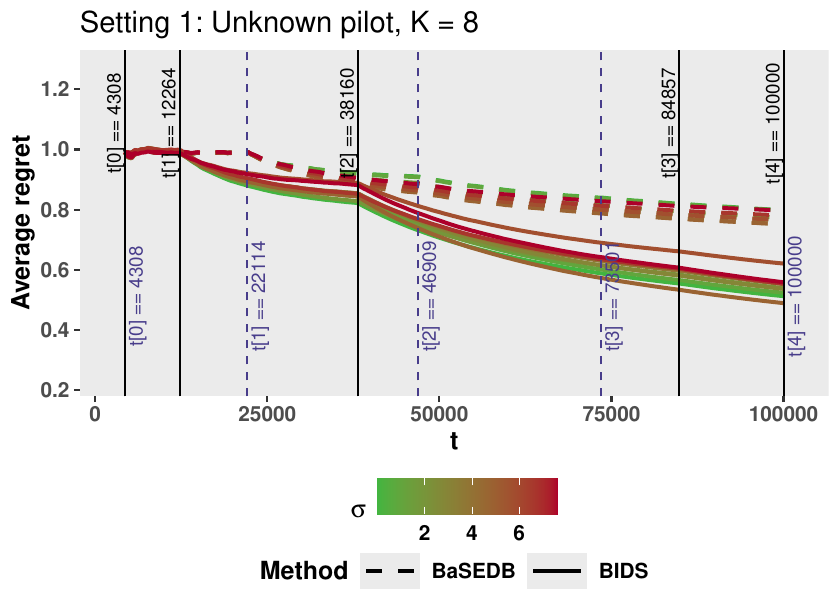}
\end{tabular}
\caption{Average regret curves for $K=3,5,8$ arms in Setting~1. 
Top row: known pilot direction with varying perturbation angles $\theta$. 
Bottom row: unknown pilot direction estimated from an initial batch under different noise levels $\sigma$. 
Solid curves correspond to the BIDS algorithm and dashed curves correspond to BaSEDB. Vertical lines denote batch endpoints.}
\label{fig:simulation_multiarm}
\end{figure}

\paragraph{Setting 2: competing nonlinear arms}
To create a more challenging configuration, we modify the design so that two arms contain nonlinear bump functions, while the remaining arms remain linear as shown in Figure~\ref{fig:Kgr2_rewardfuncs_setting2}. Specifically,
\[
f^{(1)}(x),\, f^{(2)}(x) \text{ are nonlinear bump functions,}
\qquad
f^{(k)}(x) = c + a_k(x-\eta_{\rm mid}), \quad k=3,\dots,K .
\]
This configuration introduces competition between nonlinear arms and therefore makes identification of the optimal arm more difficult.

\begin{figure}[htbp]
\centering
\begin{tabular}{ccc}
(a) & (b) & (c)\\
\includegraphics[width = 0.3\linewidth]{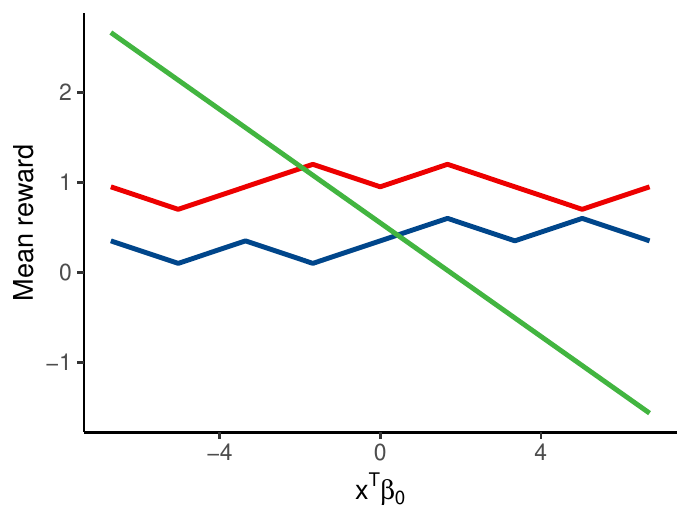} &
\includegraphics[width = 0.3\linewidth]{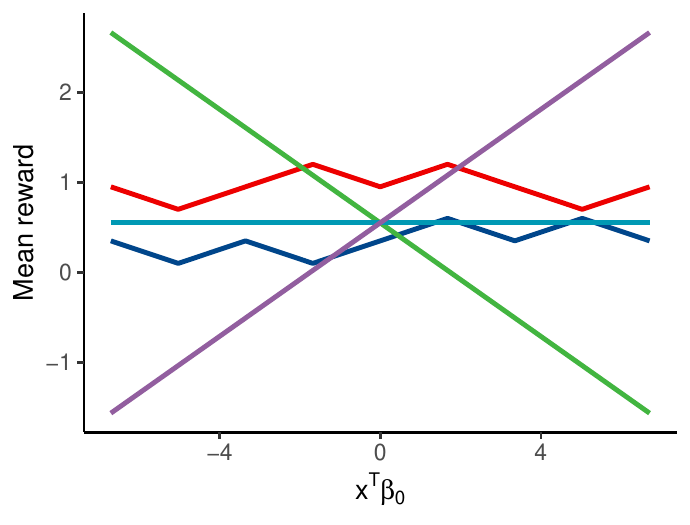} &
\includegraphics[width = 0.3\linewidth]{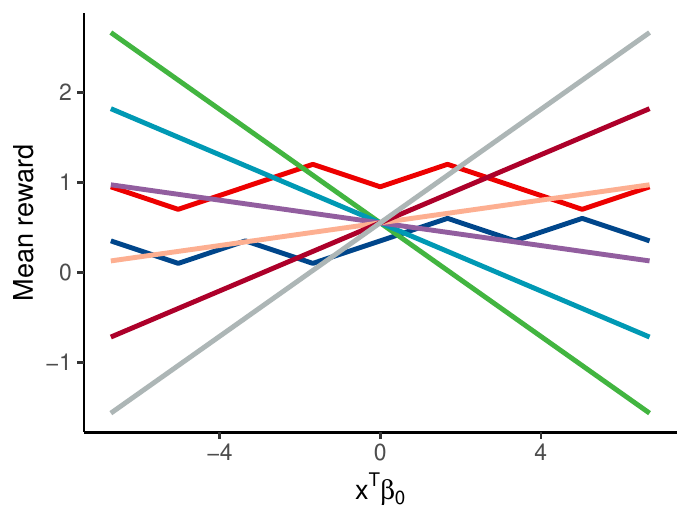}
\end{tabular}
\caption{Mean reward functions for $K=3,5,$ and $8$ arms in Setting~2 with two nonlinear bump arms.}
\label{fig:Kgr2_rewardfuncs_setting2}
\end{figure}

\begin{figure}[htbp]
\centering
\begin{tabular}{ccc}
(a) & (b) & (c)\\
\includegraphics[width=0.32\linewidth]{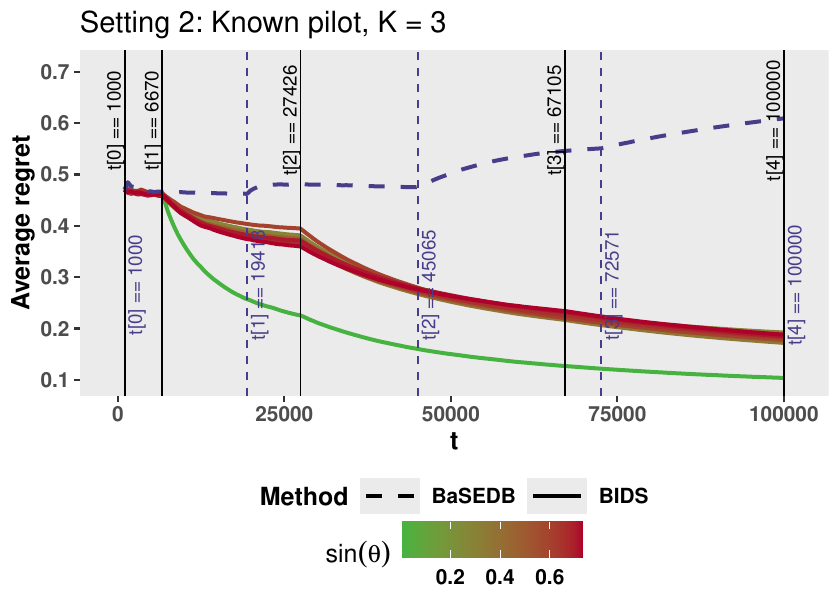} &
\includegraphics[width=0.32\linewidth]{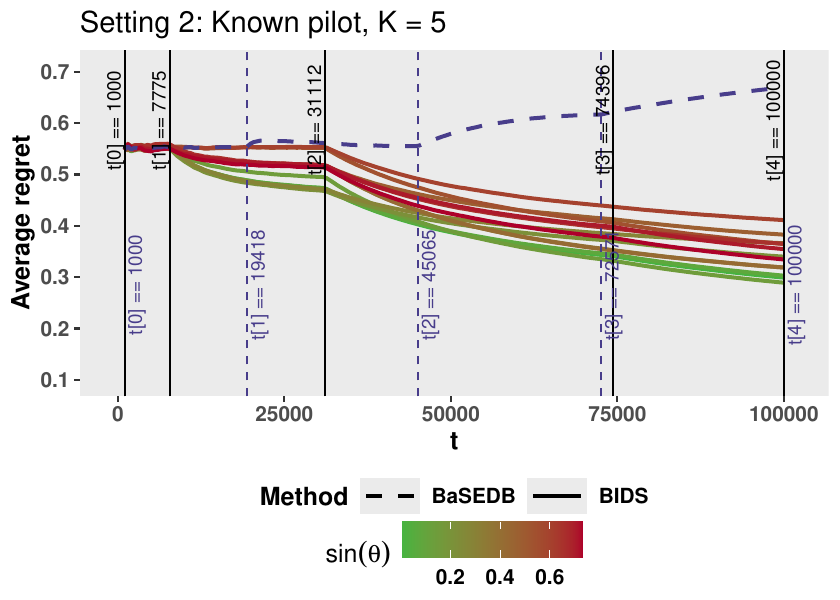} &
\includegraphics[width=0.32\linewidth]{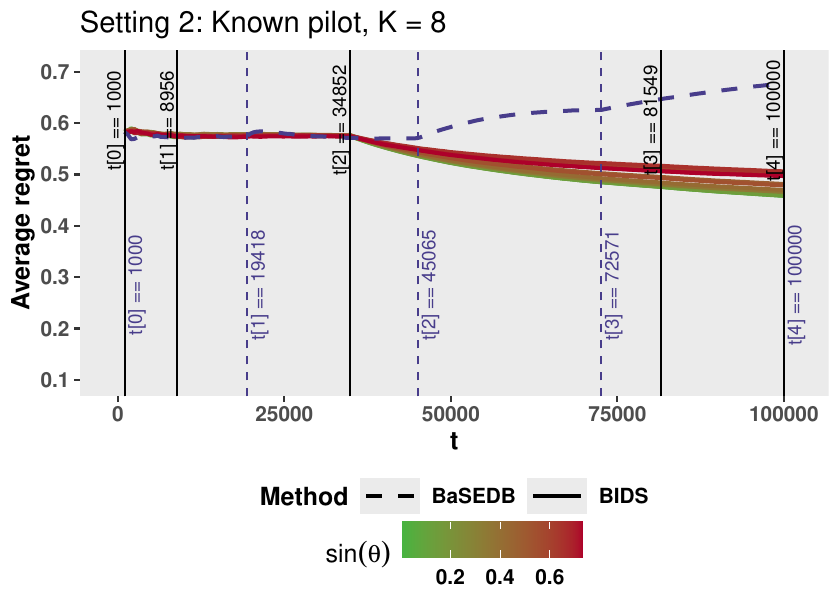}
\end{tabular}
\caption{Average regret curves for $K=3,5,8$ arms in Setting~2 with competing nonlinear arms.}
\label{fig:simulation_multiarm_setting2}
\end{figure}

\paragraph{Algorithm set-up}
The BIDS algorithm requires specifying the number of batches, the binning schedule, and the elimination thresholds. For $K>2$ arms, these quantities depend explicitly on $K$ through the theoretical scaling derived in Section~\ref{sec: regret_upper_bound}. In particular, both the batch sizes and the bin widths increase with $K$.

The nonparametric BaSEDB algorithm of \citesupp{jiang2025batched_supp} was originally proposed for the two-arm setting. In principle, one could attempt to apply the same $K$-dependent batching and binning schedules used by BIDS to BaSEDB. However, in practice this modification resulted in no arm eliminations even in the relatively favorable Setting~1. For this reason, in the experiments below we instead keep BaSEDB's original two-arm batching and binning schedule while allowing BIDS to use the $K$-dependent schedule dictated by our theory.

\paragraph{Results}
As observed in Figures~\ref{fig:simulation_multiarm} and \ref{fig:simulation_multiarm_setting2}, across both settings, the average regret curves of the BIDS algorithm lie below those of BaSEDB, indicating consistently improved performance relative to the nonparametric baseline. As expected, the regret increases as the number of arms grows, reflecting the increased difficulty of identifying the optimal arm when more alternatives are available.

For BIDS, the batch endpoints become more widely spaced for larger $K$, consistent with the theoretical scaling of the batch schedule. In the more challenging Setting~2 with competing nonlinear arms, we observe that BaSEDB fails to eliminate suboptimal arms when applied directly in the $K>2$ case, resulting in nearly flat regret curves. This behavior suggests that extending BaSEDB to multi-arm settings requires modifying its batching or elimination thresholds to account for the dependence on $K$.

{Overall, these results are in good agreement with the predicted 
$K$-dependence in Theorem~4.2; while the range $K \in \{3,5,8\}$ 
is limited, the observed trends are encouraging and consistent 
with the theoretical scaling, providing empirical support that the 
proposed single-index framework scales naturally beyond the two-arm 
case while preserving its learning behavior.}
}

\subsection{Additional real data results} \label{sec: additional_realdata}
\begin{figure}[htbp]
    \centering
\includegraphics[width = 0.7\linewidth]{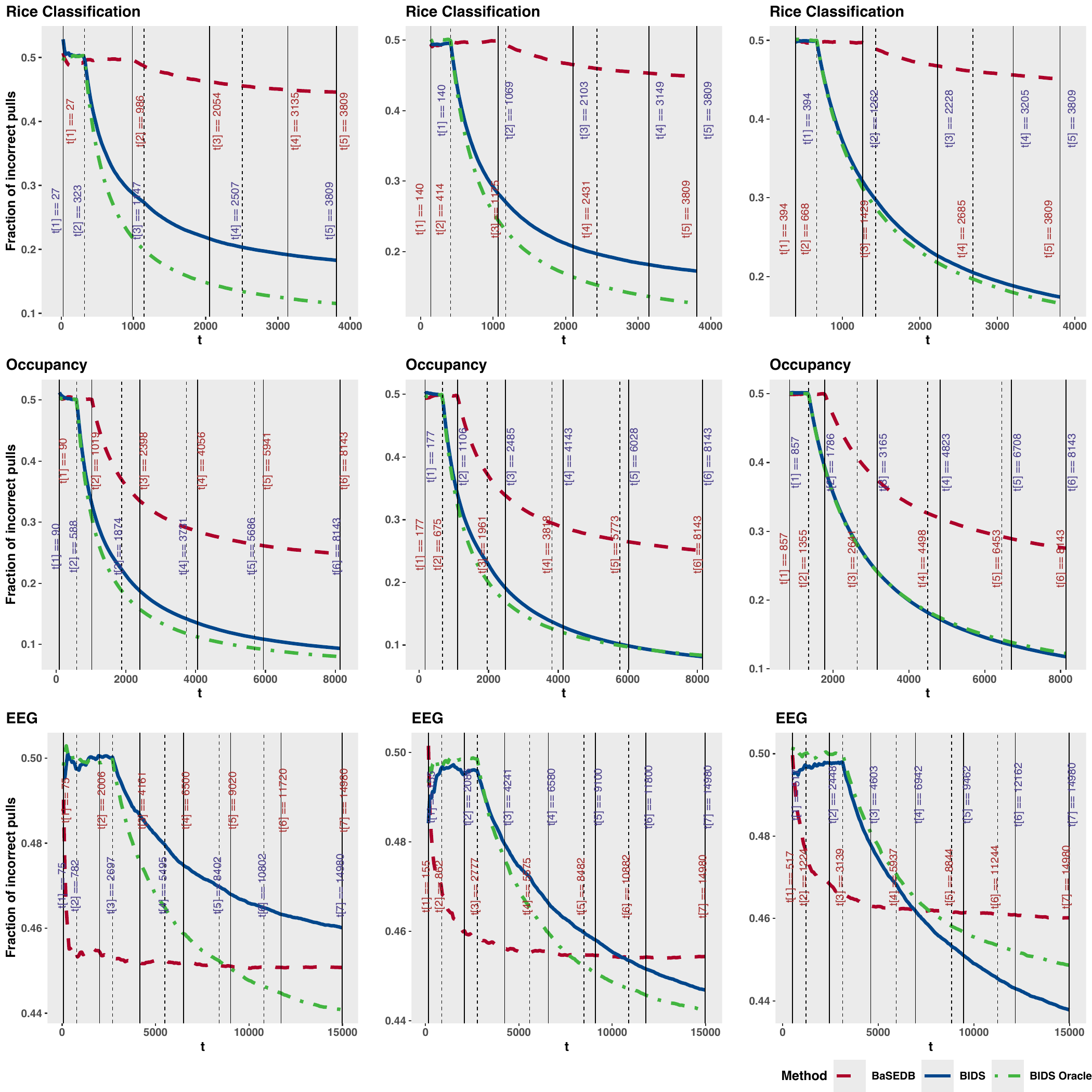}
    \caption{Comparison of expected regret of the proposed  BIDS algorithm and BaSEDB  on a) rice classification, b) occupancy detection, and c) EEG datasets with $\beta_0$ estimated in the initial phase with $t_1 = t_{\rm{init}}$ increasing as we go from left to right for the respective datasets. Vertical lines denote the batch markings for both the algorithms. Observe that the BIDS outperforms BaSEDB in all instances.}
    \label{fig:SIR_NP_ReadData_diffsamplesizes}
\end{figure}
We compare the performance of the BIDS algorithm and the BaSEDB algorithm of \citesupp{jiang2025batched_supp} when different initial batch sizes are used to estimate the direction $\beta_0$. We let $t_0 = 1$. In Figure \ref{fig:SIR_NP_ReadData_diffsamplesizes}, note that the columns denote increasing initial batch size $t_1  = t_{\rm{init}}$, as denoted by the labels on the first vertical lines in the plots. Vertical solid lines denote the batch end points for the GMABC framework as proposed in \eqref{eq: batch_size}, and the dashed lines denote the batch end points for the nonparametric batched bandits framework as suggested by \citesupp{jiang2025batched_supp}. Since the bin-widths depend on $d$ in nonparametric batched bandits, we see that the batch sizes are much larger than the corresponding GMABC setup where the bin-width does not depend on the number of covariates. \\
Similar to Section \ref{sec: 5_simulation}, we notice that BIDS outperforms BaSEDB algorithm, even though we do not know the true data generating mechanism in any of these datasets. While in the EEG dataset, for a small initial batch size ($t_{\rm{init}} = 75$), the BIDS algorithm incurred large regret in the beginning, the rate of decrease is much faster. We notice that as the initial sample size increases, the average regret for the BIDS algorithm gets closer to the oracle BIDS algorithm. In fact, the regret rate for the BIDS algorithm decreases even faster than that of the oracle BIDS algorithm. This may be because, as we incorporate more data to learn the direction, we estimate the direction for each arm separately before combining them using Algorithm \ref{algorithm: Initial_Dir_Estimation}. In contrast, the oracle direction utilizes the entire dataset to determine a single direction, which could correspond to a possibly mis-specified model.

\bibliographystylesupp{siamplain}
\bibliographysupp{refs_supp}
\fi


\end{document}